\documentclass{article}

\usepackage[preprint]{neurips_2026}
\usepackage{wrapfig}
\usepackage{tabularx}
\usepackage[most]{tcolorbox}
\usepackage[utf8]{inputenc} 
\usepackage[T1]{fontenc}    
\usepackage{hyperref}       
\usepackage{url}            
\usepackage{booktabs}       
\usepackage{amsfonts}       
\usepackage{nicefrac}       
\usepackage{microtype}      
\usepackage{xcolor}         

\usepackage[dvipsnames]{xcolor}
\usepackage[table]{xcolor}

\usepackage{microtype}
\usepackage{graphicx}
\usepackage{subcaption}
\usepackage{enumitem}
\usepackage{pifont}
\usepackage{multicol}
\usepackage{multirow}
\usepackage{booktabs} 

\usepackage{listings}
\usepackage{hyperref}

\usepackage{amsmath}
\usepackage{amssymb}
\usepackage{mathtools}
\usepackage{amsthm}

\usepackage[capitalize,noabbrev]{cleveref}

\theoremstyle{plain}
\newtheorem{theorem}{Theorem}[section]
\newtheorem{proposition}[theorem]{Proposition}
\newtheorem{lemma}[theorem]{Lemma}

\theoremstyle{definition}
\newtheorem{definition}[theorem]{Definition}

\theoremstyle{remark}

\newcommand{\cmark}{\ding{51}} 
\newcommand{\xmark}{\ding{55}} 
\newcommand{\bstar}{\ding{72}}

\DeclareMathOperator*{\argmin}{\arg\min}
\newcommand{\methodname}{MARLA}
\usepackage[textsize=tiny]{todonotes}

\title{Rank-Constrained Adaptation for \\ Reliable Real-World Performance}

\author{%
\small
Abinitha Gourabathina\textsuperscript{1} \quad
Hyewon Jeong\textsuperscript{1} \quad
Teya Bergamaschi\textsuperscript{1} \\[2pt]
Marzyeh Ghassemi\textsuperscript{1} \quad
Collin Stultz\textsuperscript{1,2,*} \\[6pt]
\footnotesize
\textsuperscript{1}Massachusetts Institute of Technology, Cambridge, MA, USA \\
\footnotesize
\textsuperscript{2}Harvard Medical School, Boston, MA, USA \\[2pt]
\footnotesize
\textsuperscript{*}Correspondence: \texttt{cmstultz@mit.edu}%
}
\begin{document}
\maketitle

\begin{abstract}
Deep learning models trained to optimize average accuracy often exhibit systematic failures on particular subpopulations. In real-world settings like healthcare, the subpopulations most affected by such disparities are frequently unlabeled, partially observed, or not known in advance. Existing group-robust methods typically assume prior knowledge of the relevant subgroups, using group annotations for training, validation, or model selection. We propose Misclassification Aware Rank-Limited Adaptation (\methodname{}), a parameter-efficient method for improving worst group performance without explicit subgroup annotations. \methodname{} leverages an ERM-trained model by calculating the model's misclassification probability scores on a held-out adaptation set to identify a low-dimensional subspace where errors concentrate. We then learn a rank-restricted additive correction to the classifier logits within that subspace. Across seven real-world datasets, we evaluate group robustness under three settings: no knowledge of subgroup relevance, partial knowledge of subgroup relevance, and full knowledge of subgroup relevance. \methodname{} improves worst-group performance while remaining fast and parameter-efficient, with data-guided hyperparameter selection. Our code is available \textcolor{blue}{\href{https://github.com/abinithago/MARLA}{here.}}
\end{abstract}

\section{Introduction}
\vspace{-5pt}
Deep neural networks (DNNs) trained to optimize average accuracy can fail systematically on subpopulations that are rare, underrepresented, or spuriously correlated with the label \citep{tsipras2019robustnessoddsaccuracy, shah2020pitfallssimplicitybiasneural, raghavan2020mitigating}. Such failures are especially concerning in high-stakes settings, including hiring \citep{raghavan2020mitigating, wilson2024gender}, hate speech detection \citep{sap2019risk}, and clinical prediction \citep{obermeyer2019dissecting, banerjee2023shortcuts}, where model errors can reinforce existing inequities. Subpopulation failures can arise from multiple forms of distribution shifts \citep{yang2023changehardcloserlook}, including spurious correlations \citep{geirhos2020shortcut}, attribute imbalance \citep{pmlr-v139-martinez21a}, class imbalance \citep{liu2019largescalelongtailedrecognitionopen}, and attribute generalization failures when models encounter previously unseen attribute values at test time \citep{santurkar2020breedsbenchmarkssubpopulationshift}.

When training a model, subgroups relevant for deployment may not be fully known a priori \citep{pmlr-v80-kearns18a}. In medical imaging, for instance, a model trained and evaluated on a coarse disease label may achieve high aggregate performance while systematically failing on rare but clinically meaningful subtypes \citep{oakden2020hidden}. Group labels may be expensive to obtain \citep{sohoni2020george, liu2021jtt, pmlr-v81-buolamwini18a}, noisy \citep{krumpal2013determinants, pmlr-v108-awasthi20a}, or missing altogether \citep{kallus2022assessing, lahoti2020fairness}. Moreover, the relevant grouping structure may itself be uncertain or evolve over time. For example, during the COVID-19 pandemic, clinically meaningful subgroups changed with variants \citep{lewnard2022clinical}, vaccination \citep{taylor2022covid}, prior infection \citep{altarawneh2022effects}, and treatment practices \citep{levy2022development}. Models trained to optimize for groups early in the pandemic could fail on these later-emerging worst groups.

Current group robustness methods largely rely on complete group relevance information \citep{yang2023changehardcloserlook}. Methods like Group DRO \citep{sagawa2020GDRO} are known to use group labels to train, but other popular methods that are considered group-agnostic, like Just Train Twice (JTT) \citep{liu2021jtt} and Automatic Feature Reweighting (AFR) \citep{qiu2023simplefastgrouprobustness}, rely on group annotations for hyperparameter tuning, early stopping, and model selection. In assuming complete group knowledge, models are vulnerable to failures on omitted or latent subgroups \citep{kaplow1999conflict, pmlr-v80-kearns18a, 10.1257/pandp.20181018, hashimoto2018fairness}. This creates a gap between standard robustness benchmarks and deployment settings where complete information about subgroup relevance is unavailable. We therefore ask: can we improve worst-group performance without assuming that relevant subgroups are fully known?

\begin{figure}
    \centering
    \includegraphics[width=0.9\linewidth]{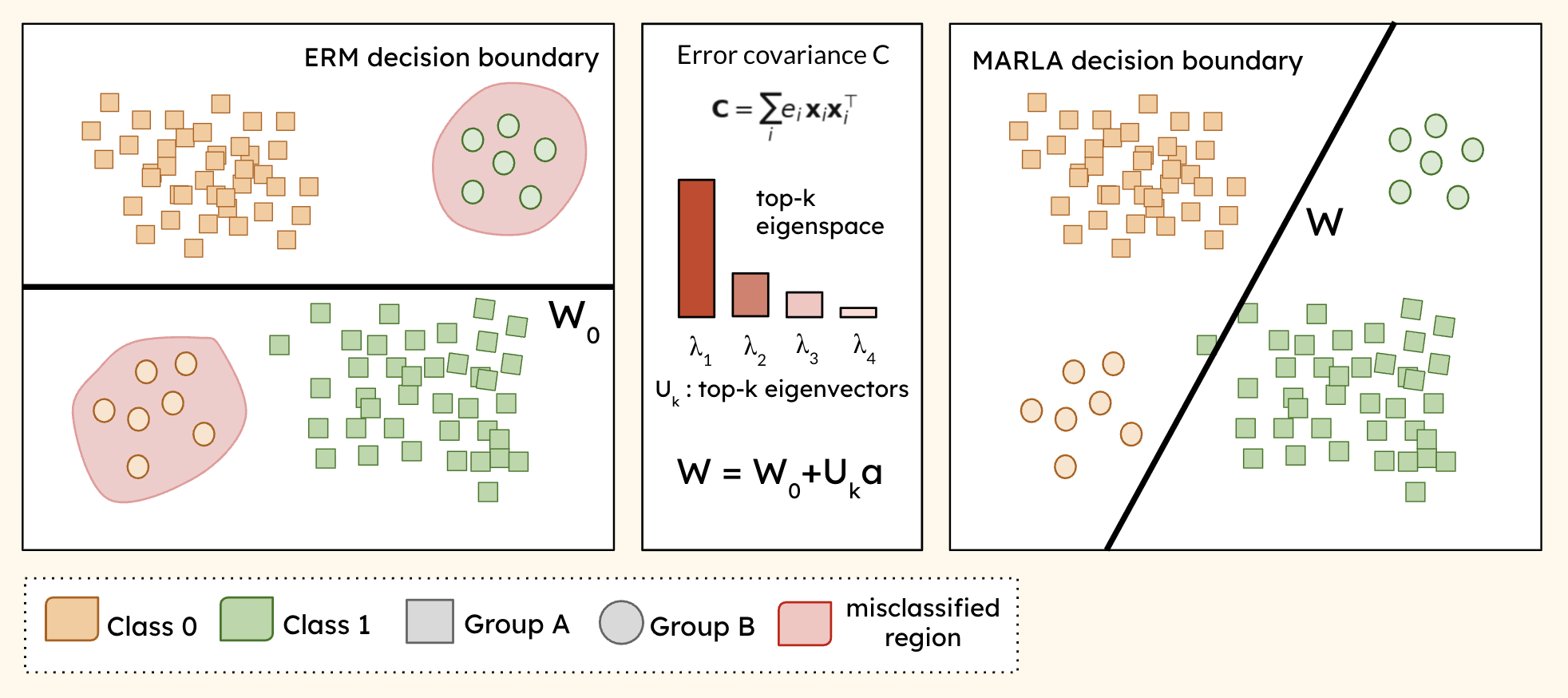}
    \caption{\textbf{\methodname{}.} 
    (\textit{Left}) ERM can learn a decision boundary that performs well on average while systematically misclassifying a subpopulation. 
    (\textit{Middle}) \methodname{} assigns larger weights to examples where the model gives low probability to the true label, and uses these weights to identify important directions in representation space. 
    (\textit{Right}) \methodname{} learns a rank-restricted additive correction to the classifier logits along these directions, improving classification for affected samples without unrestricted classifier retraining. Colors denote class labels, while marker shapes denote attribute groups.}
    \vspace{-10pt}
    \label{fig:teaser_figure}
\end{figure}

We propose \methodname{}, a parameter-efficient method for improving group robustness without explicit subgroup annotations. Starting from an ERM-trained model, \methodname{} uses the predicted probability of the true label to identify examples the model is likely to handle poorly. It assigns larger weights to examples with lower true-label probability and uses these weights to compute a weighted covariance in representation space. The leading directions of this covariance define an empirically estimated rank-constrained subspace associated with low-confidence errors. Unlike AFR/DFR \citep{qiu2023simplefastgrouprobustness,kirichenko2022dfr}, which reweight or retrain the full classifier head, \methodname{} restricts the classifier update to this specific subspace. It then learns a rank-restricted additive correction to the logits, geometrically restricting adaptation to the estimated subspace.

We evaluate \methodname{} with three levels of subgroup knowledge: no knowledge of subgroup relevance, partial knowledge of subgroup relevance, and full knowledge of subgroup relevance. We compare \methodname{}'s performance against 24 baselines \citep{arjovsky2019invariant, kang2020decouplingrepresentationclassifierlongtailed, liu2021jtt,nam2020learning,duchi2020learningmodelsuniformperformance,kirichenko2022dfr,yao2022improvingoutofdistributionrobustnessselective,zhang2022cnc,qiu2023simplefastgrouprobustness,han2024improvinggrouprobustnessspurious,to2025diverseprototypicalensemblesimprove,zhang2017mixup,cui2019class,lin2017focal,cao2019learning,ren2020balanced,sagawa2020GDRO} and seven image, text, and tabular datasets \citep{granger2003predictors, fox2014should, fox2006prediction, rationale2001design, eagle2002practice,johnson2023mimic,sagawa2020GDRO,7410782,williams2018broadcoveragechallengecorpussentence,borkan2019nuanced,irvin2019chexpertlargechestradiograph}. Then, we perform a toy experiment to show how both our weighting strategy and rank-restricted adaptation are necessary components for the method. Our core contributions are as follows: 
\vspace{-5pt}
\begin{enumerate}[itemsep=0em]
    \item We introduce \emph{partial subgroup knowledge} as an evaluation regime for group robustness, reflecting settings where only some subgroup structure is available for model selection while deployment-relevant groups may be omitted or only partially specified.
    \item Across datasets, \methodname{} is consistently competitive against state-of-the-art methods. \methodname{} improves worst-group accuracy under no, partial, and complete subgroup-knowledge settings while maintaining strong average accuracy. 
    \item \methodname{} is fast and parameter-efficient, requiring only a small rank-restricted correction and using the weighted covariance spectrum to guide data-guided rank selection. 
\end{enumerate}
\vspace{-8pt}
\section{Related Work}
\vspace{-8pt}
A growing body of literature focuses on group robustness to assess whether models learn meaningful feature relationships and achieve equitable performance \citep{yang2023changehardcloserlook}. Traditional ERM training, however, has been shown to learn misleading shortcuts \citep{mccoy2019right, ribeiro2016should, beery2018recognition, zech2018variable, winkler2019association, shahamatdar2024deceptive}, causing performance gaps in subpopulations \citep{zhao2017men, banerjee2023shortcuts}. Many existing methods for group robustness assume access to labeled relevant subgroups or attributes \citep{sagawa2020GDRO, arjovsky2019invariant} or on a subset of the available data used to train the parameters of the model \citep{kirichenko2022dfr, yao2022improvingoutofdistributionrobustnessselective}. Such access is often unrealistic given how labels may be expensive to obtain, noisy, missing, or unknown.

Recent work has explored the setting of assuming access to all subgroups during validation. These methods
commonly train an initial ERM model and then retrain a classifier with balanced classes \citep{kang2020decouplingrepresentationclassifierlongtailed}, upweighting misclassified samples \citep{liu2021jtt, nam2020learning}, or contrastive learning \citep{zhang2022cnc}. A prominent paradigm of second stage classifier retraining is group inference from the representations of the initial ERM. Prior approaches infer groups from clustered representations \citep{sohoni2020george}, identifying violations of the invariant risk minimization (IRM) objective, a trained spurious attribute classifier \citep{han2024improvinggrouprobustnessspurious}, and ensemble models \citep{to2025diverseprototypicalensemblesimprove}. Just Train Twice (JTT) retrains the model by upsampling incorrectly predicted examples \citep{liu2021jtt}. Automatic Feature Reweighting (AFR) retrains the model's last layer with loss-based reweighting \citep{qiu2023simplefastgrouprobustness}.However, many of these methods are evaluated assuming full knowledge of group relevance in the validation set, ignoring scenarios where some groups are unknown a priori. 

Our method also draws on prior work that established how biases can be mitigated with low dimensional corrections in transformers and vision language models \citep{zhao2025bias, jang2025target, chuang2023debiasing, jung2024unified, gerych2024bendvlm, zhao2025bias}. Additionally, parameter-efficient adaptation techniques such as low-rank updates (LoRA) \citep{hu2022lora} have shown that task- or subgroup-specific improvements can be achieved by updating only a small fraction of model parameters \citep{hu2022lora, rimsky2024steering, wu2024reft, yin2024lofit}.

We provide an extended discussion of  related work in App. \ref{app:related_work_plus}.

\section{Motivation for Evaluating Under Incomplete Group Knowledge}
\label{sec:motivation}
Relevant subgroup structure may become fully known only \emph{after} model deployment. For example, clinical evidence and treatment guidelines have historically underrepresented women in cardiovascular research, contributing to sex-specific gaps in evidence and treatment effectiveness \citep{vogel2021lancet, mas2023representativeness, rossello2025beta}. Similar post-hoc discoveries of vulnerable subpopulations arise in pharmacogenomics \citep{martin2017assessment}, diagnostic imaging \citep{oakden2020hidden}, and public health \citep{mccoy2024dataadaptiveidentificationeffectmodifiers}. These examples show how models trained and evaluated with only complete group information obscures potential underperformance on unknown subgroups. While \emph{some} clinical phenotyping knowledge may be known \citep{oakden2020hidden, mori2020using}, it is unrealistic to assume complete group relevance. 

The following observation formalizes that optimizing worst-case risk over the known groups provides no direct control over the risk of an omitted subgroup for group-aware methods. This implies that group-aware methods can worsen performance compared to standard training, and thus serves as motivation for evaluating methods under assumptions of \emph{partial} subgroup knowledge.

\noindent
\textbf{Problem setup.}
Let $\mathcal{X}$ and $\mathcal{Y}$ denote the input and label spaces, and let $\ell$ denote a loss applied to predictions and labels. Data are drawn from an unknown distribution $\mathcal{D}$ over $\mathcal{X}\times\mathcal{Y}$, and $\mathcal{H}$ denotes a hypothesis class of predictors $h:\mathcal{X}\to\mathcal{Y}$. We assume the population admits a latent partition into disjoint subgroups
\(
\mathcal{G}=\{G_1,\dots,G_K,G^\star\},
\)
where all subgroups are present in the training distribution, but only $G_1,\dots,G_K$ are identified and explicitly available for group-conditional optimization. The subgroup $G^\star$ represents an unknown group whose membership is not annotated, and therefore cannot be directly targeted by group-conditional objectives such as Group DRO.

For any subgroup $G\subseteq\mathcal{X}\times\mathcal{Y}$, define the subgroup risk
\[
\mathcal{R}_G(h)
=
\mathbb{E}_{(x,y)\sim\mathcal{D}}
\big[\ell(h(x),y)\mid (x,y)\in G\big].
\]
The population risk is
\[
\mathcal{R}(h)=\mathbb{E}_{(x,y)\sim\mathcal{D}}[\ell(h(x),y)]
=
\sum_{G\in\mathcal{G}}\pi_G\mathcal{R}_G(h),
\qquad
\pi_G=\mathbb{P}_{(x,y)\sim\mathcal{D}}((x,y)\in G).
\]

\begin{definition}[Empirical Risk Minimization (ERM)]
Given a finite training sample $\widehat{\mathcal{D}}=\{(x_i,y_i)\}_{i=1}^n$, empirical risk minimization selects
\[
h_{\mathrm{ERM}}
\in
\argmin_{h\in\mathcal{H}}
\frac{1}{n}\sum_{i=1}^n \ell(h(x_i),y_i).
\]
In the population-level analysis below, we use $\mathcal{R}(h)$ as the corresponding population risk.
\end{definition}

\begin{definition}[Group Distributionally Robust Optimization (Group DRO) \citep{sagawa2020GDRO}]
Given access only to annotated groups $\mathcal{G}_{\mathrm{known}}=\{G_1,\dots,G_K\}$,
Group DRO selects a hypothesis
\[
h_{\mathrm{DRO}}
\in
\argmin_{h\in\mathcal{H}}
\max_{G\in\mathcal{G}_{\mathrm{known}}}
\mathcal{R}_G(h).
\]
\end{definition}

Although Group DRO controls worst-case risk over annotated groups, it does not directly control performance on latent subgroups omitted from the optimization objective. Proposition~\ref{prop:unknown_group_dro} formalizes this limitation: when the subgroup structure used for robust optimization is incomplete, optimizing over known groups either provides no improvement over ERM for some omitted group or admits a solution with higher risk on that group. We use Group DRO as a canonical example, but the same concern applies broadly to methods that require relevant groups to be specified in advance \citep{arjovsky2019invariant, idrissi2021simple}.

\begin{proposition}
\label{prop:unknown_group_dro}
Let $\mathcal{G}=\{G_1,\dots,G_K\}$ be a latent partition of the population with $K\geq 2$ and $\pi_G=\mathbb{P}((x,y)\in G)>0$ for every $G\in\mathcal{G}$. Assume the ERM and Group DRO objectives considered below attain their minima over $\mathcal{H}$. Let
\(
h_{\mathrm{ERM}}
\in
\argmin_{h\in\mathcal{H}}
\sum_{G\in\mathcal{G}}\pi_G\mathcal{R}_G(h)
\)
be a population-risk minimizer. For any nonempty subset $\mathcal{G}'\subseteq\mathcal{G}$, define
\(
\mathcal{H}_{\mathrm{DRO}}(\mathcal{G}')
=
\argmin_{h\in\mathcal{H}}
\max_{G\in\mathcal{G}'}
\mathcal{R}_G(h)
\)
as the set of hypotheses that minimize the Group DRO objective over $\mathcal{G}'$.
Then there exists an index $i\in\{1,\dots,K\}$ such that, when $G_i$ is omitted from the group-robust objective, at least one of the following holds:
\begin{enumerate}[label=(\roman*),itemsep=0em]
    \item ERM is also optimal for the incomplete Group DRO objective:
    \(
    h_{\mathrm{ERM}}\in\mathcal{H}_{\mathrm{DRO}}(\mathcal{G}\setminus\{G_i\});
\)
    \item some optimizer of the incomplete Group DRO objective has higher risk on the omitted group than ERM:
    \(
    \exists h\in\mathcal{H}_{\mathrm{DRO}}(\mathcal{G}\setminus\{G_i\})
    \quad\text{s.t.}\quad
    \mathcal{R}_{G_i}(h)>\mathcal{R}_{G_i}(h_{\mathrm{ERM}}).
    \)
\end{enumerate}
\end{proposition}

In this statement, the omitted group $G_i$ represents a subgroup that is present in the population but not identified as relevant when constructing the group-robust objective. Proposition~\ref{prop:unknown_group_dro} is an existence result rather than a claim that every omitted group is harmed. We provide the full proof in App.~\ref{app:prop_proof} and controlled synthetic results illustrating this failure mode across multiple parameter settings in App.~\ref{app:synthetic}.
\vspace{-2pt}
\section{Misclassification Aware Rank-Limited Adaptation}
\label{sec:method}
\vspace{-8pt}
Systematic failures on underrepresented subpopulations may manifest not as isolated mistakes, but as shared regions in representation space. \methodname{} adopts this geometric view by using the ERM model's misclassification probabilities to estimate a low-dimensional subspace associated with low-confidence errors. It then learns a rank-constrained correction to the classifier logits along those directions, improving performance on latent failure modes without requiring subgroup labels.
\vspace{-4pt}
\paragraph{Base Model and Adaptation Data.}
We consider a predictor $m_\theta=c_\phi\circ e_\psi$, where $e_\psi:\mathcal{X}\to\mathbb{R}^d$ is a feature extractor and $c_\phi:\mathbb{R}^d\to\mathbb{R}^C$ is a classifier head. Given $x$, the model produces logits $f(x)=c_\phi(e_\psi(x))$. We train $m_\theta$ on $\mathcal{D}_{\mathrm{ERM}}$ using standard ERM and use a held-out adaptation set $\mathcal{D}_{\mathrm{MARLA}}$ to estimate the weighted subspace and learn the additive logit correction.

\paragraph{Weighted Covariance from True-Label Probabilities.}
Using the ERM-trained model $m_{\hat{\theta}}$, compute logits $f(x_i)$ for all $(x_i,y_i)\in\mathcal{D}_{\mathrm{MARLA}}$. Let
\(
p_{y_i}:=\Pr_{m_{\hat{\theta}}}(Y=y_i\mid x_i)
=
\operatorname{softmax}(f(x_i))_{y_i}
\)
denote the probability assigned to the ground-truth class label. We define the misclassification score
\(
u_i=1-p_{y_i},
\) which is the total probability assigned to incorrect labels. We then assign normalized weights
\begin{equation}
\setlength\abovedisplayskip{2pt}
\setlength\belowdisplayskip{1pt}
\mu_i=
\frac{\exp(\gamma u_i)}{\sum_j \exp(\gamma u_j)}
=
\frac{\exp\!\big(\gamma(1-p_{y_i})\big)}{\sum_j \exp\!\big(\gamma(1-p_{y_j})\big)}.
\label{eq:mu}
\end{equation}
Here, $\gamma>0$ controls how strongly the weighting emphasizes examples assigned low probability on the true label. Because $u_i\in[0,1]$, this weighting is bounded and avoids allowing a small number of extremely low-confidence examples to dominate the covariance estimate.

Let $z_i=e_{\hat{\psi}}(x_i)\in\mathbb{R}^d$ and $\bar{z}=\sum_i\mu_i z_i$. We compute the weighted covariance
\begin{equation}
\setlength\abovedisplayskip{1pt}
\setlength\belowdisplayskip{1pt}
\Sigma_{\mathrm{w}}
=
\sum_i \mu_i (z_i-\bar{z})(z_i-\bar{z})^\top
\in\mathbb{R}^{d\times d}.
\label{eq:sigma}
\end{equation}
The top-$k$ eigenvectors of $\Sigma_{\mathrm{w}}$ form $V_k\in\mathbb{R}^{d\times k}$ and define the low-dimensional subspace used for the correction. Intuitively, this subspace captures directions where examples with low true-label probability concentrate.

\vspace{-2mm}
\paragraph{Rank-Restricted Logit Correction.}
\methodname{} learns an additive correction to the classifier logits supported only on the subspace spanned by $V_k$. For input $x$, the adapted logits are
\begin{equation}
\setlength\abovedisplayskip{1pt}
\setlength\belowdisplayskip{1pt}
f_{\methodname{}}(x)
=
f(x)+A^\top V_k^\top e_{\hat{\psi}}(x),
\label{eq:MARLA_logits}
\end{equation}
where $A\in\mathbb{R}^{k\times C}$ is the only learnable parameter. We initialize $A=0$, so the model initially recovers the ERM logits exactly. We optimize $A$ by minimizing
\begin{equation}
\min_A
\sum_{(x_i,y_i)\in\mathcal{D}_{\mathrm{MARLA}}}
\mu_i\,
\ell\!\left(
f(x_i)+A^\top V_k^\top e_{\hat{\psi}}(x_i),
y_i
\right).
\label{eq:marla_update}
\end{equation}
This update is exactly restricted to $\operatorname{span}(V_k)$ and has rank at most $k$, since the induced classifier perturbation has the form $A^\top V_k^\top$. This algebraic restriction is exact; whether the estimated subspace captures subgroup-relevant failures depends on the adaptation data, which we study further in App.~\ref{app:real_data_subspace} and App.~\ref{app:subspace_conditions}. The rank controls the expressivity of the update, with smaller values yielding more structured corrections and larger values capturing more complex failure patterns.  

\vspace{-8pt}
\section{Experiments}
We evaluate \methodname{} on seven datasets, compared against 24 baselines, and provide detailed ablations on design decisions and hyperparameters. Further details on datasets, baselines, and implementation can be found in App.~\ref{sec:expdetails}-\ref{app:ablations}.
\vspace{-1em}
\subsection{Datasets}
\vspace{-4pt}
We consider seven image, text, and tabular classification benchmarks. Following \cite{qiu2023simplefastgrouprobustness}, we split  the training set in a 80\%/20\% split for $\mathcal{D}_{\mathrm{ERM}}$ and $\mathcal{D}_{\mathrm{MARLA}}$, but we show that performance is not
sensitive to the split in App.~\ref{ref:splitting_ratio}. We use two heterogeneous tabular datasets demonstrating real-world group shift and five standard robustness benchmarks. 

\textbf{\textsc{The Global Registry of Acute Coronary Events (GRACE)}} \citep{granger2003predictors, fox2014should, fox2006prediction, rationale2001design, eagle2002practice} is a large-scale multinational observational registry documenting hospitalized patients with acute coronary syndrome (ACS) across multiple countries. \textbf{\textsc{MIMIC-IV}} \citep{johnson2023mimic} is a publicly available electronic health record dataset from ICU admissions. In GRACE, certain age and gender groups are underrepresented; for MIMIC-IV, the underrepresented attribute groups are at the intersection of race and gender. 

\textbf{\textsc{Waterbirds}} \citep{sagawa2020GDRO} is an image classification dataset where bird types are associated with a spurious background attribute (water or land). \textbf{\textsc{CelebA}} \citep{7410782} is an image classification dataset, where hair color is an imbalanced attribute with respect to gender. \textbf{\textsc{MultiNLI}} \citep{williams2018broadcoveragechallengecorpussentence} is a text benchmark of categorizing two sentences as entailing each other, contradicting each other, or neutral. The spurious correlation is between negation words like “never” and the  “contradiction” label. \textbf{\textsc{CivilComments}} \citep{borkan2019nuanced} is a toxicity text classification dataset containing underrepresented demographic groups. \textbf{\textsc{CheXpert}} \citep{irvin2019chexpertlargechestradiograph} is a large-scale medical dataset of chest radiographs with rare pathologies, especially among certain minority groups. 
\vspace{-1em}
\subsection{Implementation}
\label{sec:implementation}
\vspace{-4pt}
We follow standard model choices and dataset splits consistent with prior work \citep{liu2021jtt, kirichenko2022dfr, qiu2023simplefastgrouprobustness}: for Waterbirds, CelebA, and CheXpert, we use ResNet-50 \citep{he2016deep} pretrained on ImageNet1k \citep{russakovsky2015imagenet}; for MultiNLI and CivilComments, we use BERT \citep{devlin2018bert}
pre-trained on Book Corpus and English Wikipedia data; for GRACE and MIMIC-IV, we use a 6-layer residual MLP (see Appendix \ref{app:mlp}). \methodname{}'s hyperparameters $k$ and $\gamma$ are tuned on a held-out validation set, with the validation metric varying by attribute availability setting. 
We choose candidate values of $k$ from the score-weighted covariance spectrum. Specifically, we compute
\(
\mathrm{CWV}(k) = \sum_{i=1}^{k}\lambda_i / \sum_{i=1}^{d}\lambda_i,
\)
and form the rank-search range from the values of $k$ corresponding to 50\%-90\% CWV, then select within this range on the validation set. The numerical range for $\gamma$ is likewise data-guided by the observed misclassification-score distribution (App.~\ref{app:parameter_ablation}). All reported results use the score weights in Eq.~\ref{eq:mu} without any additional class-rebalancing factor. 
\vspace{-1em}
\subsection{Baselines}
\vspace{-4pt}
To rigorously evaluate the effectiveness of \methodname{}, we compare performance against current state-of-the-art methods: Invariant Risk Minimization (\textbf{IRM}) \citep{arjovsky2019invariant}, Classifier Re-train (\textbf{CRT}, \textbf{ReweightCRT}) \citep{kang2020decouplingrepresentationclassifierlongtailed}, Just Train Twice (\textbf{JTT}) \citep{liu2021jtt}, Learning from Failure (\textbf{LfF)} \citep{nam2020learning}, Conditional Value at Risk DRO (\textbf{CVaRDRO}) \citep{duchi2020learningmodelsuniformperformance}, Deep Feature Reweighting (\textbf{DFR}) \citep{kirichenko2022dfr}, Learning Invariant Predictors with Selective Augmentation (\textbf{LISA}) \citep{yao2022improvingoutofdistributionrobustnessselective}, Correct-n-Contrast (\textbf{CnC}) \citep{zhang2022cnc},
 Automatic Feature Reweighting (\textbf{AFR}) \citep{qiu2023simplefastgrouprobustness}, Group Inference via data
Comparison (\textbf{GIC}) \citep{han2024improvinggrouprobustnessspurious}, and Diverse Prototypical Ensembles (\textbf{DPE}) \citep{to2025diverseprototypicalensemblesimprove}. We also consider data augmentation and rebalancing strategies, including Mixup (\textbf{Mixup}) \citep{zhang2017mixup}, class-balanced resampling (\textbf{ReSample}), and loss reweighting approaches such as Class-Balanced Loss (\textbf{CBLoss}) \citep{cui2019class}, Focal Loss (\textbf{Focal}) \citep{lin2017focal}, Label-Distribution-Aware Margin loss (\textbf{LDAM}) \citep{cao2019learning}, and Balanced Softmax (\textbf{BSoftmax}) \citep{ren2020balanced}. 
We also compare performance with  \textbf{ERM} and \textbf{Group DRO} \cite{sagawa2020GDRO}. We additionally evaluate \textbf{SELF}, \textbf{EVaLS}, \textbf{GSR}, and \textbf{Evidential Alignment (EA)} in the no- and partial-information settings.

\subsection{Attribute Availability}
\label{sec:attribute_availability}
We evaluate performance with three settings of attribute availability:
\vspace{-10pt}
\newtcolorbox{subgroupbox}[2]{
    enhanced,
    width=\linewidth,
    colback=#1!6,
    colframe=#1!35!black,
    boxrule=0.8pt,
    arc=1.5mm,
    outer arc=1.5mm,
    left=4pt,
    right=4pt,
    top=4pt,
    bottom=3pt,
    before skip=0pt,
    after skip=0pt,
    sharp corners=south,
    borderline west={1.6pt}{0pt}{#1!65!black},
    title={#2},
    fonttitle=\footnotesize\bfseries,
    coltitle=#1!35!black,
    attach title to upper={\par\vspace{3pt}},
}

\noindent
\begin{minipage}[t]{0.315\textwidth}
\vspace{0pt}
\begin{subgroupbox}{blue}{\textsc{\underline{No knowledge} \\ of subgroup relevance}}
\footnotesize
Attribute annotations are unavailable during training or validation. We select hyperparameters and early stop using worst class-label accuracy \citep{yang2023changehardcloserlook, to2025diverseprototypicalensemblesimprove}.
\end{subgroupbox}
\end{minipage}
\hfill
\begin{minipage}[t]{0.315\textwidth}
\vspace{0pt}
\begin{subgroupbox}{violet}{\textsc{\underline{Partial knowledge} \\ of subgroup relevance}}
\footnotesize
Only some subgroups are defined during training and validation, while evaluation covers all subgroups, including unknown ones. We select hyperparameters and early stop using WGA over known groups.
\end{subgroupbox}
\end{minipage}
\hfill
\begin{minipage}[t]{0.315\textwidth}
\vspace{0pt}
\begin{subgroupbox}{teal}{\textsc{\underline{Complete knowledge} \\ of subgroup relevance}}
\footnotesize
All subgroups of interest are assumed to be defined and available in validation. This setting provides a conventional comparison with prior work \citep{kirichenko2022dfr, qiu2023simplefastgrouprobustness, to2025diverseprototypicalensemblesimprove}. We tune using WGA over all subgroups.
\end{subgroupbox}
\end{minipage}

\vspace{-10pt}
\section{Results}
\subsection{Setting 1: No Knowledge of Subgroup Relevance}

In this setting, \methodname{} achieves the highest WGA on five of seven datasets and is within 0.9 percentage points of the best WGA on the remaining two (Table~\ref{tab:no_group_info}). On the clinical datasets, \methodname{} attains the highest WGA on both \textsc{GRACE} and \textsc{MIMIC-IV}, improving over ERM by 0.9 and 6.1 percentage points, respectively. On \textsc{CelebA}, \methodname{} achieves the highest WGA while maintaining 95.2\% average accuracy. We also observe that several robustness baselines, including JTT and AFR on GRACE, can underperform ERM substantially without subgroup supervision. Additional clinical metrics and analyses of the learned subspace are reported in App.~\ref{app:clinical_metrics} and App.~\ref{app:real_data_subspace}. 

\setlength{\tabcolsep}{5pt}
\begin{table*}[!h]
\centering
\begin{minipage}[t]{0.49\textwidth}
\vspace{0pt}
\caption{
\textbf{Worst-group accuracy (WGA) and average accuracy on test set with \underline{no group information.}} Rows with asterisks (*) are reproduced or evaluated in our implementation; other results are reported from original papers. Best and second-best values are \textbf{bold} and \underline{underlined}, respectively. ``-'' indicates results not reported. All results are the mean $\pm$ std (in \%) averaged over three runs. $\dagger$GSR is not supervision-matched in this setting because it uses group-labeled validation examples.
}
\label{tab:no_group_info}
\end{minipage}
\hfill
\begin{minipage}[t]{0.49\linewidth}
\vspace{0pt}
\centering
\small
\resizebox{0.95\textwidth}{!}{ %
\begin{tabular}{cc|cc|cc}
\toprule[2pt]
\multirow{2}{*}{\textbf{Algorithm}} &
\textbf{Group Info} &
\multicolumn{2}{c}{\textsc{GRACE}} &
\multicolumn{2}{c}{\textsc{MIMIC-IV}} \\
\cmidrule(lr){3-6}
& \textbf{(Train/Val)} &
\textbf{WGA} & \textbf{Avg Acc} &
\textbf{WGA} & \textbf{Avg Acc} \\
\midrule
ERM* & \xmark/\xmark
& $93.3\scriptstyle{\pm0.7}$ & $95.1\scriptstyle{\pm0.4}$
& $72.1\scriptstyle{\pm0.0}$ & $80.0\scriptstyle{\pm0.0}$ \\
IRM* & \xmark/\xmark
& $89.9\scriptstyle{\pm1.3}$ & $93.7\scriptstyle{\pm0.6}$
& $64.7\scriptstyle{\pm13.2}$ & $75.9\scriptstyle{\pm6.5}$ \\
CVaRDRO* & \xmark/\xmark
& $88.7\scriptstyle{\pm2.5}$ & $92.1\scriptstyle{\pm1.4}$
& $38.3\scriptstyle{\pm11.4}$ & $54.8\scriptstyle{\pm6.9}$ \\
JTT* & \xmark/\xmark
& $67.2\scriptstyle{\pm3.9}$ & $74.7\scriptstyle{\pm6.3}$
& $72.1\scriptstyle{\pm0.0}$ & $79.2\scriptstyle{\pm0.4}$ \\
LfF* & \xmark/\xmark
& $91.3\scriptstyle{\pm1.0}$ & $94.1\scriptstyle{\pm0.9}$
& $72.6\scriptstyle{\pm1.5}$ & $80.0\scriptstyle{\pm0.2}$ \\
Mixup* & \xmark/\xmark
& $92.3\scriptstyle{\pm0.8}$ & $94.5\scriptstyle{\pm0.5}$
& $72.1\scriptstyle{\pm0.0}$ & $80.0\scriptstyle{\pm0.1}$ \\
ReSample* & \xmark/\xmark
& $86.1\scriptstyle{\pm3.4}$ & $90.1\scriptstyle{\pm2.0}$
& $45.6\scriptstyle{\pm18.3}$ & $60.9\scriptstyle{\pm7.9}$ \\
CBLoss* & \xmark/\xmark
& $91.0\scriptstyle{\pm1.0}$ & $94.1\scriptstyle{\pm0.6}$
& $72.1\scriptstyle{\pm0.0}$ & $\underline{80.9}\scriptstyle{\pm0.3}$ \\
Focal* & \xmark/\xmark
& $91.3\scriptstyle{\pm1.1}$ & $94.2\scriptstyle{\pm0.6}$
& $72.1\scriptstyle{\pm0.0}$ & $80.7\scriptstyle{\pm0.2}$ \\
BSoftmax* & \xmark/\xmark
& $89.1\scriptstyle{\pm1.8}$ & $91.5\scriptstyle{\pm1.4}$
& $53.3\scriptstyle{\pm1.2}$ & $59.5\scriptstyle{\pm1.6}$ \\
AFR* & \xmark/\xmark
& $52.7\scriptstyle{\pm14.4}$ & $58.5\scriptstyle{\pm15.4}$
& $74.2\scriptstyle{\pm1.2}$ & $80.2\scriptstyle{\pm0.1}$ \\
SELF* & \xmark/\xmark
& $46.7\scriptstyle{\pm14.0}$ & $64.0\scriptstyle{\pm13.7}$
& $36.4\scriptstyle{\pm4.7}$ & $54.0\scriptstyle{\pm2.4}$ \\
EVaLS* & \xmark/\xmark
& $43.0\scriptstyle{\pm8.7}$ & $60.9\scriptstyle{\pm3.8}$
& $43.0\scriptstyle{\pm1.9}$ & $53.6\scriptstyle{\pm1.9}$ \\
GSR$^{\dagger}$* & \xmark/\bstar
& $\underline{93.7}\scriptstyle{\pm0.4}$ & $\mathbf{96.3}\scriptstyle{\pm0.1}$
& $\underline{77.3}\scriptstyle{\pm0.3}$ & $80.2\scriptstyle{\pm0.1}$ \\
EA* & \xmark/\xmark
& $85.3\scriptstyle{\pm1.1}$ & $91.8\scriptstyle{\pm0.8}$
& $57.6\scriptstyle{\pm3.3}$ & $63.6\scriptstyle{\pm1.2}$ \\
\rowcolor{teal!20}
\methodname{} (Ours) & \xmark/\xmark
& $\mathbf{94.2}\scriptstyle{\pm1.6}$ & $\underline{95.8}\scriptstyle{\pm1.3}$
& $\mathbf{78.2}\scriptstyle{\pm2.0}$ & $\mathbf{81.0}\scriptstyle{\pm0.2}$ \\
\bottomrule[1pt]
\end{tabular}
}
\end{minipage}
\vspace{0.75em}
\resizebox{\textwidth}{!}{%
\begin{tabular}{cc|cc|cc|cc|cc|cc}
\toprule[2pt]
\multirow{2}{*}{\textbf{Algorithm}} &
\textbf{Group Info} &
\multicolumn{2}{c}{\textsc{Waterbirds}} &
\multicolumn{2}{c}{\textsc{CelebA}} &
\multicolumn{2}{c}{\textsc{CivilComments}} &
\multicolumn{2}{c}{\textsc{MultiNLI}} &
\multicolumn{2}{c}{\textsc{CheXpert}} \\
\cmidrule(lr){3-12}
& \textbf{(Train/Val)} &
\textbf{WGA} & \textbf{Avg Acc} &
\textbf{WGA} & \textbf{Avg Acc} &
\textbf{WGA} & \textbf{Avg Acc} &
\textbf{WGA} & \textbf{Avg Acc} &
\textbf{WGA} & \textbf{Avg Acc} \\
\midrule
ERM* & \xmark/\xmark
& $69.1\scriptstyle{\pm4.7}$ & $84.1\scriptstyle{\pm1.7}$
& $57.6\scriptstyle{\pm0.8}$ & $95.0\scriptstyle{\pm0.1}$
& $63.2\scriptstyle{\pm1.2}$ & $85.4\scriptstyle{\pm0.2}$
& $66.4\scriptstyle{\pm2.3}$ & $81.0\scriptstyle{\pm0.3}$
& $41.7\scriptstyle{\pm3.4}$ & $\mathbf{88.6}\scriptstyle{\pm0.7}$ \\
CRT* & \xmark/\xmark
& $76.3\scriptstyle{\pm0.8}$ & $89.2\scriptstyle{\pm0.1}$
& $69.6\scriptstyle{\pm0.7}$ & $94.1\scriptstyle{\pm0.1}$
& $67.8\scriptstyle{\pm0.3}$ & $82.7\scriptstyle{\pm0.1}$
& $65.4\scriptstyle{\pm0.2}$ & $80.2\scriptstyle{\pm0.0}$
& $74.6\scriptstyle{\pm0.4}$ & $79.3\scriptstyle{\pm0.1}$ \\
ReWeightCRT* & \xmark/\xmark
& $76.1\scriptstyle{\pm0.3}$ & $89.5\scriptstyle{\pm0.2}$
& $70.4\scriptstyle{\pm0.8}$ & $94.1\scriptstyle{\pm0.2}$
& $64.6\scriptstyle{\pm0.1}$ & $82.2\scriptstyle{\pm0.2}$
& $65.3\scriptstyle{\pm0.3}$ & $80.0\scriptstyle{\pm0.1}$
& $74.8\scriptstyle{\pm0.4}$ & $79.2\scriptstyle{\pm0.0}$ \\
JTT* & \xmark/\xmark
& $71.2\scriptstyle{\pm0.5}$ & $88.9\scriptstyle{\pm0.6}$
& $48.3\scriptstyle{\pm1.5}$ & $\mathbf{95.9}\scriptstyle{\pm0.0}$
& $51.0\scriptstyle{\pm4.2}$ & $79.0\scriptstyle{\pm1.8}$
& $65.1\scriptstyle{\pm1.6}$ & $81.4\scriptstyle{\pm0.0}$
& $60.4\scriptstyle{\pm4.8}$ & $75.2\scriptstyle{\pm0.8}$ \\
CVaRDRO* & \xmark/\xmark
& $75.5\scriptstyle{\pm2.2}$ & $89.9\scriptstyle{\pm0.4}$
& $60.2\scriptstyle{\pm3.0}$ & $95.1\scriptstyle{\pm0.1}$
& $62.9\scriptstyle{\pm3.8}$ & $81.6\scriptstyle{\pm0.7}$
& $48.2\scriptstyle{\pm3.4}$ & $75.4\scriptstyle{\pm0.2}$
& $50.2\scriptstyle{\pm1.8}$ & $73.7\scriptstyle{\pm1.0}$ \\
DFR* & \xmark/\xmark
& $89.0\scriptstyle{\pm0.2}$ & $92.2\scriptstyle{\pm0.2}$
& $73.7\scriptstyle{\pm0.8}$ & $93.6\scriptstyle{\pm0.0}$
& $64.4\scriptstyle{\pm0.1}$ & $80.7\scriptstyle{\pm0.0}$
& $63.8\scriptstyle{\pm0.0}$ & $80.2\scriptstyle{\pm0.0}$
& $\underline{75.8}\scriptstyle{\pm0.3}$ & $79.1\scriptstyle{\pm0.0}$ \\
DPE & \xmark/\xmark
& $\mathbf{91.0}\scriptstyle{\pm0.5}$ & $\underline{92.5}\scriptstyle{\pm0.2}$
& $\underline{81.9}\scriptstyle{\pm0.2}$ & $89.8\scriptstyle{\pm0.2}$
& $\underline{69.9}\scriptstyle{\pm0.9}$ & $82.2\scriptstyle{\pm0.2}$
& $\underline{69.3}\scriptstyle{\pm0.8}$ & $81.3\scriptstyle{\pm0.2}$
& $-$ & $-$ \\
LISA* & \xmark/\xmark
& $77.5\scriptstyle{\pm0.7}$ & $89.2\scriptstyle{\pm0.6}$
& $57.8\scriptstyle{\pm0.8}$ & $95.4\scriptstyle{\pm0.1}$
& $65.8\scriptstyle{\pm1.5}$ & $84.6\scriptstyle{\pm0.1}$
& $66.8\scriptstyle{\pm0.3}$ & $81.7\scriptstyle{\pm0.1}$
& $37.4\scriptstyle{\pm3.5}$ & $81.9\scriptstyle{\pm6.2}$ \\
SELF* & \xmark/\xmark
& $77.4\scriptstyle{\pm4.4}$ & $89.5\scriptstyle{\pm1.7}$
& $43.6\scriptstyle{\pm10.2}$ & $56.1\scriptstyle{\pm14.2}$
& $36.9\scriptstyle{\pm23.7}$ & $73.5\scriptstyle{\pm15.0}$
& $20.9\scriptstyle{\pm6.9}$ & $45.6\scriptstyle{\pm9.2}$
& $45.0\scriptstyle{\pm4.5}$ & $47.4\scriptstyle{\pm5.8}$ \\
EVaLS* & \xmark/\xmark
& $80.7\scriptstyle{\pm4.8}$ & $89.0\scriptstyle{\pm2.6}$
& $58.3\scriptstyle{\pm4.8}$ & $79.0\scriptstyle{\pm4.4}$
& $38.9\scriptstyle{\pm20.0}$ & $67.7\scriptstyle{\pm4.3}$
& $40.6\scriptstyle{\pm9.8}$ & $68.3\scriptstyle{\pm4.3}$
& $69.1\scriptstyle{\pm1.3}$ & $71.0\scriptstyle{\pm1.6}$ \\
GSR$^{\dagger}$* & \xmark/\bstar
& $70.9\scriptstyle{\pm1.9}$ & $90.6\scriptstyle{\pm1.0}$
& $39.6\scriptstyle{\pm0.3}$ & $\underline{95.6}\scriptstyle{\pm0.1}$
& $60.0\scriptstyle{\pm3.7}$ & $\mathbf{92.5}\scriptstyle{\pm0.5}$
& $64.3\scriptstyle{\pm2.4}$ & $\mathbf{82.1}\scriptstyle{\pm0.1}$
& $\mathbf{76.2}\scriptstyle{\pm0.4}$ & $78.9\scriptstyle{\pm0.2}$ \\
EA* & \xmark/\xmark
& $78.9\scriptstyle{\pm1.9}$ & $89.3\scriptstyle{\pm1.1}$
& $65.7\scriptstyle{\pm4.1}$ & $93.0\scriptstyle{\pm0.2}$
& $48.4\scriptstyle{\pm2.0}$ & $86.8\scriptstyle{\pm1.2}$
& $65.6\scriptstyle{\pm2.4}$ & $\underline{82.0}\scriptstyle{\pm0.2}$
& $74.5\scriptstyle{\pm0.8}$ & $76.2\scriptstyle{\pm0.3}$ \\
\rowcolor{teal!20}
\methodname{} (Ours) & \xmark/\xmark
& $\underline{90.1}\scriptstyle{\pm0.1}$ & $\mathbf{93.7}\scriptstyle{\pm0.7}$
& $\mathbf{82.8}\scriptstyle{\pm0.5}$ & $95.2\scriptstyle{\pm0.1}$
& $\mathbf{71.6}\scriptstyle{\pm0.7}$ & $\underline{91.4}\scriptstyle{\pm0.4}$
& $\mathbf{69.7}\scriptstyle{\pm0.1}$ & $81.2\scriptstyle{\pm0.2}$
& $75.3\scriptstyle{\pm0.4}$ & $\underline{79.6}\scriptstyle{\pm0.2}$ \\
\bottomrule[1pt]
\end{tabular}
}
\vspace{-20pt}
\end{table*}

\subsection{Setting 2: Partial Knowledge of Subgroup Relevance} 
In Section \ref{sec:motivation}, we highlight the importance of considering unknown groups, as in many real-world scenarios, relevant subgroups will not be known \emph{a priori}. 
This setting simulates scenarios where all spurious features or sensitive attributes may \textbf{not} be known during training. We create a training and evaluation setup with the \textsc{CivilComments}, \textsc{CheXpert}, \textsc{GRACE} and \textsc{MIMIC-IV} datasets such that there is one known attribute in the validation set and evaluation is performed on the test set with respect to intersectional attribute groups (the known attribute $\times$ unknown attribute) or additional attribute groups (see App.~\ref{app:setting_details} for details). For \textsc{CivilComments}, the known attribute is one identity category at a time, such as Black, LGBTQ, or Muslim. Test WGA is computed over all identity groups, including identities that are not used for model selection. For \textsc{CheXpert}, \textsc{GRACE}, and \textsc{MIMIC-IV}, we reveal only one demographic attribute during validation. This is race or gender for CheXpert and MIMIC-IV, and age or sex for GRACE. We then evaluate over intersectional groups, such as race $\times$ gender or age $\times$ sex.

\begin{table*}[!h]
\centering
\small
\setlength{\tabcolsep}{3pt}
\renewcommand{\arraystretch}{1.05}
\vspace{0pt}
\caption{\textbf{Worst-group accuracy (WGA) and average accuracy on the test set under \underline{partial information of group relevance}.}}
\label{tab:partial_group_info}
\vspace{-0.5em}
\begin{minipage}[t]{0.40\textwidth}
\vspace{0pt}
\small
For each dataset, known groups are defined by the indicated ``Known Attribute'', while WGA is computed over all groups. Group Info (Train/Val) denotes group-label usage: \xmark{} = none, \cmark{} = training, \cmark{}\cmark{} = validation labels used to retrain, \bstar{} = validation labels for tuning/early stopping only. Best and second-best results are \textbf{bold} and \underline{underlined}. All results are the mean $\pm$ std. (in \%) over three runs.
\end{minipage}
\hfill
\begin{minipage}[t]{0.58\textwidth}
\vspace{0pt}
\resizebox{\textwidth}{!}{%
\begin{tabular}{cc|cc|cc||cc|cc}
\toprule[2pt]
\multirow{3}{*}{\textbf{Algorithm}} & \textbf{Group Info} &
\multicolumn{4}{c||}{\textsc{GRACE}} & \multicolumn{4}{c}{\textsc{MIMIC-IV}} \\
\cmidrule(lr){3-10}
& \textbf{Known Attribute} & \multicolumn{2}{c}{\textbf{Age}} & \multicolumn{2}{c||}{\textbf{Sex}} & \multicolumn{2}{c}{\textbf{Race}} & \multicolumn{2}{c}{\textbf{Gender}} \\
\cmidrule(lr){3-10}
& \textbf{(Train/Val)} & \textbf{WGA} & \textbf{Avg Acc} & \textbf{WGA} & \textbf{Avg Acc} & \textbf{WGA} & \textbf{Avg Acc} & \textbf{WGA} & \textbf{Avg Acc} \\
\midrule
ERM & \xmark/\bstar
& $91.1\scriptstyle{\pm1.1}$ & $\underline{95.1}\scriptstyle{\pm0.4}$
& $93.3\scriptstyle{\pm0.7}$ & $95.1\scriptstyle{\pm0.4}$
& $72.1\scriptstyle{\pm0.3}$ & $80.0\scriptstyle{\pm0.2}$
& $72.1\scriptstyle{\pm0.3}$ & $80.0\scriptstyle{\pm0.2}$ \\
IRM & \xmark/\bstar
& $44.8\scriptstyle{\pm9.6}$ & $70.3\scriptstyle{\pm6.1}$
& $50.2\scriptstyle{\pm8.1}$ & $73.6\scriptstyle{\pm5.4}$
& $72.4\scriptstyle{\pm0.5}$ & $80.0\scriptstyle{\pm0.3}$
& $71.8\scriptstyle{\pm0.4}$ & $80.2\scriptstyle{\pm0.2}$ \\
CVaRDRO & \xmark/\bstar
& $47.5\scriptstyle{\pm11.2}$ & $72.0\scriptstyle{\pm5.8}$
& $52.6\scriptstyle{\pm9.3}$ & $75.4\scriptstyle{\pm4.9}$
& $76.3\scriptstyle{\pm0.7}$ & $\underline{80.3}\scriptstyle{\pm0.1}$
& $76.4\scriptstyle{\pm0.6}$ & $\underline{80.3}\scriptstyle{\pm0.1}$ \\
JTT & \xmark/\bstar
& $86.4\scriptstyle{\pm4.3}$ & $92.6\scriptstyle{\pm1.1}$
& $90.2\scriptstyle{\pm3.7}$ & $93.7\scriptstyle{\pm1.2}$
& $72.0\scriptstyle{\pm0.4}$ & $80.1\scriptstyle{\pm0.2}$
& $72.2\scriptstyle{\pm0.3}$ & $79.8\scriptstyle{\pm0.4}$ \\
LfF & \xmark/\bstar
& $39.7\scriptstyle{\pm10.5}$ & $68.9\scriptstyle{\pm7.2}$
& $45.1\scriptstyle{\pm9.7}$ & $71.8\scriptstyle{\pm6.5}$
& $72.6\scriptstyle{\pm0.6}$ & $80.2\scriptstyle{\pm0.3}$
& $72.1\scriptstyle{\pm0.5}$ & $80.0\scriptstyle{\pm0.3}$ \\
Mixup & \xmark/\bstar
& $52.3\scriptstyle{\pm8.4}$ & $77.1\scriptstyle{\pm4.2}$
& $57.8\scriptstyle{\pm7.9}$ & $80.2\scriptstyle{\pm3.8}$
& $71.7\scriptstyle{\pm0.4}$ & $79.9\scriptstyle{\pm0.3}$
& $72.5\scriptstyle{\pm0.5}$ & $80.1\scriptstyle{\pm0.2}$ \\
ReSample & \xmark/\bstar
& $35.4\scriptstyle{\pm12.8}$ & $64.6\scriptstyle{\pm8.9}$
& $41.2\scriptstyle{\pm11.4}$ & $68.3\scriptstyle{\pm7.6}$
& $75.8\scriptstyle{\pm0.6}$ & $79.9\scriptstyle{\pm0.1}$
& $73.7\scriptstyle{\pm2.2}$ & $80.0\scriptstyle{\pm0.1}$ \\
Focal & \xmark/\bstar
& $49.8\scriptstyle{\pm9.1}$ & $75.8\scriptstyle{\pm4.7}$
& $54.6\scriptstyle{\pm8.2}$ & $78.4\scriptstyle{\pm4.1}$
& $72.2\scriptstyle{\pm0.3}$ & $80.0\scriptstyle{\pm0.2}$
& $71.9\scriptstyle{\pm0.4}$ & $80.1\scriptstyle{\pm0.3}$ \\
LDAM & \xmark/\bstar
& $56.9\scriptstyle{\pm7.5}$ & $81.4\scriptstyle{\pm3.6}$
& $61.7\scriptstyle{\pm6.8}$ & $84.2\scriptstyle{\pm3.1}$
& $\underline{77.7}\scriptstyle{\pm0.1}$ & $\mathbf{80.4}\scriptstyle{\pm0.0}$
& $\underline{77.7}\scriptstyle{\pm0.1}$ & $\mathbf{80.4}\scriptstyle{\pm0.0}$ \\
BSoftmax & \xmark/\bstar
& $42.1\scriptstyle{\pm10.7}$ & $69.5\scriptstyle{\pm6.8}$
& $48.3\scriptstyle{\pm9.9}$ & $72.7\scriptstyle{\pm5.9}$
& $73.5\scriptstyle{\pm2.2}$ & $80.1\scriptstyle{\pm0.0}$
& $73.4\scriptstyle{\pm2.3}$ & $80.1\scriptstyle{\pm0.0}$ \\
AFR & \xmark/\bstar
& $84.0\scriptstyle{\pm0.8}$ & $90.9\scriptstyle{\pm0.6}$
& $87.6\scriptstyle{\pm0.7}$ & $90.9\scriptstyle{\pm0.6}$
& $75.3\scriptstyle{\pm0.9}$ & $80.2\scriptstyle{\pm0.1}$
& $73.0\scriptstyle{\pm1.7}$ & $80.1\scriptstyle{\pm0.1}$ \\
SELF* & \xmark/\bstar
& $69.2\scriptstyle{\pm2.9}$ & $77.9\scriptstyle{\pm0.9}$
& $70.8\scriptstyle{\pm0.7}$ & $79.4\scriptstyle{\pm2.9}$
& $62.5\scriptstyle{\pm5.6}$ & $75.1\scriptstyle{\pm0.9}$
& $54.7\scriptstyle{\pm7.2}$ & $67.6\scriptstyle{\pm4.4}$ \\
EVaLS* & \xmark/\bstar
& $75.8\scriptstyle{\pm0.9}$ & $82.8\scriptstyle{\pm1.8}$
& $78.9\scriptstyle{\pm11.9}$ & $86.2\scriptstyle{\pm8.1}$
& $63.8\scriptstyle{\pm4.6}$ & $72.8\scriptstyle{\pm0.2}$
& $68.0\scriptstyle{\pm4.5}$ & $74.4\scriptstyle{\pm5.2}$ \\
GSR* & \xmark/\bstar
& $83.7\scriptstyle{\pm0.4}$ & $93.3\scriptstyle{\pm0.2}$
& $\underline{93.6}\scriptstyle{\pm0.4}$ & $\underline{96.3}\scriptstyle{\pm0.2}$
& $75.6\scriptstyle{\pm2.5}$ & $80.2\scriptstyle{\pm0.1}$
& $77.3\scriptstyle{\pm0.2}$ & $\underline{80.3}\scriptstyle{\pm0.0}$ \\
EA* & \xmark/\bstar
& $\underline{91.7}\scriptstyle{\pm0.8}$ & $\mathbf{95.3}\scriptstyle{\pm0.2}$
& $91.6\scriptstyle{\pm0.6}$ & $95.3\scriptstyle{\pm0.2}$
& $77.2\scriptstyle{\pm0.3}$ & $\underline{80.3}\scriptstyle{\pm0.0}$
& $77.3\scriptstyle{\pm0.2}$ & $80.2\scriptstyle{\pm0.4}$ \\
\rowcolor{teal!20}
\methodname{} (Ours) & \xmark/\bstar
& $\mathbf{92.6}\scriptstyle{\pm3.7}$ & $\mathbf{95.3}\scriptstyle{\pm1.7}$
& $\mathbf{95.0}\scriptstyle{\pm0.4}$ & $\mathbf{96.5}\scriptstyle{\pm0.3}$
& $\mathbf{78.5}\scriptstyle{\pm1.9}$ & $80.1\scriptstyle{\pm0.0}$
& $\mathbf{79.1}\scriptstyle{\pm0.6}$ & $80.0\scriptstyle{\pm0.7}$ \\
\bottomrule[1pt]
\end{tabular}
}
\end{minipage}
\vspace{0.25em}
\resizebox{\textwidth}{!}{%
\begin{tabular}{cc|cc|cc||cc|cc|cc}
\toprule[2pt]
\multirow{3}{*}{\textbf{Algorithm}} & \textbf{Group Info} &
\multicolumn{4}{c}{\textsc{CheXpert}} $\vert\vert$& \multicolumn{6}{c}{\textsc{CivilComments}}\\
\cmidrule(lr){3-12}
& \textbf{Known Attribute} & \multicolumn{2}{c}{\textbf{Race}} $\vert$ & \multicolumn{2}{c}{\textbf{Gender}} $\vert\vert$& \multicolumn{2}{c}{\textbf{Black}} $\vert$& \multicolumn{2}{c}{\textbf{LGBTQ}} $\vert$& \multicolumn{2}{c}{\textbf{Muslim}} \\
\cmidrule(lr){3-12}
& \textbf{(Train/Val)} & \textbf{WGA} & \textbf{Avg Acc} & \textbf{WGA} & \textbf{Avg Acc} & \textbf{WGA} & \textbf{Avg Acc} & \textbf{WGA} & \textbf{Avg Acc} & \textbf{WGA} & \textbf{Avg Acc} \\
\midrule
ERM & \xmark/\bstar
& $41.7\scriptstyle{\pm3.4}$ & $\underline{88.6}\scriptstyle{\pm0.7}$
& $41.7\scriptstyle{\pm3.4}$ & $\underline{88.6}\scriptstyle{\pm0.7}$
& $48.0\scriptstyle{\pm7.5}$ & $\underline{92.3}\scriptstyle{\pm0.1}$
& $48.0\scriptstyle{\pm7.5}$ & $\underline{92.3}\scriptstyle{\pm0.1}$
& $48.0\scriptstyle{\pm7.5}$ & $\underline{92.3}\scriptstyle{\pm0.1}$ \\
CRT & \xmark/\bstar
& $71.1\scriptstyle{\pm6.9}$ & $80.9\scriptstyle{\pm1.7}$
& $71.2\scriptstyle{\pm6.9}$ & $80.8\scriptstyle{\pm1.8}$
& $62.8\scriptstyle{\pm3.5}$ & $90.4\scriptstyle{\pm1.9}$
& $63.1\scriptstyle{\pm3.5}$ & $90.6\scriptstyle{\pm1.6}$
& $61.5\scriptstyle{\pm2.5}$ & $90.0\scriptstyle{\pm2.2}$ \\
ReWeightCRT & \xmark/\bstar
& $71.3\scriptstyle{\pm6.1}$ & $80.8\scriptstyle{\pm1.9}$
& $71.3\scriptstyle{\pm6.1}$ & $80.7\scriptstyle{\pm1.9}$
& $61.9\scriptstyle{\pm3.2}$ & $90.7\scriptstyle{\pm1.6}$
& $63.4\scriptstyle{\pm3.5}$ & $90.7\scriptstyle{\pm1.3}$
& $62.1\scriptstyle{\pm2.3}$ & $90.4\scriptstyle{\pm1.7}$ \\
JTT & \xmark/\bstar
& $58.7\scriptstyle{\pm15.0}$ & $84.5\scriptstyle{\pm4.0}$
& $58.7\scriptstyle{\pm15.0}$ & $84.5\scriptstyle{\pm4.0}$
& $57.8\scriptstyle{\pm2.4}$ & $91.9\scriptstyle{\pm0.2}$
& $57.7\scriptstyle{\pm2.5}$ & $91.9\scriptstyle{\pm0.2}$
& $57.6\scriptstyle{\pm2.7}$ & $91.9\scriptstyle{\pm0.2}$ \\
CVaRDRO & \xmark/\bstar
& $37.1\scriptstyle{\pm17.6}$ & $88.0\scriptstyle{\pm4.2}$
& $37.1\scriptstyle{\pm17.6}$ & $88.0\scriptstyle{\pm4.2}$
& $57.1\scriptstyle{\pm3.1}$ & $91.9\scriptstyle{\pm0.2}$
& $55.8\scriptstyle{\pm3.2}$ & $92.1\scriptstyle{\pm0.3}$
& $56.2\scriptstyle{\pm2.8}$ & $92.0\scriptstyle{\pm0.2}$ \\
AFR & \xmark/\bstar
& $63.9\scriptstyle{\pm13.1}$ & $82.9\scriptstyle{\pm3.7}$
& $63.9\scriptstyle{\pm13.1}$ & $82.9\scriptstyle{\pm3.7}$
& $38.1\scriptstyle{\pm11.8}$ & $92.0\scriptstyle{\pm0.7}$
& $38.4\scriptstyle{\pm11.7}$ & $92.0\scriptstyle{\pm0.7}$
& $38.2\scriptstyle{\pm11.8}$ & $92.0\scriptstyle{\pm0.7}$ \\
DFR & \xmark/\cmark\cmark
& $63.9\scriptstyle{\pm10.6}$ & $72.3\scriptstyle{\pm6.0}$
& $63.4\scriptstyle{\pm11.5}$ & $72.3\scriptstyle{\pm6.0}$
& $66.5\scriptstyle{\pm2.6}$ & $90.6\scriptstyle{\pm0.6}$
& $64.0\scriptstyle{\pm1.6}$ & $90.1\scriptstyle{\pm1.6}$
& $\underline{64.1}\scriptstyle{\pm4.5}$ & $89.2\scriptstyle{\pm1.1}$ \\
Group DRO & \cmark/\bstar
& $69.6\scriptstyle{\pm7.6}$ & $80.9\scriptstyle{\pm1.9}$
& $70.0\scriptstyle{\pm7.5}$ & $80.9\scriptstyle{\pm1.9}$
& $\underline{66.7}\scriptstyle{\pm1.9}$ & $89.5\scriptstyle{\pm2.3}$
& $62.9\scriptstyle{\pm4.0}$ & $89.3\scriptstyle{\pm2.5}$
& $63.3\scriptstyle{\pm3.4}$ & $89.4\scriptstyle{\pm2.3}$ \\
IRM & \cmark/\bstar
& $71.4\scriptstyle{\pm7.2}$ & $81.0\scriptstyle{\pm1.7}$
& $71.3\scriptstyle{\pm7.0}$ & $80.9\scriptstyle{\pm1.8}$
& $\mathbf{67.0}\scriptstyle{\pm1.6}$ & $90.0\scriptstyle{\pm1.4}$
& $\underline{64.4}\scriptstyle{\pm1.6}$ & $90.3\scriptstyle{\pm1.5}$
& $63.6\scriptstyle{\pm3.2}$ & $89.3\scriptstyle{\pm1.5}$ \\
LISA & \cmark/\bstar
& $39.5\scriptstyle{\pm2.7}$ & $\mathbf{89.4}\scriptstyle{\pm0.5}$
& $44.7\scriptstyle{\pm3.1}$ & $\mathbf{88.7}\scriptstyle{\pm0.4}$
& $61.7\scriptstyle{\pm5.7}$ & $91.3\scriptstyle{\pm0.7}$
& $59.2\scriptstyle{\pm2.5}$ & $91.7\scriptstyle{\pm0.2}$
& $59.4\scriptstyle{\pm5.1}$ & $91.6\scriptstyle{\pm0.6}$ \\
SELF* & \xmark/\bstar
& $66.8\scriptstyle{\pm5.2}$ & $69.6\scriptstyle{\pm4.8}$
& $71.4\scriptstyle{\pm0.5}$ & $73.8\scriptstyle{\pm1.3}$
& $54.3\scriptstyle{\pm4.4}$ & $74.8\scriptstyle{\pm16.3}$
& $63.9\scriptstyle{\pm9.4}$ & $86.6\scriptstyle{\pm2.1}$
& $55.4\scriptstyle{\pm6.2}$ & $79.8\scriptstyle{\pm13.0}$ \\
EVaLS* & \xmark/\bstar
& $73.2\scriptstyle{\pm1.4}$ & $76.2\scriptstyle{\pm1.2}$
& $\underline{73.9}\scriptstyle{\pm1.2}$ & $77.0\scriptstyle{\pm0.5}$
& $59.4\scriptstyle{\pm11.9}$ & $73.4\scriptstyle{\pm18.6}$
& $62.9\scriptstyle{\pm7.9}$ & $88.0\scriptstyle{\pm2.0}$
& $58.6\scriptstyle{\pm16.3}$ & $75.5\scriptstyle{\pm6.9}$ \\
GSR* & \xmark/\bstar
& $\underline{74.0}\scriptstyle{\pm0.5}$ & $75.9\scriptstyle{\pm0.3}$
& $66.1\scriptstyle{\pm0.4}$ & $78.9\scriptstyle{\pm0.2}$
& $60.2\scriptstyle{\pm3.5}$ & $\mathbf{92.5}\scriptstyle{\pm0.5}$
& $60.1\scriptstyle{\pm3.6}$ & $\mathbf{92.5}\scriptstyle{\pm0.5}$
& $59.9\scriptstyle{\pm3.7}$ & $\mathbf{92.6}\scriptstyle{\pm0.5}$ \\
EA* & \xmark/\bstar
& $\mathbf{75.8}\scriptstyle{\pm0.9}$ & $78.6\scriptstyle{\pm0.3}$
& $\mathbf{75.6}\scriptstyle{\pm0.8}$ & $78.9\scriptstyle{\pm0.3}$
& $62.1\scriptstyle{\pm1.1}$ & $91.3\scriptstyle{\pm0.7}$
& $63.7\scriptstyle{\pm2.7}$ & $90.5\scriptstyle{\pm0.2}$
& $59.4\scriptstyle{\pm2.0}$ & $90.5\scriptstyle{\pm0.9}$ \\
\rowcolor{teal!20}
\methodname{} (Ours) & \xmark/\bstar
& $73.4\scriptstyle{\pm1.9}$ & $78.7\scriptstyle{\pm0.8}$
& $73.3\scriptstyle{\pm1.7}$ & $78.5\scriptstyle{\pm1.2}$
& $66.1\scriptstyle{\pm1.2}$ & $90.2\scriptstyle{\pm0.7}$
& $\mathbf{66.8}\scriptstyle{\pm1.9}$ & $90.1\scriptstyle{\pm0.7}$
& $\mathbf{65.2}\scriptstyle{\pm1.8}$ & $89.5\scriptstyle{\pm0.4}$ \\
\bottomrule[1pt]
\end{tabular}
}
\vspace{-10pt}
\end{table*}

In Table \ref{tab:partial_group_info}, \methodname{} achieves the highest WGA in six of nine settings after including the additional SELF, EVaLS, GSR, and EA comparisons. The largest remaining gaps occur on the two CheXpert settings, while \methodname{} remains strongest on both GRACE and MIMIC-IV settings and on two of the three CivilComments attributes. Competitive group-aware baselines such as AFR, DFR, and Group DRO can also degrade substantially relative to their complete-information performance when only partial subgroup information is available.

\setlength{\tabcolsep}{5pt}
\begin{table*}[!h]
\centering
\begin{minipage}[t]{0.49\textwidth}
\vspace{0pt}
\caption{
\textbf{Worst-group accuracy (WGA) and average accuracy on the test set with \underline{complete information of group relevance.}} Rows with asterisks (*) and CheXpert results reproduce baselines.
Group Info (Train/Val): \xmark{} = no group info used, \bstar{} = used for tuning/early stopping.
Best and second-best values within each half are \textbf{bold} and \underline{underlined}.
``-'' indicates results not reported.
All results are mean $\pm$ std (in \%) over three runs.
}
\label{tab:complete_group_info}
\end{minipage}
\hfill
\begin{minipage}[t]{0.49\linewidth}
\vspace{0pt}
\centering
\small
\resizebox{0.95\textwidth}{!}{%
\begin{tabular}{cc|cc|cc}
\toprule[2pt]
\multirow{2}{*}{\textbf{Algorithm}} &
\textbf{Group Info} &
\multicolumn{2}{c}{\textsc{GRACE}} &
\multicolumn{2}{c}{\textsc{MIMIC-IV}} \\
\cmidrule(lr){3-6}
&
\textbf{(Train/Val)} &
\textbf{WGA} & \textbf{Avg Acc} &
\textbf{WGA} & \textbf{Avg Acc} \\
\midrule
ERM* & \xmark/\bstar 
& $\underline{93.3\scriptstyle{\pm0.7}}$ & $\underline{95.1\scriptstyle{\pm0.4}}$ 
& $72.1\scriptstyle{\pm0.3}$ & $80.0\scriptstyle{\pm0.2}$ \\

CVaRDRO* & \xmark/\bstar 
& $88.9\scriptstyle{\pm2.3}$ & $93.5\scriptstyle{\pm1.4}$ 
& $\underline{76.2\scriptstyle{\pm0.7}}$ & $80.0\scriptstyle{\pm0.6}$ \\

JTT* & \xmark/\bstar 
& $90.2\scriptstyle{\pm3.7}$ & $93.7\scriptstyle{\pm1.2}$ 
& $71.9\scriptstyle{\pm0.4}$ & $79.9\scriptstyle{\pm0.3}$ \\

LfF* & \xmark/\bstar 
& $87.5\scriptstyle{\pm2.8}$ & $92.6\scriptstyle{\pm1.6}$ 
& $72.4\scriptstyle{\pm0.5}$ & $80.1\scriptstyle{\pm0.2}$ \\

Mixup* & \xmark/\bstar 
& $88.3\scriptstyle{\pm2.5}$ & $92.9\scriptstyle{\pm1.5}$ 
& $71.8\scriptstyle{\pm0.4}$ & $79.8\scriptstyle{\pm0.3}$ \\

CBLoss* & \xmark/\bstar 
& $87.9\scriptstyle{\pm2.6}$ & $92.8\scriptstyle{\pm1.5}$ 
& $72.2\scriptstyle{\pm0.3}$ & $80.1\scriptstyle{\pm0.2}$ \\

Focal* & \xmark/\bstar 
& $87.6\scriptstyle{\pm2.7}$ & $92.7\scriptstyle{\pm1.6}$ 
& $71.9\scriptstyle{\pm0.4}$ & $79.9\scriptstyle{\pm0.3}$ \\

BSoftmax* & \xmark/\bstar 
& $86.2\scriptstyle{\pm3.1}$ & $91.8\scriptstyle{\pm1.9}$ 
& $76.1\scriptstyle{\pm0.1}$ & $79.1\scriptstyle{\pm0.0}$ \\

CRT* & \xmark/\bstar 
& $86.8\scriptstyle{\pm3.0}$ & $92.0\scriptstyle{\pm1.8}$ 
& $76.0\scriptstyle{\pm0.4}$ & $79.4\scriptstyle{\pm0.4}$ \\

ReWeightCRT* & \xmark/\bstar 
& $85.9\scriptstyle{\pm3.6}$ & $91.5\scriptstyle{\pm2.1}$ 
& $73.8\scriptstyle{\pm3.0}$ & $\underline{80.2\scriptstyle{\pm0.2}}$ \\

VanillaCRT* & \xmark/\bstar 
& $87.1\scriptstyle{\pm2.9}$ & $92.2\scriptstyle{\pm1.7}$ 
& $72.1\scriptstyle{\pm0.3}$ & $80.0\scriptstyle{\pm0.2}$ \\

AFR* & \xmark/\bstar 
& $87.6\scriptstyle{\pm0.7}$ & $90.9\scriptstyle{\pm0.6}$ 
& $75.3\scriptstyle{\pm0.9}$ & $80.2\scriptstyle{\pm0.1}$ \\

\rowcolor{teal!20}
\methodname{} (Ours) & \xmark/\bstar 
& $\mathbf{95.0}\scriptstyle{\pm0.4}$ & $\mathbf{96.5}\scriptstyle{\pm0.3}$ 
& $\mathbf{79.0}\scriptstyle{\pm0.7}$ & $\mathbf{81.3}\scriptstyle{\pm0.2}$ \\

\bottomrule[1pt]
\end{tabular}
}
\end{minipage}
\vspace{-0.1em}

\resizebox{\textwidth}{!}{%
\begin{tabular}{cc|cc|cc|cc|cc|cc}
\toprule[2pt]
\multirow{2}{*}{\textbf{Algorithm}} &
\textbf{Group Info} &
\multicolumn{2}{c}{\textsc{Waterbirds}} &
\multicolumn{2}{c}{\textsc{CelebA}} &
\multicolumn{2}{c}{\textsc{CivilComments}} &
\multicolumn{2}{c}{\textsc{MultiNLI}} &
\multicolumn{2}{c}{\textsc{CheXpert}}
\\
\cmidrule(lr){3-12}
&
\textbf{(Train/Val)} &
\textbf{WGA} & \textbf{Avg Acc} &
\textbf{WGA} & \textbf{Avg Acc} &
\textbf{WGA} & \textbf{Avg Acc} &
\textbf{WGA} & \textbf{Avg Acc} &
\textbf{WGA} & \textbf{Avg Acc} \\
\midrule
ERM* & \xmark/\bstar
& $69.1\scriptstyle{\pm4.7}$ & $84.1\scriptstyle{\pm1.7}$
    &
    $57.6\scriptstyle{\pm0.8}$ & $\underline{95.0}\scriptstyle{\pm0.1}$
    &
    $63.2\scriptstyle{\pm1.2}$ & $85.4\scriptstyle{\pm0.2}$
    &
    $69.5\scriptstyle{\pm0.3}$ & $80.9\scriptstyle{\pm0.3}$
    &
    $41.7\scriptstyle{\pm3.4}$ & $\textbf{88.6}\scriptstyle{\pm0.7}$ \\
CRT* & \xmark/\bstar
& $76.3\scriptstyle{\pm0.8}$ & $89.2\scriptstyle{\pm0.1}$
& $70.4\scriptstyle{\pm}0.4$ & $94.1\scriptstyle{\pm0.1}$
& $68.5\scriptstyle{\pm0.0}$ & $83.0\scriptstyle{\pm0.0}$
& $65.4\scriptstyle{\pm0.1}$ & $80.2\scriptstyle{\pm0.0}$
& $74.0\scriptstyle{\pm0.2}$ & $79.1\scriptstyle{\pm0.1}$\\

ReWeightCRT* & \xmark/\bstar
& $76.3\scriptstyle{\pm0.2}$ & $89.4\scriptstyle{\pm0.3}$
& $71.1\scriptstyle{\pm0.5}$ & $94.2\scriptstyle{\pm0.1}$
& $68.2\scriptstyle{\pm0.4}$ & $83.4\scriptstyle{\pm0.0}$
& $65.3\scriptstyle{\pm0.1}$ & $80.2\scriptstyle{\pm0.0}$
& $73.9\scriptstyle{\pm0.2}$ & $79.0\scriptstyle{\pm0.1}$ \\

JTT & \xmark/\bstar
& $86.7$ & $93.3$
& $81.1$ & $88.0$
& $69.3$ & $\textbf{91.1}$
& $\underline{72.6}$ & $78.6$
& $60.4\scriptstyle{\pm4.8}$ & $75.2\scriptstyle{\pm0.8}$ \\

CVaRDRO* & \xmark/\bstar
& $75.5\scriptstyle{\pm}2.2$ 
& $89.9\scriptstyle{\pm}0.4$  &
 $62.2\scriptstyle{\pm}3.1$ &
 $95.1\scriptstyle{\pm}0.1$ &
 $68.7\scriptstyle{\pm}1.3$ &
 $83.5\scriptstyle{\pm}0.3$ &
$63.0\scriptstyle{\pm}1.5$ &
$75.1\scriptstyle{\pm}0.1$ &
$50.2\scriptstyle{\pm}1.8$ &
$73.7\scriptstyle{\pm}1.0$
 \\

CnC & \xmark/\bstar
& $88.5\scriptstyle{\pm}0.3$ & $90.9\scriptstyle{\pm}0.1$
& $\underline{88.8}\scriptstyle{\pm}0.9$ & $89.9\scriptstyle{\pm}0.5$
& $68.9\scriptstyle{\pm}2.1$ & $81.7\scriptstyle{\pm}0.5$
& $-$ & $-$
& $-$ & $-$ \\

AFR & \xmark/\bstar
& $\underline{90.4}\scriptstyle{\pm}1.1$ & $\textbf{94.2}\scriptstyle{\pm}1.2$
& $82.0\scriptstyle{\pm}0.5$ & $91.3\scriptstyle{\pm}0.3$
& $68.7\scriptstyle{\pm}0.6$ 
& $89.8\scriptstyle{\pm}0.6$
& $\textbf{73.4}\scriptstyle{\pm}0.6$ 
& $\underline{81.4}\scriptstyle{\pm}0.2$ 
& $-$ & $-$ \\

GIC & \xmark/\bstar
& $86.3\scriptstyle{\pm}0.1$ & $89.6\scriptstyle{\pm}1.3$
& $\textbf{89.4}\scriptstyle{\pm}0.2$ & $91.9\scriptstyle{\pm}0.1$
& $\underline{72.5}\scriptstyle{\pm}0.3$ & $90.0\scriptstyle{\pm}0.3$
& $-$ & $-$  
& $-$ & $-$\\

\rowcolor{teal!20} \methodname{} (Ours) & \xmark/\bstar
& $\textbf{90.7}\scriptstyle{\pm}0.2$ & $\underline{93.3}\scriptstyle{\pm}0.7$
& $85.0\scriptstyle{\pm}0.9$ & $\textbf{95.2}\scriptstyle{\pm}0.1$
& $\textbf{72.9}\scriptstyle{\pm}0.2$ &
$\underline{90.9}\scriptstyle{\pm}0.8$
& $69.6\scriptstyle{\pm}0.8$ 
& $\textbf{81.4}\scriptstyle{\pm}0.4$ 
& $\textbf{75.2}\scriptstyle{\pm}0.3$ & $\underline{79.6}\scriptstyle{\pm}0.3$ \\
\bottomrule[1pt]
\end{tabular}
}
\vspace{-5pt}
\end{table*}

\subsection{Setting 3: Complete Knowledge of Subgroup Relevance}
Finally, to compare with recent works \citep{kirichenko2022dfr, qiu2023simplefastgrouprobustness, to2025diverseprototypicalensemblesimprove}, we consider \textit{complete knowledge of subgroup relevance}, where all evaluation-relevant groups are available for model selection. This setting gives group-aware methods stronger supervision than our primary no- and partial-information regimes and is included for comparability with prior work. We apply hyperparameter selection and early stopping using WGA over all subgroups. \methodname{} remains competitive, while group-aware methods can be preferable when the complete deployment-relevant group structure is genuinely known (Table~\ref{tab:complete_group_info} and App.~\ref{app:complete_group_oracle}). 
\vspace{-10pt}

\section{Synthetic Data Experiments}
\label{sec:interpretation}
\vspace{-10pt}
To understand how, and when, MARLA is expected to be performant, we analyze how \methodname{}'s weighting mechanism and rank-constrained correction interact using a controlled binary classification task setup with $K$ latent subgroup relevant directions (App.~\ref{app:synthetic_lora}). The majority group is classified by a content feature, while each of the $K$ minority subgroups requires a distinct direction.  

\begin{figure}[!h]
\includegraphics[width=1.0\linewidth]{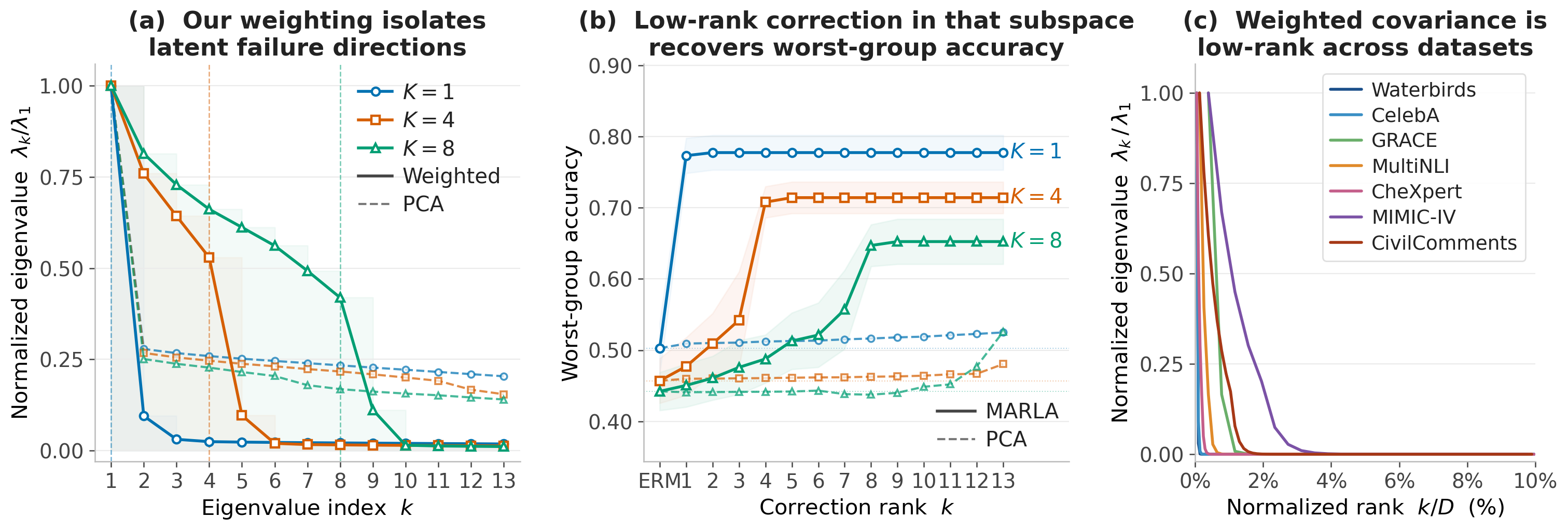} 
\caption{\textbf{\methodname{} performance under variable rank conditions}
(a) The score-weighted covariance spectrum (a central component of \methodname{})  exhibits a sharp elbow exactly at true rank $K$ 
(solid curves), while PCA on the unweighted covariance matrix shows shallow non-specific decay (dashed curves).
(b) Worst-group accuracy recovers only once the correction rank (the number of eigenvectors used in \methodname{}) reaches $K$; 
PCA-directed correction yields no improvement at any rank. (c) Seven benchmark datasets exhibit concentrated score-weighted covariance spectra.}
\label{fig:synthetic_score_covariance}
\vspace{-15pt}
\end{figure}
\paragraph{Weighting Mechanism Isolates Failure Modes} First, we test whether the weighting mechanism is needed to select a useful subspace through this toy example. As a baseline, we apply the same rank-constrained correction along the top unweighted PCA directions. The PCA uses the same rank $k$ as  \methodname{}, but selects directions by total feature variance rather than by 
the misclassification score. In the synthetic task, these PCA directions  are dominated by majority content and nuisance variation, so PCA-directed correction yields no WGA improvement at any rank (Figure~\ref{fig:synthetic_score_covariance}b, dashed). By contrast, using \methodname{} produces a 
sharp spectral elbow at rank $K$ 
(Figure~\ref{fig:synthetic_score_covariance}a), indicating that the weighting 
mechanism concentrates the covariance on the subgroup-relevant directions.
\vspace{-8pt}
\paragraph{WGA Improves with Sufficient Rank Update, Contingent on Weighting} Given this selected subspace, the correction rank determines whether 
\methodname{} can span the relevant directions. When $K=1$, a rank-one 
correction recovers WGA. As $K$ increases, lower-rank corrections under-correct, 
and WGA plateaus only once the rank reaches $K$ 
(Figure~\ref{fig:synthetic_score_covariance}b). Thus, \methodname{} succeeds when the weighting mechanism selects a concentrated subspace and the correction rank is large enough to span it. \methodname{} works well on the evaluation benchmarks, as these datasets have concentrated score-weighted covariance spectra. For five out of seven datasets,  fewer than 0.5\% of available eigenvectors capture over 95\% of the variance  under the score-weighted covariance (Figure~\ref{fig:synthetic_score_covariance}c).
\vspace{-8pt}
\paragraph{Ablation with Full-rank Update on Real Datasets} Finally, we compare \methodname{} to a Full-rank variant that fine-tunes the 
entire classifier head. Full-rank has greater expressivity, but can move the 
classifier in directions unrelated to the low true-label probability examples 
that define the selected subspace. Table~\ref{tab:rank_ablation_table} shows that this rank-limited constraint is  consistently beneficial on five out of five datasets.

\begin{table}[!h]
\centering
\small
\setlength{\tabcolsep}{4pt}
\caption{\textbf{Ablation study of low-rank vs full-rank correction:} 
Worst-group accuracy (WGA) and average accuracy (Avg Acc) in the no 
group-relevance information setting. ERM values match the corresponding 
no-group-information ERM results in Table~\ref{tab:no_group_info}.}
\label{tab:rank_ablation_table}
\resizebox{\linewidth}{!}{%
\begin{tabular}{l *{6}{cc}}
\toprule
& \multicolumn{2}{c}{\textsc{Waterbirds}}
& \multicolumn{2}{c}{\textsc{MultiNLI}}
& \multicolumn{2}{c}{\textsc{CelebA}}
& \multicolumn{2}{c}{\textsc{CivilComments}}
& \multicolumn{2}{c}{\textsc{CheXpert}} \\
\cmidrule(lr){2-3}\cmidrule(lr){4-5}\cmidrule(lr){6-7}
\cmidrule(lr){8-9}\cmidrule(lr){10-11}
Method & WGA & Avg Acc & WGA & Avg Acc & WGA & Avg Acc & WGA & Avg Acc & WGA & Avg Acc \\
\midrule
\rowcolor{teal!20}\methodname{} (Low-rank)
  & $\bf 90.1_{\pm 0.1}$ & $\bf 93.7_{\pm 0.7}$
  & $\bf 69.7_{\pm 0.1}$ & $\bf 81.2_{\pm 0.2}$
  & $\bf 82.8_{\pm 0.5}$ & $\bf 95.2_{\pm 0.1}$
  & $\bf 71.6_{\pm 0.7}$ & $\bf 91.4_{\pm 0.4}$
  & $\bf 75.3_{\pm 0.4}$ & $79.6_{\pm 0.2}$\\
Full-rank
  & $82.0_{\pm 4.5}$ & $92.0_{\pm 2.0}$
  & $ 69.7_{\pm 1.7}$ & $81.0_{\pm 0.5}$
  & $81.3_{\pm 0.4}$ & $90.2_{\pm 0.1}$
  & $68.7_{\pm 1.3}$ & $91.0_{\pm 0.6}$
  & $56.4_{\pm 9.1}$ & $79.1_{\pm 0.8}$
\\
ERM
  & $69.1_{\pm 4.7}$ & $84.1_{\pm 1.7}$
  & $66.4_{\pm 2.3}$ & $81.0_{\pm 0.3}$
  & $57.6_{\pm 0.8}$ & $95.0_{\pm 0.1}$
  & $63.2_{\pm 1.2}$ & $85.4_{\pm 0.2}$
  & $41.7_{\pm 3.4}$ & $88.6_{\pm 0.7}$
\\
\bottomrule
\end{tabular}%
}
\vspace{-10pt}
\end{table}
\vspace{-5pt}

\section{MARLA's Advantages for Real-World Deployment}
\label{sec:advantages}
\methodname{}'s  performance in all three settings is especially compelling for real-world settings given how lightweight and fast it is compared to baselines. 

\paragraph{\methodname{} is Fast and Parameter-Efficient}
In Figure \ref{fig:runtime_params}, we visualize the training time required by \methodname{} and baselines on four datasets
(see App.~\ref{app:runtime}). We also highlight that \methodname{} is a particularly attractive in settings where computational budget, memory, or training stability are key constraints, such as use cases requiring larger models (see App.~\ref{app:large_models}). This speed and efficiency is also valuable in settings like healthcare where models must be updated with new cohort information \citep{10.3389/fdgth.2022.958284, otokiti2023need, guo2021systematic}.   

\begin{figure}[!h]
    \centering
    \includegraphics[width=\linewidth]{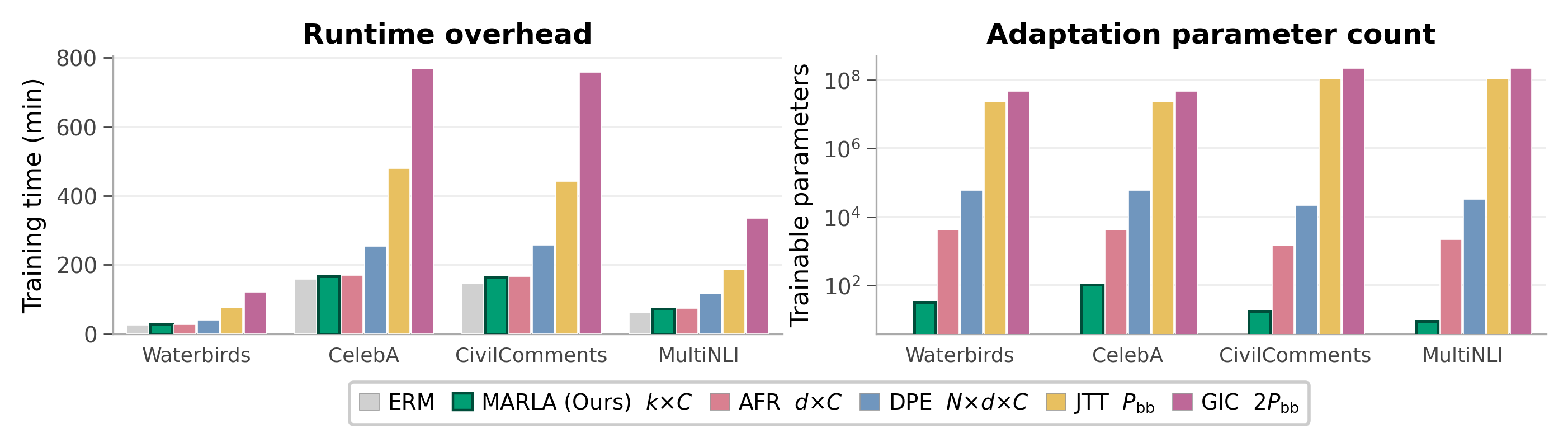}
    \caption{(\textit{Left}) \textbf{Training time (in min.)} across datasets. \methodname{} only incurs a small overhead compared to standard ERM and is substantially less expensive than recent methods. (\textit{Right}) \textbf{Average number of trainable parameters} updated in the adaptation from ERM for comparable baselines in an independent run. $d$ = embedding dimension, $C$ = number of classes, $N$ = DPE number of prototypes per class, and $P_\text{bb}$ = total parameters of base model backbone.}
    \label{fig:runtime_params}
    \vspace{-5pt}
\end{figure}

\paragraph{Data-Guided Hyperparameter Selection} We study sensitivity to the rank $k$ and weighting strength $\gamma$ in App.~\ref{app:parameter_ablation}. Candidate ranks are derived from the 50-90\% CWV region of the weighted covariance spectrum, while the useful numerical scale of $\gamma$ reflects the dataset-specific misclassification-score distribution. Across the tested ranges, performance is generally stable around the selected values. 
\section{Limitations}
\label{sec:limitations}
\vspace{-5pt}
A key limitation is that \methodname{} uses the base model's true-label probability to select the subspace for adaptation. Thus, the method can inherit mismatches between low true-label probability and true subgroup harm, including those caused by miscalibration, noisy labels, or
poisoned data. Controlled temperature-scaling experiments show limited degradation across the tested calibration perturbations (App.~\ref{app:calibration}), but do not provide robustness to arbitrary or structured miscalibration. \methodname{} also assumes that the relevant failure modes are represented in the adaptation data and are concentrated enough to be captured by a low-dimensional subspace. Finally, because the correction is applied to the classifier head of the base model, \methodname{} cannot recover predictive features that are
absent or poorly encoded in the representation, nor does it improve pretrained embeddings for transfer to other tasks.

\section{Conclusion}
\vspace{-1em}
We propose \methodname{}, a rank-constrained adaptation method that leverages the calculated misclassification probability of a base model  to derive an additive correction to the model's classification logits. Across seven benchmarks, \methodname{}'s performance is on par, or improves, worst-group accuracy relative to other baselines under three evaluation regimes: no knowledge of
subgroup relevance, partial knowledge of subgroup relevance, and full knowledge
of subgroup relevance.  In particular, \methodname{} shows the most pronounced gains in the partial knowledge setting and on real-world clinical datasets, reflecting MARLA's strength in the conditions closest to realistic deployment. Together with its computational efficiency, \methodname{} provides a practical path towards reliable performance in real-world settings.

\bibliography{refs}

@inproceedings{wilson2024gender,
  title={Gender, race, and intersectional bias in resume screening via language model retrieval},
  author={Wilson, Kyra and Caliskan, Aylin},
  booktitle={Proceedings of the AAAI/ACM Conference on AI, Ethics, and Society},
  volume={7},
  pages={1578--1590},
  year={2024}
}

@article{altarawneh2022effects,
  title={Effects of previous infection and vaccination on symptomatic omicron infections},
  author={Altarawneh, Heba N and Chemaitelly, Hiam and Ayoub, Houssein H and Tang, Patrick and Hasan, Mohammad R and Yassine, Hadi M and Al-Khatib, Hebah A and Smatti, Maria K and Coyle, Peter and Al-Kanaani, Zaina and others},
  journal={New England Journal of Medicine},
  volume={387},
  number={1},
  pages={21--34},
  year={2022},
  publisher={Mass Medical Soc}
}

@article{guo2021systematic,
  title={Systematic review of approaches to preserve machine learning performance in the presence of temporal dataset shift in clinical medicine},
  author={Guo, Lin Lawrence and Pfohl, Stephen R and Fries, Jason and Posada, Jose and Fleming, Scott Lanyon and Aftandilian, Catherine and Shah, Nigam and Sung, Lillian},
  journal={Applied clinical informatics},
  volume={12},
  number={04},
  pages={808--815},
  year={2021},
  publisher={Georg Thieme Verlag KG}
}

@article{otokiti2023need,
  title={The need to prioritize model-updating processes in clinical artificial intelligence (AI) models: protocol for a scoping review},
  author={Otokiti, Ahmed Umar and Ozoude, Makuochukwu Maryann and Williams, Karmen S and Sadiq-Onilenla, Rasheedat A and Ojo, Soji Akin and Wasarme, Leyla B and Walsh, Samantha and Edomwande, Maxwell},
  journal={JMIR Research Protocols},
  volume={12},
  number={1},
  pages={e37685},
  year={2023},
  publisher={JMIR Publications Inc., Toronto, Canada}
}

@ARTICLE{10.3389/fdgth.2022.958284,
    
AUTHOR={Davis, Sharon E.  and Walsh, Colin G.  and Matheny, Michael E. },
           
TITLE={Open questions and research gaps for monitoring and updating AI-enabled tools in clinical settings},
          
JOURNAL={Frontiers in Digital Health},
          
VOLUME={Volume 4 - 2022},
  
YEAR={2022},
  
URL={https://www.frontiersin.org/journals/digital-health/articles/10.3389/fdgth.2022.958284},
  
DOI={10.3389/fdgth.2022.958284},
  
ISSN={2673-253X}}

@article{mori2020using,
  title={Using latent class analysis to identify hidden clinical phenotypes},
  author={Mori, Makoto and Krumholz, Harlan M and Allore, Heather G},
  journal={Jama},
  volume={324},
  number={7},
  pages={700--701},
  year={2020}
}

@article{levy2022development,
  title={Development and validation of self-monitoring auto-updating prognostic models of survival for hospitalized COVID-19 patients},
  author={Levy, Todd J and Coppa, Kevin and Cang, Jinxuan and Barnaby, Douglas P and Paradis, Marc D and Cohen, Stuart L and Makhnevich, Alex and Van Klaveren, David and Kent, David M and Davidson, Karina W and others},
  journal={Nature communications},
  volume={13},
  number={1},
  pages={6812},
  year={2022},
  publisher={Nature Publishing Group UK London}
}

@article{lewnard2022clinical,
  title={Clinical outcomes associated with SARS-CoV-2 Omicron (B. 1.1. 529) variant and BA. 1/BA. 1.1 or BA. 2 subvariant infection in Southern California},
  author={Lewnard, Joseph A and Hong, Vennis X and Patel, Manish M and Kahn, Rebecca and Lipsitch, Marc and Tartof, Sara Y},
  journal={Nature medicine},
  volume={28},
  number={9},
  pages={1933--1943},
  year={2022},
  publisher={Nature Publishing Group US New York}
}

@article{taylor2022covid,
  title={COVID-19--associated hospitalizations among adults during SARS-CoV-2 delta and omicron variant predominance, by race/ethnicity and vaccination status—COVID-NET, 14 states, July 2021--January 2022},
  author={Taylor, Christopher A},
  journal={MMWR. Morbidity and Mortality Weekly Report},
  volume={71},
  year={2022}
}

@inproceedings{raghavan2020mitigating,
  title={Mitigating bias in algorithmic hiring: Evaluating claims and practices},
  author={Raghavan, Manish and Barocas, Solon and Kleinberg, Jon and Levy, Karen},
  booktitle={Proceedings of the 2020 conference on fairness, accountability, and transparency},
  pages={469--481},
  year={2020}
}

@inproceedings{sap2019risk,
  title={The risk of racial bias in hate speech detection},
  author={Sap, Maarten and Card, Dallas and Gabriel, Saadia and Choi, Yejin and Smith, Noah A},
  booktitle={Proceedings of the 57th annual meeting of the association for computational linguistics},
  pages={1668--1678},
  year={2019}
}

@article{hu2022lora,
  title={Lora: Low-rank adaptation of large language models.},
  author={Hu, Edward J and Shen, Yelong and Wallis, Phillip and Allen-Zhu, Zeyuan and Li, Yuanzhi and Wang, Shean and Wang, Lu and Chen, Weizhu and others},
  journal={ICLR},
  volume={1},
  number={2},
  pages={3},
  year={2022}
}

@article{zhao2025bias,
  title={Bias Is a Subspace, Not a Coordinate: A Geometric Rethinking of Post-hoc Debiasing in Vision-Language Models},
  author={Zhao, Dachuan and Li, Weiyue and Shen, Zhenda and Qiu, Yushu and Xu, Bowen and Chen, Haoyu and Chen, Yongchao},
  journal={arXiv preprint arXiv:2511.18123},
  year={2025}
}

@inproceedings{jang2025target,
  title={Target Bias Is All You Need: Zero-Shot Debiasing of Vision-Language Models with Bias Corpus},
  author={Jang, Taeuk and Jung, Hoin and Wang, Xiaoqian},
  booktitle={Proceedings of the IEEE/CVF International Conference on Computer Vision},
  pages={1935--1946},
  year={2025}
}

@article{chuang2023debiasing,
  title={Debiasing vision-language models via biased prompts},
  author={Chuang, Ching-Yao and Jampani, Varun and Li, Yuanzhen and Torralba, Antonio and Jegelka, Stefanie},
  journal={arXiv preprint arXiv:2302.00070},
  year={2023}
}

@article{gerych2024bendvlm,
  title={Bendvlm: Test-time debiasing of vision-language embeddings},
  author={Gerych, Walter and Zhang, Haoran and Hamidieh, Kimia and Pan, Eileen and Sharma, Maanas K and Hartvigsen, Tom and Ghassemi, Marzyeh},
  journal={Advances in Neural Information Processing Systems},
  volume={37},
  pages={62480--62502},
  year={2024}
}

@article{hirota2024saner,
  title={Saner: Annotation-free societal attribute neutralizer for debiasing clip},
  author={Hirota, Yusuke and Chen, Min-Hung and Wang, Chien-Yi and Nakashima, Yuta and Wang, Yu-Chiang Frank and Hachiuma, Ryo},
  journal={arXiv preprint arXiv:2408.10202},
  year={2024}
}

@article{jung2024unified,
  title={A unified debiasing approach for vision-language models across modalities and tasks},
  author={Jung, Hoin and Jang, Taeuk and Wang, Xiaoqian},
  journal={Advances in Neural Information Processing Systems},
  volume={37},
  pages={21034--21058},
  year={2024}
}

@inproceedings{zhang2025joint,
  title={Joint Vision-Language Social Bias Removal for CLIP},
  author={Zhang, Haoyu and Guo, Yangyang and Kankanhalli, Mohan},
  booktitle={Proceedings of the Computer Vision and Pattern Recognition Conference},
  pages={4246--4255},
  year={2025}
}

@inproceedings{hashimoto2018fairness,
  title={Fairness without demographics in repeated loss minimization},
  author={Hashimoto, Tatsunori and Srivastava, Megha and Namkoong, Hongseok and Liang, Percy},
  booktitle={International Conference on Machine Learning},
  pages={1929--1938},
  year={2018},
  organization={PMLR}
}

@misc{rudner2024mindgapimprovingrobustness,
      title={Mind the GAP: Improving Robustness to Subpopulation Shifts with Group-Aware Priors}, 
      author={Tim G. J. Rudner and Ya Shi Zhang and Andrew Gordon Wilson and Julia Kempe},
      year={2024},
      eprint={2403.09869},
      archivePrefix={arXiv},
      primaryClass={stat.ML},
      url={https://arxiv.org/abs/2403.09869}, 
}

@misc{qiu2023simplefastgrouprobustness,
      title={Simple and Fast Group Robustness by Automatic Feature Reweighting}, 
      author={Shikai Qiu and Andres Potapczynski and Pavel Izmailov and Andrew Gordon Wilson},
      year={2023},
      eprint={2306.11074},
      archivePrefix={arXiv},
      primaryClass={cs.LG},
      url={https://arxiv.org/abs/2306.11074}, 
}

@inproceedings{lin2017focal,
  title={Focal loss for dense object detection},
  author={Lin, Tsung-Yi and Goyal, Priya and Girshick, Ross and He, Kaiming and Doll{\'a}r, Piotr},
  booktitle={Proceedings of the IEEE international conference on computer vision},
  pages={2980--2988},
  year={2017}
}

@article{borkan2019nuanced,
  title={{Nuanced Metrics For Measuring Unintended Bias With Real Data For Text Classification}},
  author={Borkan, Daniel and Dixon, Lucas and Sorensen, Jeffrey and Thain, Nithum and Vasserman, Lucy},
  journal={Companion proceedings of the 2019 world wide web conference},
  pages={491--500},
  year={2019}
}

@article{cui2019class,
  title={{Class-balanced Loss Based On Effective Number Of Samples}},
  author={Cui, Yin and Jia, Menglin and Lin, Tsung-Yi and Song, Yang and Belongie, Serge},
  journal={Proceedings of the IEEE/CVF Conference on Computer Vision and Pattern Recognition},
  pages={9268--9277},
  year={2019}
}

@article{devlin2018bert,
  title={{BERT: Pre-training Of Deep Bidirectional Transformers For Language Understanding}},
  author={Devlin, Jacob and Chang, Ming-Wei and Lee, Kenton and Toutanova, Kristina},
  journal={arXiv preprint arXiv:1810.04805},
  year={2018}
}

@article{geirhos2020shortcut,
  title={{Shortcut Learning In Deep Neural Networks}},
  author={Geirhos, Robert and Jacobsen, J{\"o}rn-Henrik and Michaelis, Claudio and Zemel, Richard and Brendel, Wieland and Bethge, Matthias and Wichmann, Felix A},
  journal={Nature Machine Intelligence},
  volume={2},
  number={11},
  pages={665--673},
  year={2020},
  publisher={Nature Publishing Group}
}

@article{he2016deep,
  title={{Deep Residual Learning For Image Recognition}},
  author={He, Kaiming and Zhang, Xiangyu and Ren, Shaoqing and Sun, Jian},
  journal={Proceedings of the IEEE conference on computer vision and pattern recognition},
  pages={770--778},
  year={2016}
}

@article{idrissi2021simple,
  title={{Simple Data Balancing Achieves Competitive Worst-Group-Accuracy}},
  author={Idrissi, Badr Youbi and Arjovsky, Martin and Pezeshki, Mohammad and Lopez-Paz, David},
  journal={arXiv preprint arXiv:2110.14503},
  year={2021}
}

@article{kirichenko2022dfr,
  title={{Last Layer Re-Training Is Sufficient For Robustness To Spurious Correlations}},
  author={Polina Kirichenko and Pavel Izmailov and Andrew Gordon Wilson},
  journal={Preprint arXiv 2204.02937v1},
  year={2022},
}

@article{liu2021jtt,
  title={{Just Train Twice: Improving Group Robustness Without Training Group Information}},
  author={Evan Zheran Liu and Behzad Haghgoo and Annie S. Chen and Aditi Raghunathan
  and Pang Wei Koh and Shiori Sagawa and Percy Liang and Chelsea Finn},
  journal={International Conference on Machine Learning (ICML)},
  year={2021},
}

@article{sagawa2020GDRO,
  title={{Distributionally Robust Neural Networks For Group Shifts: On The Importance Of Regularization For Worst-Case Generalization}},
  author={Shiori Sagawa and Pang Wei Koh and Tatsunori B. Hashimoto and Percy Liang},
  journal={International Conference on Learning Representations (ICLR)},
  year={2020},
}

@article{sohoni2020george, 
  title={{No Subclass Left Behind: Fine-Grained Robustness In Coarse-Grained Classification Problems}},
  author={Sohoni, Nimit and Dunnmon, Jared and Angus, Geoffrey and Gu, Albert and R\'{e}, Christopher},
  journal={Neural Information Processing Systems (NeurIPS)},
  pages={19339--19352},
  volume={33},
  year={2020}
}

@article{jain2024improving,
  title={Improving subgroup robustness via data selection},
  author={Jain, Saachi and Hamidieh, Kimia and Georgiev, Kristian and Ilyas, Andrew and Ghassemi, Marzyeh and Madry, Aleksander},
  journal={Advances in Neural Information Processing Systems},
  volume={37},
  pages={94490--94511},
  year={2024}
}

@article{labonte2023towards,
  title={Towards last-layer retraining for group robustness with fewer annotations},
  author={LaBonte, Tyler and Muthukumar, Vidya and Kumar, Abhishek},
  journal={Advances in Neural Information Processing Systems},
  volume={36},
  pages={11552--11579},
  year={2023}
}

@inproceedings{
ahmed2021systematic,
title={Systematic generalisation with group invariant predictions},
author={Faruk Ahmed and Yoshua Bengio and Harm van Seijen and Aaron Courville},
booktitle={International Conference on Learning Representations},
year={2021},
url={https://openreview.net/forum?id=b9PoimzZFJ}
}

@article{zech2018variable,
  title={{Variable Generalization Performance Of A Deep Learning Model To Detect Pneumonia In Chest Radiographs: A Cross-Sectional Study}},
  author={Zech, John R and Badgeley, Marcus A and Liu, Manway and Costa, Anthony B and Titano, Joseph J and Oermann, Eric Karl},
  journal={PLoS medicine},
  volume={15},
  number={11},
  pages={e1002683},
  year={2018},
  publisher={Public Library of Science San Francisco, CA USA}
}

@article{zhang2022cnc,
  title={{Correct-n-Contrast: A Contrastive Approach For Improving Robustness To Spurious Correlations}},
  author={Michael Zhang and Nimit S. Sohoni and Hongyang R. Zhang and Chelsea Finn and Christopher R\'{e}},
  journal={Preprint arXiv 2203.01517v1},
  year={2022},
}

@article{russakovsky2015imagenet,
  title={{ImageNet Large Scale Visual Recognition Challenge}},
  author={Russakovsky, Olga and Deng, Jia and Su, Hao and Krause, Jonathan and Satheesh, Sanjeev and Ma, Sean and Huang, Zhiheng and Karpathy, Andrej and Khosla, Aditya and Bernstein, Michael and others},
  journal={International journal of computer vision},
  volume={115},
  number={3},
  pages={211--252},
  year={2015},
  publisher={Springer}
}

@misc{duchene2023benchmarktoxiccommentclassification,
      title={A benchmark for toxic comment classification on Civil Comments dataset}, 
      author={Corentin Duchene and Henri Jamet and Pierre Guillaume and Reda Dehak},
      year={2023},
      eprint={2301.11125},
      archivePrefix={arXiv},
      primaryClass={cs.CL},
      url={https://arxiv.org/abs/2301.11125}, 
}

@misc{yao2022improvingoutofdistributionrobustnessselective,
      title={Improving Out-of-Distribution Robustness via Selective Augmentation}, 
      author={Huaxiu Yao and Yu Wang and Sai Li and Linjun Zhang and Weixin Liang and James Zou and Chelsea Finn},
      year={2022},
      eprint={2201.00299},
      archivePrefix={arXiv},
      primaryClass={cs.LG},
      url={https://arxiv.org/abs/2201.00299}, 
}

@misc{han2024improvinggrouprobustnessspurious,
      title={Improving Group Robustness on Spurious Correlation Requires Preciser Group Inference}, 
      author={Yujin Han and Difan Zou},
      year={2024},
      eprint={2404.13815},
      archivePrefix={arXiv},
      primaryClass={cs.LG},
      url={https://arxiv.org/abs/2404.13815}, 
}

@misc{mccoy2024dataadaptiveidentificationeffectmodifiers,
      title={Data-Adaptive Identification of Effect Modifiers through Stochastic Shift Interventions and Cross-Validated Targeted Learning}, 
      author={David McCoy and Wenxin Zhang and Alan Hubbard and Mark van der Laan and Alejandro Schuler},
      year={2024},
      eprint={2406.10792},
      archivePrefix={arXiv},
      primaryClass={stat.ME},
      url={https://arxiv.org/abs/2406.10792}, 
}

@article{martin2017assessment,
  title={An assessment of the impact of pharmacogenomics on health disparities: a systematic literature review},
  author={Martin, Antony and Downing, Jennifer and Maden, Michelle and Fleeman, Nigel and Alfirevic, Ana and Haycox, Alan and Pirmohamed, Munir},
  journal={Pharmacogenomics},
  volume={18},
  number={16},
  pages={1541--1550},
  year={2017},
  publisher={Taylor \& Francis}
}

@article{rossello2025beta,
  title={Beta-blockers after myocardial infarction: effects according to sex in the REBOOT trial},
  author={Rossello, Xavier and Dominguez-Rodriguez, Alberto and Latini, Roberto and S{\'a}nchez, Pedro L and Raposeiras-Roub{\'\i}n, Sergio and Anguita, Manuel and Barrab{\'e}s, Jos{\'e} A and Grigis, Giulietta and Owen, Ruth and Pocock, Stuart and others},
  journal={European Heart Journal},
  pages={ehaf673},
  year={2025},
  publisher={Oxford University Press UK}
}

@article{mas2023representativeness,
  title={Representativeness in randomised clinical trials supporting acute coronary syndrome guidelines},
  author={Mas-Llado, Caterina and Gonzalez-Del-Hoyo, Maribel and Siquier-Padilla, Joan and Blaya-Pe{\~n}a, Laura and Coughlan, JJ and Garc{\'\i}a de la Villa, Bernardo and Peral, Vicente and Rossello, Xavier},
  journal={European Heart Journal-Quality of Care and Clinical Outcomes},
  volume={9},
  number={8},
  pages={796--805},
  year={2023},
  publisher={Oxford University Press}
}

@article{vogel2021lancet,
  title={The Lancet women and cardiovascular disease Commission: reducing the global burden by 2030},
  author={Vogel, Birgit and Acevedo, Monica and Appelman, Yolande and Merz, C Noel Bairey and Chieffo, Alaide and Figtree, Gemma A and Guerrero, Mayra and Kunadian, Vijay and Lam, Carolyn SP and Maas, Angela HEM and others},
  journal={The Lancet},
  volume={397},
  number={10292},
  pages={2385--2438},
  year={2021},
  publisher={Elsevier}
}

@misc{liu2019largescalelongtailedrecognitionopen,
      title={Large-Scale Long-Tailed Recognition in an Open World}, 
      author={Ziwei Liu and Zhongqi Miao and Xiaohang Zhan and Jiayun Wang and Boqing Gong and Stella X. Yu},
      year={2019},
      eprint={1904.05160},
      archivePrefix={arXiv},
      primaryClass={cs.CV},
      url={https://arxiv.org/abs/1904.05160}, 
}

@misc{kang2020decouplingrepresentationclassifierlongtailed,
      title={Decoupling Representation and Classifier for Long-Tailed Recognition}, 
      author={Bingyi Kang and Saining Xie and Marcus Rohrbach and Zhicheng Yan and Albert Gordo and Jiashi Feng and Yannis Kalantidis},
      year={2020},
      eprint={1910.09217},
      archivePrefix={arXiv},
      primaryClass={cs.CV},
      url={https://arxiv.org/abs/1910.09217}, 
}

@article{kaplow1999conflict,
  title={The conflict between notions of fairness and the Pareto principle},
  author={Kaplow, Louis and Shavell, Steven},
  journal={American Law and Economics Review},
  volume={1},
  number={1},
  pages={63--77},
  year={1999},
  publisher={Oxford University Press}
}

@article{10.1257/pandp.20181018,
Author = {Kleinberg, Jon and Ludwig, Jens and Mullainathan, Sendhil and Rambachan, Ashesh},
Title = {Algorithmic Fairness},
Journal = {AEA Papers and Proceedings},
Volume = {108},
Year = {2018},
Month = {May},
Pages = {22–27},
DOI = {10.1257/pandp.20181018},
URL = {https://www.aeaweb.org/articles?id=10.1257/pandp.20181018}}

@article{oakden2020hidden,
  title={{Hidden Stratification Causes Clinically Meaningful Failures In Machine Learning For Medical Imaging}},
  author={Oakden-Rayner, Luke and Dunnmon, Jared and Carneiro, Gustavo and R{\'e}, Christopher},
  journal={Proceedings of the ACM conference on health, inference, and learning},
  pages={151--159},
  year={2020}
}

@article{zhang2017mixup,
  title={mixup: Beyond empirical risk minimization},
  author={Zhang, Hongyi and Cisse, Moustapha and Dauphin, Yann N and Lopez-Paz, David},
  journal={arXiv preprint arXiv:1710.09412},
  year={2017}
}

@article{arjovsky2019invariant,
  title={Invariant risk minimization},
  author={Arjovsky, Martin and Bottou, L{\'e}on and Gulrajani, Ishaan and Lopez-Paz, David},
  journal={arXiv preprint arXiv:1907.02893},
  year={2019}
}

@article{lahoti2020fairness,
  title={Fairness without demographics through adversarially reweighted learning},
  author={Lahoti, Preethi and Beutel, Alex and Chen, Jilin and Lee, Kang and Prost, Flavien and Thain, Nithum and Wang, Xuezhi and Chi, Ed},
  journal={Advances in neural information processing systems},
  volume={33},
  pages={728--740},
  year={2020}
}

@article{kallus2022assessing,
  title={Assessing algorithmic fairness with unobserved protected class using data combination},
  author={Kallus, Nathan and Mao, Xiaojie and Zhou, Angela},
  journal={Management Science},
  volume={68},
  number={3},
  pages={1959--1981},
  year={2022},
  publisher={INFORMS}
}

@InProceedings{pmlr-v81-buolamwini18a,
  title = 	 {Gender Shades: Intersectional Accuracy Disparities in Commercial Gender Classification},
  author = 	 {Buolamwini, Joy and Gebru, Timnit},
  booktitle = 	 {Proceedings of the 1st Conference on Fairness, Accountability and Transparency},
  pages = 	 {77--91},
  year = 	 {2018},
  editor = 	 {Friedler, Sorelle A. and Wilson, Christo},
  volume = 	 {81},
  series = 	 {Proceedings of Machine Learning Research},
  month = 	 {23--24 Feb},
  publisher =    {PMLR},
  url = 	 {https://proceedings.mlr.press/v81/buolamwini18a.html}
}

@article{yang2024limits,
  title={The limits of fair medical imaging AI in real-world generalization},
  author={Yang, Yuzhe and Zhang, Haoran and Gichoya, Judy W and Katabi, Dina and Ghassemi, Marzyeh},
  journal={Nature Medicine},
  volume={30},
  number={10},
  pages={2838--2848},
  year={2024},
  publisher={Nature Publishing Group US New York}
}

@InProceedings{pmlr-v119-martinez20a,
  title = 	 {Minimax Pareto Fairness: A Multi Objective Perspective},
  author =       {Martinez, Natalia and Bertran, Martin and Sapiro, Guillermo},
  booktitle = 	 {Proceedings of the 37th International Conference on Machine Learning},
  pages = 	 {6755--6764},
  year = 	 {2020},
  editor = 	 {III, Hal Daumé and Singh, Aarti},
  volume = 	 {119},
  series = 	 {Proceedings of Machine Learning Research},
  month = 	 {13--18 Jul},
  publisher =    {PMLR},
  url = 	 {https://proceedings.mlr.press/v119/martinez20a.html}
}

@inproceedings{yang2024mitigating,
  title={On mitigating shortcut learning for fair chest x-ray classification under distribution shift},
  author={Yang, Yuzhe and Zhang, Haoran and Katabi, Dina and Ghassemi, Marzyeh},
  booktitle={NeurIPS 2023 Workshop on Distribution Shifts: New Frontiers with Foundation Models},
  year={2024}
}

@misc{santurkar2020breedsbenchmarkssubpopulationshift,
      title={BREEDS: Benchmarks for Subpopulation Shift}, 
      author={Shibani Santurkar and Dimitris Tsipras and Aleksander Madry},
      year={2020},
      eprint={2008.04859},
      archivePrefix={arXiv},
      primaryClass={cs.CV},
      url={https://arxiv.org/abs/2008.04859}, 
}

@InProceedings{pmlr-v108-awasthi20a,
  title = 	 {Equalized odds postprocessing under imperfect group information},
  author =       {Awasthi, Pranjal and Kleindessner, Matth\"aus and Morgenstern, Jamie},
  booktitle = 	 {Proceedings of the Twenty Third International Conference on Artificial Intelligence and Statistics},
  pages = 	 {1770--1780},
  year = 	 {2020},
  editor = 	 {Chiappa, Silvia and Calandra, Roberto},
  volume = 	 {108},
  series = 	 {Proceedings of Machine Learning Research},
  month = 	 {26--28 Aug},
  publisher =    {PMLR},
  url = 	 {https://proceedings.mlr.press/v108/awasthi20a.html}
}

@article{krumpal2013determinants,
  title={Determinants of social desirability bias in sensitive surveys: a literature review},
  author={Krumpal, Ivar},
  journal={Quality \& quantity},
  volume={47},
  number={4},
  pages={2025--2047},
  year={2013},
  publisher={Springer}
}

@misc{duchi2020learningmodelsuniformperformance,
      title={Learning Models with Uniform Performance via Distributionally Robust Optimization}, 
      author={John Duchi and Hongseok Namkoong},
      year={2020},
      eprint={1810.08750},
      archivePrefix={arXiv},
      primaryClass={stat.ML},
      url={https://arxiv.org/abs/1810.08750}, 
}

@article{obermeyer2019dissecting,
  title={Dissecting racial bias in an algorithm used to manage the health of populations},
  author={Obermeyer, Ziad and Powers, Brian and Vogeli, Christine and Mullainathan, Sendhil},
  journal={Science},
  volume={366},
  number={6464},
  pages={447--453},
  year={2019},
  publisher={American Association for the Advancement of Science}
}

@InProceedings{pmlr-v80-kearns18a,
  title = 	 {Preventing Fairness Gerrymandering: Auditing and Learning for Subgroup Fairness},
  author =       {Kearns, Michael and Neel, Seth and Roth, Aaron and Wu, Zhiwei Steven},
  booktitle = 	 {Proceedings of the 35th International Conference on Machine Learning},
  pages = 	 {2564--2572},
  year = 	 {2018},
  editor = 	 {Dy, Jennifer and Krause, Andreas},
  volume = 	 {80},
  series = 	 {Proceedings of Machine Learning Research},
  month = 	 {10--15 Jul},
  publisher =    {PMLR},
  url = 	 {https://proceedings.mlr.press/v80/kearns18a.html}
}

@misc{to2025diverseprototypicalensemblesimprove,
      title={Diverse Prototypical Ensembles Improve Robustness to Subpopulation Shift}, 
      author={Minh Nguyen Nhat To and Paul F RWilson and Viet Nguyen and Mohamed Harmanani and Michael Cooper and Fahimeh Fooladgar and Purang Abolmaesumi and Parvin Mousavi and Rahul G. Krishnan},
      year={2025},
      eprint={2505.23027},
      archivePrefix={arXiv},
      primaryClass={cs.LG},
      url={https://arxiv.org/abs/2505.23027}, 
}

@misc{yang2023changehardcloserlook,
      title={Change is Hard: A Closer Look at Subpopulation Shift}, 
      author={Yuzhe Yang and Haoran Zhang and Dina Katabi and Marzyeh Ghassemi},
      year={2023},
      eprint={2302.12254},
      archivePrefix={arXiv},
      primaryClass={cs.LG},
      url={https://arxiv.org/abs/2302.12254}, 
}

@misc{irvin2019chexpertlargechestradiograph,
      title={CheXpert: A Large Chest Radiograph Dataset with Uncertainty Labels and Expert Comparison}, 
      author={Jeremy Irvin and Pranav Rajpurkar and Michael Ko and Yifan Yu and Silviana Ciurea-Ilcus and Chris Chute and Henrik Marklund and Behzad Haghgoo and Robyn Ball and Katie Shpanskaya and Jayne Seekins and David A. Mong and Safwan S. Halabi and Jesse K. Sandberg and Ricky Jones and David B. Larson and Curtis P. Langlotz and Bhavik N. Patel and Matthew P. Lungren and Andrew Y. Ng},
      year={2019},
      eprint={1901.07031},
      archivePrefix={arXiv},
      primaryClass={cs.CV},
      url={https://arxiv.org/abs/1901.07031}, 
}

@misc{tsipras2019robustnessoddsaccuracy,
      title={Robustness May Be at Odds with Accuracy}, 
      author={Dimitris Tsipras and Shibani Santurkar and Logan Engstrom and Alexander Turner and Aleksander Madry},
      year={2019},
      eprint={1805.12152},
      archivePrefix={arXiv},
      primaryClass={stat.ML},
      url={https://arxiv.org/abs/1805.12152}, 
}

@misc{salaudeen2025aggregationhidesoutofdistributiongeneralization,
      title={Aggregation Hides Out-of-Distribution Generalization Failures from Spurious Correlations}, 
      author={Olawale Salaudeen and Haoran Zhang and Kumail Alhamoud and Sara Beery and Marzyeh Ghassemi},
      year={2025},
      eprint={2510.24884},
      archivePrefix={arXiv},
      primaryClass={cs.LG},
      url={https://arxiv.org/abs/2510.24884}, 
}

@article{balashankar2019fair,
  title={What is fair? exploring pareto-efficiency for fairness constrained classifiers},
  author={Balashankar, Ananth and Lees, Alyssa and Welty, Chris and Subramanian, Lakshminarayanan},
  journal={arXiv preprint arXiv:1910.14120},
  year={2019}
}

@inproceedings{lu2024neural,
  title={Neural collapse inspired debiased representation learning for min-max fairness},
  author={Lu, Shenyu and Chai, Junyi and Wang, Xiaoqian},
  booktitle={Proceedings of the 30th ACM SIGKDD Conference on Knowledge Discovery and Data Mining},
  pages={2048--2059},
  year={2024}
}

@InProceedings{pmlr-v139-martinez21a,
  title = 	 {Blind Pareto Fairness and Subgroup Robustness},
  author =       {Martinez, Natalia L and Bertran, Martin A and Papadaki, Afroditi and Rodrigues, Miguel and Sapiro, Guillermo},
  booktitle = 	 {Proceedings of the 38th International Conference on Machine Learning},
  pages = 	 {7492--7501},
  year = 	 {2021},
  editor = 	 {Meila, Marina and Zhang, Tong},
  volume = 	 {139},
  series = 	 {Proceedings of Machine Learning Research},
  month = 	 {18--24 Jul},
  publisher =    {PMLR},
  url = 	 {https://proceedings.mlr.press/v139/martinez21a.html}
}

@misc{shah2020pitfallssimplicitybiasneural,
      title={The Pitfalls of Simplicity Bias in Neural Networks}, 
      author={Harshay Shah and Kaustav Tamuly and Aditi Raghunathan and Prateek Jain and Praneeth Netrapalli},
      year={2020},
      eprint={2006.07710},
      archivePrefix={arXiv},
      primaryClass={cs.LG},
      url={https://arxiv.org/abs/2006.07710}, 
}

@misc{nagarajan2024understandingfailuremodesoutofdistribution,
      title={Understanding the Failure Modes of Out-of-Distribution Generalization}, 
      author={Vaishnavh Nagarajan and Anders Andreassen and Behnam Neyshabur},
      year={2024},
      eprint={2010.15775},
      archivePrefix={arXiv},
      primaryClass={cs.LG},
      url={https://arxiv.org/abs/2010.15775}, 
}

@misc{koh2021wildsbenchmarkinthewilddistribution,
      title={WILDS: A Benchmark of in-the-Wild Distribution Shifts}, 
      author={Pang Wei Koh and Shiori Sagawa and Henrik Marklund and Sang Michael Xie and Marvin Zhang and Akshay Balsubramani and Weihua Hu and Michihiro Yasunaga and Richard Lanas Phillips and Irena Gao and Tony Lee and Etienne David and Ian Stavness and Wei Guo and Berton A. Earnshaw and Imran S. Haque and Sara Beery and Jure Leskovec and Anshul Kundaje and Emma Pierson and Sergey Levine and Chelsea Finn and Percy Liang},
      year={2021},
      eprint={2012.07421},
      archivePrefix={arXiv},
      primaryClass={cs.LG},
      url={https://arxiv.org/abs/2012.07421}, 
}

@misc{williams2018broadcoveragechallengecorpussentence,
      title={A Broad-Coverage Challenge Corpus for Sentence Understanding through Inference}, 
      author={Adina Williams and Nikita Nangia and Samuel R. Bowman},
      year={2018},
      eprint={1704.05426},
      archivePrefix={arXiv},
      primaryClass={cs.CL},
      url={https://arxiv.org/abs/1704.05426}, 
}

@INPROCEEDINGS{7410782,
  author={Liu, Ziwei and Luo, Ping and Wang, Xiaogang and Tang, Xiaoou},
  booktitle={2015 IEEE International Conference on Computer Vision (ICCV)}, 
  title={Deep Learning Face Attributes in the Wild}, 
  year={2015},
  volume={},
  number={},
  pages={3730-3738},
  doi={10.1109/ICCV.2015.425}}

@article{goel2020model,
  title={Model patching: Closing the subgroup performance gap with data augmentation},
  author={Goel, Karan and Gu, Albert and Li, Yixuan and R{\'e}, Christopher},
  journal={arXiv preprint arXiv:2008.06775},
  year={2020}
}

@article{nam2020learning,
  title={Learning from failure: De-biasing classifier from biased classifier},
  author={Nam, Junhyun and Cha, Hyuntak and Ahn, Sungsoo and Lee, Jaeho and Shin, Jinwoo},
  journal={Advances in Neural Information Processing Systems},
  volume={33},
  pages={20673--20684},
  year={2020}
}

@inproceedings{creager2021environment,
  title={Environment inference for invariant learning},
  author={Creager, Elliot and Jacobsen, J{\"o}rn-Henrik and Zemel, Richard},
  booktitle={International Conference on Machine Learning},
  pages={2189--2200},
  year={2021},
  organization={PMLR}
}

@inproceedings{ribeiro2016should,
  title={" Why should i trust you?" Explaining the predictions of any classifier},
  author={Ribeiro, Marco Tulio and Singh, Sameer and Guestrin, Carlos},
  booktitle={Proceedings of the 22nd ACM SIGKDD international conference on knowledge discovery and data mining},
  pages={1135--1144},
  year={2016}
}

@inproceedings{beery2018recognition,
  title={Recognition in terra incognita},
  author={Beery, Sara and Van Horn, Grant and Perona, Pietro},
  booktitle={Proceedings of the European conference on computer vision (ECCV)},
  pages={456--473},
  year={2018}
}

@article{banerjee2023shortcuts,
  title={“Shortcuts” causing bias in radiology artificial intelligence: causes, evaluation, and mitigation},
  author={Banerjee, Imon and Bhattacharjee, Kamanasish and Burns, John L and Trivedi, Hari and Purkayastha, Saptarshi and Seyyed-Kalantari, Laleh and Patel, Bhavik N and Shiradkar, Rakesh and Gichoya, Judy},
  journal={Journal of the American College of Radiology},
  volume={20},
  number={9},
  pages={842--851},
  year={2023},
  publisher={Elsevier}
}

@article{winkler2019association,
  title={Association between surgical skin markings in dermoscopic images and diagnostic performance of a deep learning convolutional neural network for melanoma recognition},
  author={Winkler, Julia K and Fink, Christine and Toberer, Ferdinand and Enk, Alexander and Deinlein, Teresa and Hofmann-Wellenhof, Rainer and Thomas, Luc and Lallas, Aimilios and Blum, Andreas and Stolz, Wilhelm and others},
  journal={JAMA dermatology},
  volume={155},
  number={10},
  pages={1135--1141},
  year={2019},
  publisher={American Medical Association}
}

@article{shahamatdar2024deceptive,
  title={Deceptive learning in histopathology},
  author={Shahamatdar, Sahar and Saeed-Vafa, Daryoush and Linsley, Drew and Khalil, Farah and Lovinger, Katherine and Li, Lester and McLeod, Howard T and Ramachandran, Sohini and Serre, Thomas},
  journal={Histopathology},
  volume={85},
  number={1},
  pages={116--132},
  year={2024},
  publisher={Wiley Online Library}
}

@article{singla2022core,
  title={Core risk minimization using salient imagenet},
  author={Singla, Sahil and Moayeri, Mazda and Feizi, Soheil},
  journal={arXiv preprint arXiv:2203.15566},
  year={2022}
}

@article{mccoy2019right,
  title={Right for the wrong reasons: Diagnosing syntactic heuristics in natural language inference},
  author={McCoy, R Thomas and Pavlick, Ellie and Linzen, Tal},
  journal={arXiv preprint arXiv:1902.01007},
  year={2019}
}

@article{zhao2017men,
  title={Men also like shopping: Reducing gender bias amplification using corpus-level constraints},
  author={Zhao, Jieyu and Wang, Tianlu and Yatskar, Mark and Ordonez, Vicente and Chang, Kai-Wei},
  journal={arXiv preprint arXiv:1707.09457},
  year={2017}
}

@article{wu2024reft,
  title={Reft: Representation finetuning for language models},
  author={Wu, Zhengxuan and Arora, Aryaman and Wang, Zheng and Geiger, Atticus and Jurafsky, Dan and Manning, Christopher D and Potts, Christopher},
  journal={Advances in Neural Information Processing Systems},
  volume={37},
  pages={63908--63962},
  year={2024}
}

@article{yin2024lofit,
  title={Lofit: Localized fine-tuning on llm representations},
  author={Yin, Fangcong and Ye, Xi and Durrett, Greg},
  journal={Advances in Neural Information Processing Systems},
  volume={37},
  pages={9474--9506},
  year={2024}
}

@inproceedings{rimsky2024steering,
  title={Steering llama 2 via contrastive activation addition},
  author={Rimsky, Nina and Gabrieli, Nick and Schulz, Julian and Tong, Meg and Hubinger, Evan and Turner, Alexander},
  booktitle={Proceedings of the 62nd Annual Meeting of the Association for Computational Linguistics (Volume 1: Long Papers)},
  pages={15504--15522},
  year={2024}
}

@inproceedings{wang2020towards,
  title={Towards fairness in visual recognition: Effective strategies for bias mitigation},
  author={Wang, Zeyu and Qinami, Klint and Karakozis, Ioannis Christos and Genova, Kyle and Nair, Prem and Hata, Kenji and Russakovsky, Olga},
  booktitle={Proceedings of the IEEE/CVF conference on computer vision and pattern recognition},
  pages={8919--8928},
  year={2020}
}

@inproceedings{japkowicz2000class,
  title={The class imbalance problem: Significance and strategies},
  author={Japkowicz, Nathalie},
  booktitle={Proc. of the Int’l Conf. on artificial intelligence},
  volume={56},
  pages={111--117},
  year={2000}
}

@article{tu2020robusttospurious,
  author       = {Lifu Tu and
                  Garima Lalwani and
                  Spandana Gella and
                  He He},
  title        = {An Empirical Study on Robustness to Spurious Correlations using Pre-trained
                  Language Models},
  journal      = {CoRR},
  volume       = {abs/2007.06778},
  year         = {2020},
  url          = {https://arxiv.org/abs/2007.06778},
  eprinttype    = {arXiv},
  eprint       = {2007.06778},
  bibsource    = {dblp computer science bibliography, https://dblp.org}
}

@misc{wang2022identifyingmitigatingspuriouscorrelations,
      title={Identifying and Mitigating Spurious Correlations for Improving Robustness in NLP Models}, 
      author={Tianlu Wang and Rohit Sridhar and Diyi Yang and Xuezhi Wang},
      year={2022},
      eprint={2110.07736},
      archivePrefix={arXiv},
      primaryClass={cs.CL},
      url={https://arxiv.org/abs/2110.07736}, 
}

@misc{du2023learnshortcutanalyzingmitigating,
      title={Less Learn Shortcut: Analyzing and Mitigating Learning of Spurious Feature-Label Correlation}, 
      author={Yanrui Du and Jing Yan and Yan Chen and Jing Liu and Sendong Zhao and Qiaoqiao She and Hua Wu and Haifeng Wang and Bing Qin},
      year={2023},
      eprint={2205.12593},
      archivePrefix={arXiv},
      primaryClass={cs.CL},
      url={https://arxiv.org/abs/2205.12593}, 
}

@article{granger2003predictors,
  title={Predictors of hospital mortality in the global registry of acute coronary events},
  author={Granger, Christopher B and Goldberg, Robert J and Dabbous, Omar and Pieper, Karen S and Eagle, Kim A and Cannon, Christopher P and Van de Werf, Frans and Avezum, Alvaro and Goodman, Shaun G and Flather, Marcus D and others},
  journal={Archives of internal medicine},
  volume={163},
  number={19},
  pages={2345--2353},
  year={2003},
  publisher={American Medical Association}
}

@article{fox2014should,
  title={Should patients with acute coronary disease be stratified for management according to their risk? Derivation, external validation and outcomes using the updated GRACE risk score},
  author={Fox, Keith AA and FitzGerald, Gordon and Puymirat, Etienne and Huang, Wei and Carruthers, Kathryn and Simon, Tabassome and Coste, Pierre and Monsegu, Jacques and Steg, Philippe Gabriel and Danchin, Nicolas and others},
  journal={BMJ open},
  volume={4},
  number={2},
  pages={e004425},
  year={2014},
  publisher={British Medical Journal Publishing Group}
}

@article{fox2006prediction,
  title={Prediction of risk of death and myocardial infarction in the six months after presentation with acute coronary syndrome: prospective multinational observational study (GRACE)},
  author={Fox, Keith AA and Dabbous, Omar H and Goldberg, Robert J and Pieper, Karen S and Eagle, Kim A and Van de Werf, Frans and Avezum, {\'A}lvaro and Goodman, Shaun G and Flather, Marcus D and Anderson, Frederick A and others},
  journal={bmj},
  volume={333},
  number={7578},
  pages={1091},
  year={2006},
  publisher={British Medical Journal Publishing Group}
}

@article{rationale2001design,
  title={design of the GRACE (Global Registry of Acute Coronary Events) Project: a multinational registry of patients hospitalized with acute coronary syndromes},
  author={Rationale, GRACE},
  journal={Am Heart J},
  volume={141},
  number={2},
  pages={190--199},
  year={2001}
}

@article{eagle2002practice,
  title={Practice variation and missed opportunities for reperfusion in ST-segment-elevation myocardial infarction: findings from the Global Registry of Acute Coronary Events (GRACE)},
  author={Eagle, Kim A and Goodman, Shaun G and Avezum, {\'A}lvaro and Budaj, Andrzej and Sullivan, Cynthia M and L{\'o}pez-Send{\'o}n, Jos{\'e}},
  journal={The Lancet},
  volume={359},
  number={9304},
  pages={373--377},
  year={2002},
  publisher={Elsevier}
}

@article{johnson2023mimic,
  title={{MIMIC-IV}, a freely accessible electronic health record dataset},
  author={Johnson, Alistair EW and Bulgarelli, Lucas and Shen, Lu and Gayles, Alvin and Shammout, Ayad and Horng, Steven and Pollard, Tom J and Hao, Sicheng and Moody, Benjamin and Gow, Brian and others},
  journal={Scientific data},
  volume={10},
  number={1},
  pages={1},
  year={2023},
  publisher={Nature Publishing Group UK London}
}

@article{cao2019learning,
  title={Learning imbalanced datasets with label-distribution-aware margin loss},
  author={Cao, Kaidi and Wei, Colin and Gaidon, Adrien and Arechiga, Nikos and Ma, Tengyu},
  journal={Advances in neural information processing systems},
  volume={32},
  year={2019}
}

@article{ren2020balanced,
  title={Balanced meta-softmax for long-tailed visual recognition},
  author={Ren, Jiawei and Yu, Cunjun and Ma, Xiao and Zhao, Haiyu and Yi, Shuai and others},
  journal={Advances in neural information processing systems},
  volume={33},
  pages={4175--4186},
  year={2020}
}
\bibliographystyle{plain}

\newpage
\appendix

\section*{Appendix Outline}
The appendix is organized as follows:
\begin{itemize}
    \item In Appendix~\ref{app:algorithm_implementation}, we provide code illustrating the implementation of \methodname{}.
    \item In Appendix~\ref{app:unknown_groups}, we motivate evaluation under incomplete group knowledge and provide the proof and controlled experiment associated with Proposition~\ref{prop:unknown_group_dro}.
    \item In Appendix~\ref{sec:expdetails}, we provide dataset, baseline, implementation, hyperparameter, and compute details.
    \item In Appendix~\ref{app:ablations}, we report ablations of the method mechanism, hyperparameters, and data split.
    \item In Appendix~\ref{app:synthetic_lora}, we provide the controlled synthetic analysis of \methodname{}'s low-rank mechanism.
    \item In Appendix~\ref{app:subspace_conditions}, we give sufficient conditions for recovering a failure-associated subspace.
    \item In Appendix~\ref{app:additional_results}, we report additional clinical metrics, calibration sensitivity, complete-information baselines, and larger-model results.
    \item In Appendix~\ref{app:related_work_plus}, we discuss related work in more detail than in the main text.
\end{itemize}

\section{Algorithm Implementation Details}
\label{app:algorithm_implementation}
This section provides implementation details for \methodname{}. We show how the method computes weights from the base model's true-label probabilities, forms the score-weighted covariance matrix, extracts the selected subspace $V_k$ by spectral decomposition, and learns a low-rank additive correction to the classifier logits. All reported experiments use these score-based weights directly, equivalent to \texttt{rebalance=False} in the implementation. No class-frequency multiplier is applied to either the covariance or the Stage~2 loss.

\begin{lstlisting}[caption={Computing \methodname{} score-based weights}]
def compute_marla_weights(base_logits, y, gamma):
    """
    Compute MARLA weights from the base model's true-label probabilities.
    """
    probs = F.softmax(base_logits, dim=-1)

    y_onehot = torch.zeros_like(base_logits).scatter_(
        -1, y.unsqueeze(-1), 1
    )
    p_true = (probs * y_onehot).sum(-1)

    log_weights = gamma * (1 - p_true)
    log_weights = log_weights - log_weights.max()
    weights = log_weights.exp()
    weights = weights / weights.sum()

    return weights.detach()
\end{lstlisting}

\begin{lstlisting}[caption={Score-weighted covariance computation}]
def explain_score_weighted_covariance(embeddings, weights):
    """
    Compute the centered score-weighted covariance matrix.
    """
    weights = weights / weights.sum()

    mean_embedding = (weights.unsqueeze(-1) * embeddings).sum(dim=0)
    centered = embeddings - mean_embedding.unsqueeze(0)

    covariance = (weights.unsqueeze(-1) * centered).T @ centered

    return covariance, mean_embedding
\end{lstlisting}

\begin{lstlisting}[caption={Spectral decomposition to identify the selected subspace}]
def explain_spectral_decomposition(embeddings, weights, k):
    """
    Identify the k-dimensional selected subspace via eigendecomposition.
    """
    covariance, _ = explain_score_weighted_covariance(
        embeddings, weights
    )

    eigenvalues, eigenvectors = torch.linalg.eigh(covariance)

    idx = eigenvalues.argsort(descending=True)
    eigenvalues = eigenvalues[idx]
    eigenvectors = eigenvectors[:, idx]

    V_k = eigenvectors[:, :k]
    top_eigenvalues = eigenvalues[:k]

    return V_k, top_eigenvalues
\end{lstlisting}

\begin{lstlisting}[caption={\methodname{} model with a low-rank additive adjustment}]
class MARLAModel(nn.Module):
    """
    Rank-constrained additive adjustment model.
    """
    def __init__(self, base_weight, base_bias, V_k):
        super().__init__()
        self.base_weight = base_weight
        self.base_bias = base_bias
        self.V_k = V_k

        k, num_classes = V_k.shape[1], base_weight.shape[0]
        self.adjustment = nn.Parameter(
            torch.zeros(k, num_classes)
        )

    def forward(self, embeddings):
        base_logits = F.linear(
            embeddings, self.base_weight, self.base_bias
        )

        projected = F.linear(
            embeddings, self.V_k.t(), None
        )

        adjustment_logits = F.linear(
            projected, self.adjustment.t(), None
        )

        return base_logits + adjustment_logits
\end{lstlisting}

\section{Motivation for the Unknown-Groups Setting}
\label{app:unknown_groups}
\subsection{Proof of Proposition~\ref{prop:unknown_group_dro}}
\label{app:prop_proof}

For this analysis, we identify hypotheses $h,h'\in\mathcal{H}$ as equal if and only if they induce the same subgroup risks, that is,
\[
h=h'
\quad\Longleftrightarrow\quad
\forall G\in\mathcal{G}:\mathcal{R}_G(h)=\mathcal{R}_G(h').
\]
We assume throughout that each group has positive probability mass.

\begin{lemma}
\label{lemma:l1}
    Let $h_{\mathrm{ERM}}\in\mathcal{H}$ minimize the population risk
    \[
    \mathcal{R}(h)=\sum_{G\in\mathcal{G}}\pi_G\mathcal{R}_G(h).
    \]
    Then, for any $h'\neq h_{\mathrm{ERM}}$, there exists some group $G\in\mathcal{G}$ such that
    \[
    \mathcal{R}_G(h')>\mathcal{R}_G(h_{\mathrm{ERM}}).
    \]
\end{lemma}

\begin{proof}
    Suppose for contradiction that this was not the case, so that for all $G\in\mathcal{G}$,
    \[
    \mathcal{R}_G(h')\leq \mathcal{R}_G(h_{\mathrm{ERM}}).
    \]
    Multiplying by $\pi_G>0$ and summing over groups gives
    \[
    \sum_{G\in\mathcal{G}}\pi_G\mathcal{R}_G(h')
    \leq
    \sum_{G\in\mathcal{G}}\pi_G\mathcal{R}_G(h_{\mathrm{ERM}}).
    \]
    Since $h'\neq h_{\mathrm{ERM}}$ under our subgroup-risk identification, at least one subgroup inequality is strict. Since all $\pi_G>0$, the weighted average inequality is also strict, so
    \[
    \mathcal{R}(h')<\mathcal{R}(h_{\mathrm{ERM}}),
    \]
    contradicting the optimality of $h_{\mathrm{ERM}}$. Thus, a group with the desired property must exist.
\end{proof}

\paragraph{Proof of Proposition~\ref{prop:unknown_group_dro}.}

\begin{proof}
For any nonempty $\mathcal{G}'\subseteq\mathcal{G}$, write
\[
M(\mathcal{G}')
=
\min_{h\in\mathcal{H}}
\max_{G\in\mathcal{G}'}\mathcal{R}_G(h)
\]
for the optimal Group DRO value over $\mathcal{G}'$.

\textbf{Case 1: There exists a full-information DRO solution not equivalent to ERM.}
Let
\[
h_{\mathrm{DRO}+}\in\mathcal{H}_{\mathrm{DRO}}(\mathcal{G})
\]
be such that $h_{\mathrm{DRO}+}\neq h_{\mathrm{ERM}}$. By Lemma~\ref{lemma:l1}, there exists $G_i\in\mathcal{G}$ such that
\[
\mathcal{R}_{G_i}(h_{\mathrm{DRO}+})
>
\mathcal{R}_{G_i}(h_{\mathrm{ERM}}).
\]
We claim that the proposition holds for this index $i$. Let $\mathcal{S}=\mathcal{G}\setminus\{G_i\}$. If condition (i) holds, then we are done. Otherwise, take any $h_{\mathrm{DRO}}\in\mathcal{H}_{\mathrm{DRO}}(\mathcal{S})$. Since $\mathcal{S}\subseteq\mathcal{G}$, we have $M(\mathcal{S})\leq M(\mathcal{G})$. We now do casework.

\textbf{Case 1a:} $M(\mathcal{S})<M(\mathcal{G})$.
Suppose further in this case that
\[
\mathcal{R}_{G_i}(h_{\mathrm{DRO}})
<
\mathcal{R}_{G_i}(h_{\mathrm{DRO}+}).
\]
Then $\mathcal{R}_{G_i}(h_{\mathrm{DRO}})<M(\mathcal{G})$, and by the case assumption,
\[
\max_{G\in\mathcal{S}}\mathcal{R}_G(h_{\mathrm{DRO}})=M(\mathcal{S})<M(\mathcal{G}).
\]
Therefore,
\begin{align*}
    \max_{G\in\mathcal{G}}\mathcal{R}_G(h_{\mathrm{DRO}})
    &=
    \max\left\{
    \mathcal{R}_{G_i}(h_{\mathrm{DRO}}),
    \max_{G\in\mathcal{S}}\mathcal{R}_G(h_{\mathrm{DRO}})
    \right\} \\
    &< M(\mathcal{G}),
\end{align*}
which contradicts the optimality of the full-information Group DRO value $M(\mathcal{G})$. Hence
\[
\mathcal{R}_{G_i}(h_{\mathrm{DRO}})
\geq
\mathcal{R}_{G_i}(h_{\mathrm{DRO}+})
>
\mathcal{R}_{G_i}(h_{\mathrm{ERM}}),
\]
so condition (ii) holds.

\textbf{Case 1b:} $M(\mathcal{S})=M(\mathcal{G})$.
Since $h_{\mathrm{DRO}+}$ minimizes the Group DRO objective over $\mathcal{G}$,
\[
\max_{G\in\mathcal{S}}\mathcal{R}_G(h_{\mathrm{DRO}+})
\leq
\max_{G\in\mathcal{G}}\mathcal{R}_G(h_{\mathrm{DRO}+})
=
M(\mathcal{G})=M(\mathcal{S}).
\]
By minimality of $M(\mathcal{S})$, equality must hold, and therefore $h_{\mathrm{DRO}+}\in\mathcal{H}_{\mathrm{DRO}}(\mathcal{S})$. Since
\[
\mathcal{R}_{G_i}(h_{\mathrm{DRO}+})
>
\mathcal{R}_{G_i}(h_{\mathrm{ERM}}),
\]
condition (ii) holds with $h=h_{\mathrm{DRO}+}$.

\textbf{Case 2: Every full-information DRO solution is equivalent to ERM.}
Then $h_{\mathrm{ERM}}\in\mathcal{H}_{\mathrm{DRO}}(\mathcal{G})$. Choose
\[
i\in\argmin_{j\in\{1,\dots,K\}}\mathcal{R}_{G_j}(h_{\mathrm{ERM}})
\]
and let $\mathcal{S}=\mathcal{G}\setminus\{G_i\}$. If condition (i) holds, then we are done. Otherwise, take any $h_{\mathrm{DRO}}\in\mathcal{H}_{\mathrm{DRO}}(\mathcal{S})$. Since $h_{\mathrm{ERM}}\notin\mathcal{H}_{\mathrm{DRO}}(\mathcal{S})$,
\[
\max_{G\in\mathcal{S}}\mathcal{R}_G(h_{\mathrm{DRO}})
<
\max_{G\in\mathcal{S}}\mathcal{R}_G(h_{\mathrm{ERM}}).
\]
Let
\[
m=\mathcal{R}_{G_i}(h_{\mathrm{ERM}})=\min_{G\in\mathcal{G}}\mathcal{R}_G(h_{\mathrm{ERM}}),
\qquad
M=\max_{G\in\mathcal{G}}\mathcal{R}_G(h_{\mathrm{ERM}}).
\]
Since $h_{\mathrm{ERM}}$ is full-information DRO-optimal, $M=M(\mathcal{G})$. We claim that $\mathcal{R}_{G_i}(h_{\mathrm{DRO}})>m$. Suppose not, so that $\mathcal{R}_{G_i}(h_{\mathrm{DRO}})\leq m$.

If $m<M$, then omitting the minimum-risk group does not omit any worst-risk group for $h_{\mathrm{ERM}}$, so
\[
\max_{G\in\mathcal{S}}\mathcal{R}_G(h_{\mathrm{ERM}})=M.
\]
Thus $\max_{G\in\mathcal{S}}\mathcal{R}_G(h_{\mathrm{DRO}})<M$, and together with $\mathcal{R}_{G_i}(h_{\mathrm{DRO}})\leq m<M$, this gives
\[
\max_{G\in\mathcal{G}}\mathcal{R}_G(h_{\mathrm{DRO}})<M,
\]
contradicting full-information DRO optimality.

If $m=M$, then all subgroup risks of $h_{\mathrm{ERM}}$ are equal to $M$. The inequalities
\[
\max_{G\in\mathcal{S}}\mathcal{R}_G(h_{\mathrm{DRO}})<M,
\qquad
\mathcal{R}_{G_i}(h_{\mathrm{DRO}})\leq M
\]
imply
\[
\mathcal{R}_G(h_{\mathrm{DRO}})
\leq
\mathcal{R}_G(h_{\mathrm{ERM}})
\quad
\text{for all }G\in\mathcal{G},
\]
with strict inequality for at least one group. Since all groups have positive mass, this gives $\mathcal{R}(h_{\mathrm{DRO}})<\mathcal{R}(h_{\mathrm{ERM}})$, contradicting ERM optimality.

Therefore,
\[
\mathcal{R}_{G_i}(h_{\mathrm{DRO}})
>
\mathcal{R}_{G_i}(h_{\mathrm{ERM}}),
\]
so condition (ii) holds. This completes the proof.
\end{proof}

\subsection{Synthetic Data Experiment}
\label{app:synthetic}
We generate synthetic data with \(N\) known subgroups and one unknown subgroup (indexed as group \(0\)), where each group exhibits conflicting spurious correlations. Each example has seven features in total: five stable features sampled from a standard Gaussian distribution \(\mathcal{N}(0, I)\) that determine the true label, and two spurious features that induce group-specific correlations. Binary labels are generated as
\[
y = \mathrm{sign}(\theta^\top x_{\text{stable}} + \epsilon),
\]
where \(\theta\) is a randomly sampled unit vector scaled to signal strength 1.5, \(x_{\text{stable}} \in \mathbb{R}^5\) denotes the stable features, and \(\epsilon \sim \mathcal{N}(0, 0.8^2)\) is Gaussian label noise. The unknown subgroup (group \(0\)), with size ratio \(\rho\) relative to each known group, has spurious feature 0 strongly correlated with the label (strength 4.5) and spurious feature 1 anti-correlated (strength 3.0). In contrast, the known subgroups (groups \(1\) through \(N\), with 1000 samples each) exhibit the opposite pattern: spurious feature 1 is positively correlated with the label (strength 4.0), while spurious feature 0 is anti-correlated (strength 3.0). Spurious features are generated deterministically from the labels with added Gaussian noise (\(\sigma = 0.5\)), creating a setting in which optimizing performance on known groups induces systematic errors on the unknown group. We split the dataset into 60\% training, 20\% validation, and 20\% test sets. We train three linear classifiers, each consisting of a single fully connected layer mapping from seven input features to two output classes: (1) ERM, which minimizes the average cross-entropy loss across all training data using SGD (learning rate 0.01, momentum 0.9, batch size 64) for 100 epochs, (2) Group DRO, which is trained on all data but updates adversarial group weights only for the known groups, assigning zero loss weight to the unknown group, with step size \(\eta = 0.01\), and (3) \methodname{} with rank $k=1$, $\gamma=100$, and no regularization coefficient. All methods observe the same training data. Group DRO, however, updates its group weights using losses from the known groups only. For \methodname{}, we do not use early stopping or additional hyperparameter tuning. We evaluate all three methods on every group, including the unknown subgroup, to test whether robustness optimization over known groups can degrade performance on an omitted subgroup relative to ERM. Results are averaged over three seeds (\texttt{0, 1,} and \texttt{42}).

\begin{table}[!h]
\centering
\small
\setlength{\tabcolsep}{2pt}
\renewcommand{\arraystretch}{0.9}
\caption{Unknown-group accuracy (UGA) across synthetic parameter sweeps. \(N\) denotes the number of known groups and \(\rho\) denotes the unknown group size ratio. UGA is reported for ERM and Group DRO. Harm is defined as the percentage-point decrease in UGA when using Group DRO relative to ERM.}
\label{tab:synthetic_uga_full}
\begin{tabular}{c|c|c|c|c|c}
\toprule
\textbf{\(N\)} & \textbf{\(\rho\)} & \textbf{ERM UGA (\%)} & \textbf{GDRO UGA (\%)} & \textbf{Harm (\%)} & \textbf{\methodname{} UGA (\%)}\\
\midrule
1 & 0.1 & $77.5 \pm15.5$ & $55.4 \pm7.2$ & $22.1 \pm8.5$ & $85.5 \pm11.0$ \\
1 & 0.2 & $84.9 \pm0.4$ & $49.1 \pm9.0$ & $35.8 \pm9.2$ & $91.5 \pm7.8$ \\
1 & 0.3 & $85.1 \pm4.7$ & $44.8 \pm4.8$ & $40.3 \pm9.1$ & $86.9 \pm7.6$ \\
\midrule
2 & 0.1 & $58.3 \pm7.6$ & $42.5 \pm5.9$ & $15.9 \pm3.6$ & $71.0 \pm19.0$ \\
2 & 0.2 & $82.3 \pm4.0$ & $42.2 \pm11.2$ & $40.1 \pm8.8$ & $89.6 \pm4.8$ \\
2 & 0.3 & $75.6 \pm4.9$ & $49.8 \pm9.1$ & $25.8 \pm13.3$ & $88.7 \pm1.4$ \\
\midrule
3 & 0.1 & $63.6 \pm8.8$ & $49.4 \pm9.2$ & $14.2 \pm2.2$ & $55.3 \pm11.5$ \\
3 & 0.2 & $74.8 \pm13.5$ & $57.2 \pm8.9$ & $17.6 \pm7.6$ & $76.2 \pm10.0$ \\
3 & 0.3 & $65.6 \pm3.9$ & $43.4 \pm1.0$ & $22.3 \pm2.9$ & $74.8 \pm18.5$ \\
\midrule
4 & 0.1 & $66.2 \pm4.4$ & $60.2 \pm7.4$ & $6.0 \pm6.3$ & $64.1 \pm4.1$ \\
4 & 0.2 & $61.5 \pm3.6$ & $51.3 \pm8.2$ & $10.2 \pm4.8$ & $79.9 \pm19.7$ \\
4 & 0.3 & $73.0 \pm10.9$ & $50.8 \pm4.8$ & $22.2 \pm6.3$ & $90.2 \pm6.6$ \\
\bottomrule
\end{tabular}

\end{table}

\begin{figure}[!h]
    \centering
    \includegraphics[width=0.85\linewidth]{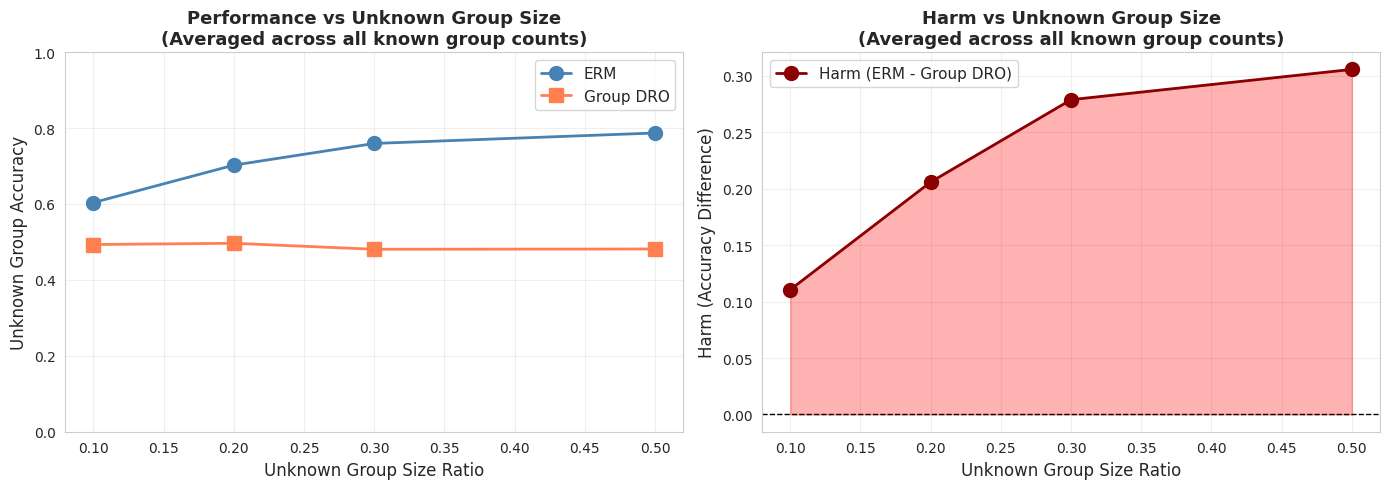}
    \caption{Performance as a function of the unknown-group size ratio. Left: unknown-group accuracy for ERM and Group DRO, averaged over the number of known groups. ERM accuracy improves as the unknown subgroup becomes larger, while Group DRO remains approximately constant or decreases. Right: harm, defined as ERM accuracy minus Group DRO accuracy, increases with the unknown-group size ratio.}
    \label{fig:unknown_group_size}
\end{figure}

\begin{figure}[!h]
    \centering
    \includegraphics[width=0.85\linewidth]{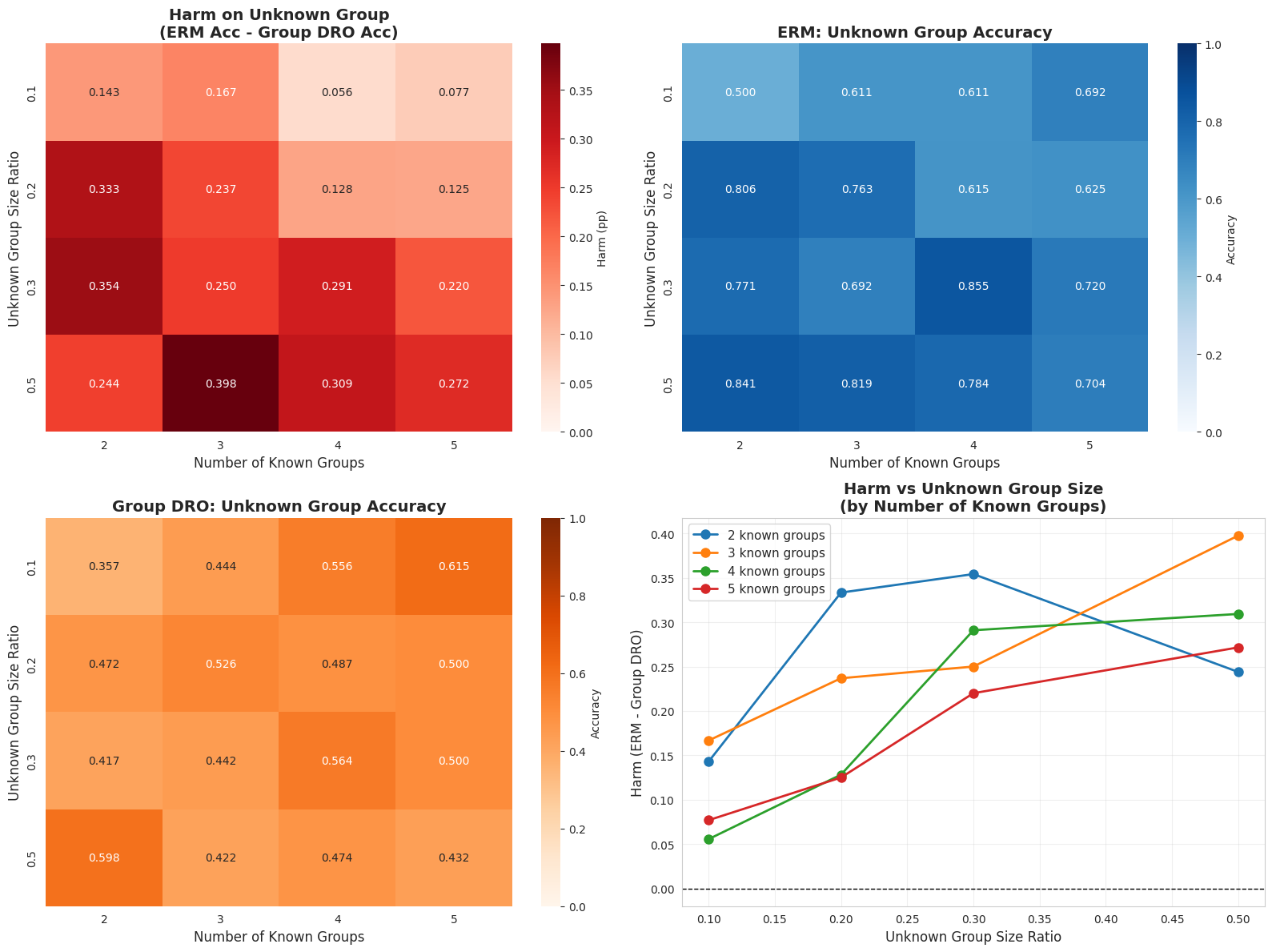}
    \caption{Performance on the unknown subgroup across parameter settings. The heatmaps show harm on the unknown subgroup (top left), ERM accuracy (top right), Group DRO accuracy (bottom left), and harm as a function of the unknown-group size ratio for different numbers of known groups (bottom right). Results are shown for 2-5 known groups and unknown-group size ratios from 0.1 to 0.5. Warmer colors indicate larger values.}
    \label{fig:heatmap_synthetic}
    \vspace{-0.3cm}
\end{figure}

\newpage
\section{Experimental Settings}
\label{sec:expdetails}

\subsection{Datasets}
\label{app:datasets}
We evaluate on seven datasets spanning several group-robustness settings. 

\begin{table}[!h]
\centering
\small
\setlength{\tabcolsep}{4pt}
\renewcommand{\arraystretch}{1.05}

\caption{Example inputs from datasets}
\label{App.:table:sample-inputs-combined}

\begin{tabularx}{\linewidth}{l X}
\toprule[1.5pt]
\textbf{Dataset} & \textbf{Examples} \\
\midrule\midrule

\multicolumn{2}{l}{\textit{Image datasets}} \\
\midrule

\texttt{Waterbirds} &
\raisebox{-.4\height}{\includegraphics[width=35pt,height=35pt]{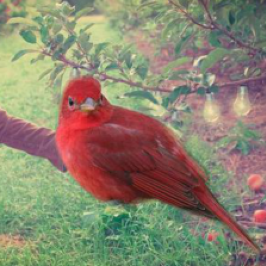}}
\hspace{3pt}
\raisebox{-.4\height}{\includegraphics[width=35pt,height=35pt]{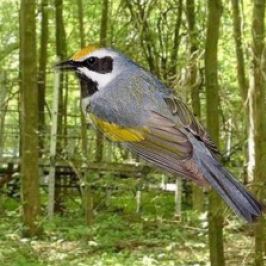}}
\hspace{3pt}
\raisebox{-.4\height}{\includegraphics[width=35pt,height=35pt]{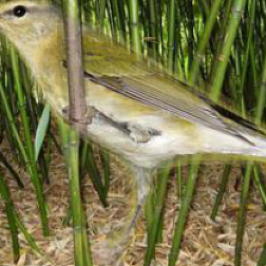}}
\hspace{3pt}
\raisebox{-.4\height}{\includegraphics[width=35pt,height=35pt]{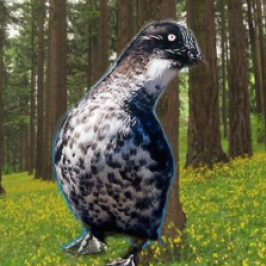}}
\hspace{3pt}
\raisebox{-.4\height}{\includegraphics[width=35pt,height=35pt]{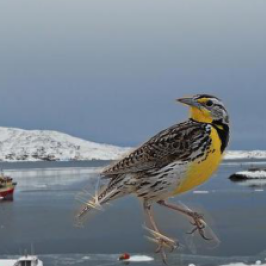}}
\hspace{3pt}
\raisebox{-.4\height}{\includegraphics[width=35pt,height=35pt]{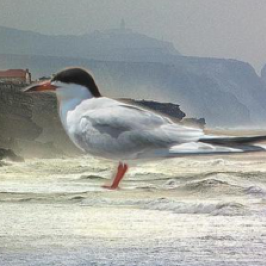}} \\

\midrule

\texttt{CelebA} &
\raisebox{-.4\height}{\includegraphics[width=35pt,height=35pt]{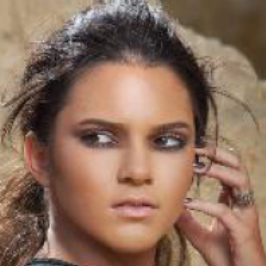}}
\hspace{3pt}
\raisebox{-.4\height}{\includegraphics[width=35pt,height=35pt]{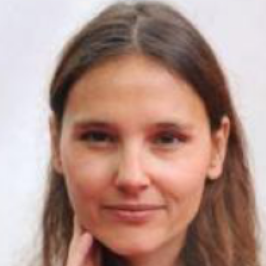}}
\hspace{3pt}
\raisebox{-.4\height}{\includegraphics[width=35pt,height=35pt]{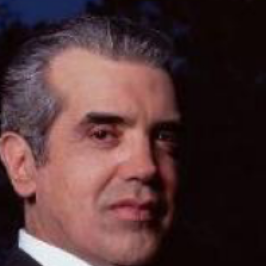}}
\hspace{3pt}
\raisebox{-.4\height}{\includegraphics[width=35pt,height=35pt]{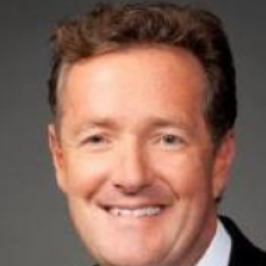}}
\hspace{3pt}
\raisebox{-.4\height}{\includegraphics[width=35pt,height=35pt]{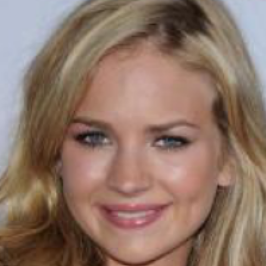}}
\hspace{3pt}
\raisebox{-.4\height}{\includegraphics[width=35pt,height=35pt]{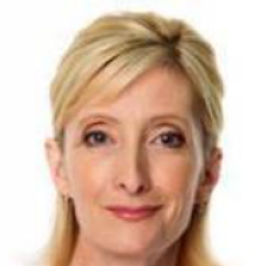}} \\

\midrule

\texttt{CheXpert} &
\raisebox{-.4\height}{\includegraphics[width=35pt,height=35pt]{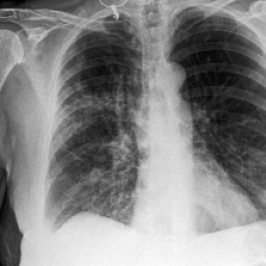}}
\hspace{3pt}
\raisebox{-.4\height}{\includegraphics[width=35pt,height=35pt]{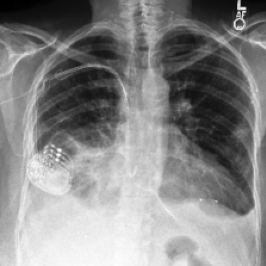}}
\hspace{3pt}
\raisebox{-.4\height}{\includegraphics[width=35pt,height=35pt]{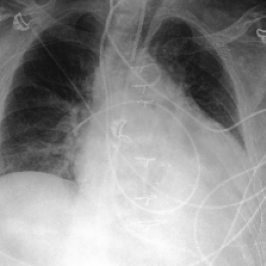}}
\hspace{3pt}
\raisebox{-.4\height}{\includegraphics[width=35pt,height=35pt]{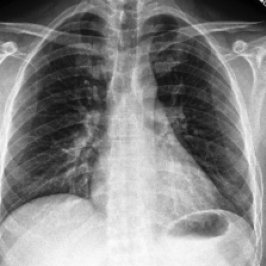}}
\hspace{3pt}
\raisebox{-.4\height}{\includegraphics[width=35pt,height=35pt]{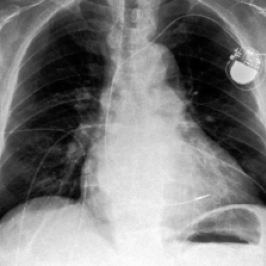}}
\hspace{3pt}
\raisebox{-.4\height}{\includegraphics[width=35pt,height=35pt]{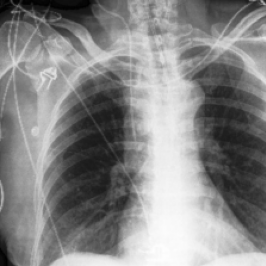}} \\

\midrule\midrule

\multicolumn{2}{l}{\textit{Text datasets}} \\
\midrule

\texttt{CivilComments} &
\begin{minipage}[t]{\linewidth}
\raggedright
``Munchins looks like a munchins. The man who dont want to show his taxes, will tell you everything...''

``The democratic party removed the filibuster to steamroll its agenda. Suck it up boys and girls.''

``so you dont use oil? no gasoline? no plastic? man you ignorant losers are pathetic.''
\end{minipage} \\

\midrule

\texttt{MultiNLI} &
\begin{minipage}[t]{\linewidth}
\raggedright
``The analysis proves that there is no link between PM and bronchitis.''

``Postal Service were to reduce delivery frequency.''

``The famous tenements (or lands) began to be built.''
\end{minipage} \\

\midrule\midrule

\multicolumn{2}{l}{\textit{Tabular datasets}} \\
\midrule

\texttt{MIMIC-IV} &
\textbf{Outcome:} 30-day readmission = No; in-hospital death = No.
\textbf{Demographics:} 58-year-old female, White, Medicare.
\textbf{Admission:} emergency admission; length of stay = 5 days; discharged home; no ICU admission; prior admissions = 2; diagnoses = 10. \\

\midrule

\texttt{GRACE} &
\textbf{Outcome:} In-hospital death = No.
\textbf{Demographics:} 66-year-old male.
\textbf{Vitals:} SBP 138 mmHg, pulse 82 bpm.
\textbf{Labs:} creatinine 1.1 mg/dL, CK-MB 18 U/L.
\textbf{Clinical:} no cardiac arrest, no history of CHF, Killip class II.
\textbf{ECG:} inferior ST elevation present, no anterior ST depression. \\

\bottomrule[1.5pt]
\end{tabularx}

\vspace{-0.3cm}
\end{table}

\textbf{\textsc{Waterbirds}} \citep{sagawa2020GDRO} is an image classification benchmark whose task is to classify birds as `waterbirds' (birds that frequent water) or `landbirds' (birds that do not). The dataset is an example of \textit{spurious correlations}: waterbirds often appears in images with water, while landbirds do not, causing models to often fail to classify waterbirds without water backgrounds or landbirds with water backgrounds.

\textbf{\textsc{CelebA}} \citep{7410782} is an image classification benchmark whose task is to classify celebrities as blond or not blond. The dataset is an example of \textit{attribute imbalance}, as blond hair is found more frequently in females, making it hard to correctly identify blond males or non-blond females.

\textbf{\textsc{MultiNLI}} \citep{williams2018broadcoveragechallengecorpussentence} is a text classification dataset where models must predict the relationship between two sentences from the possibilities of `entailment', `contradiction', or `neutral'. This dataset is an example of \textit{spurious correlations}: many examples of the `contradiction' label contain negations, which may cause models to associate the presence of these erroneously with contradictions.

\textbf{\textsc{CivilComments}} \cite{borkan2019nuanced} is a text classification dataset whose task is to classify online comments as `toxic' and `non-toxic'. We use the version from the WILDS benchmark \cite{koh2021wildsbenchmarkinthewilddistribution}. The comments are tagged with mentioned demographic group. This dataset is an example of \textit{class imbalance}: some of these groups are underrepresented, and models run the risk of associating group membership with the solution to the classification task.

\textbf{\textsc{CheXpert}} \citep{irvin2019chexpertlargechestradiograph} is a dataset of chest radiographs where the task is to predict whether or not there are abnormalities. The dataset is an example of significant \textit{class imbalance}, as some conditions (varieties of abnormalities) are not frequently found in the training dataset.

\textbf{\textsc{MIMIC-IV}} \citep{johnson2023mimic} is an electronic health record dataset comprising ICU stays at Beth Israel Deaconess Medical Center spanning 2008 to 2019. We curated a cohort of 431,231 patients from the MICU/ICU with 17 clinical features, with gender and race as sensitive attributes (see Table~\ref{tab:cohort_mimic}). This dataset exhibits \textit{attribute imbalance}: racial and demographic groups are unevenly distributed, posing risks of degraded model performance on underrepresented 
subgroups.

\textbf{\textsc{GRACE}} \citep{granger2003predictors, fox2014should, fox2006prediction, rationale2001design, eagle2002practice} is a multinational registry of patients admitted with acute coronary syndrome (ACS), spanning 15 countries and healthcare systems. Age and sex serve as sensitive demographic attributes (see Table~\ref{tab:cohort_grace}). The dataset reflects \textit{attribute imbalance}: the uneven distribution of age and sex subgroups risks biasing models toward majority demographics.

\begin{table*}[!h]
\centering
\caption{Cohort Characteristics of MIMIC-IV}
\label{tab:cohort}

\resizebox{\textwidth}{!}{
\begin{tabular}{lcccc}
\toprule
 & \textbf{Train (N=258,944)} & \textbf{Val (N=86,589)} & \textbf{Test (N=85,698)} & \textbf{Overall (N=431,231)} \\
\midrule
\multicolumn{5}{l}{\textbf{OUTCOME}} \\
30-day readmission & 52,580 (20.3\%) & 17,404 (20.1\%) & 17,102 (20.0\%) & 87,086 (20.2\%) \\
In-hospital death & 5,178 (2.0\%) & 1,704 (2.0\%) & 1,727 (2.0\%) & 8,609 (2.0\%) \\
\midrule
\multicolumn{5}{l}{\textbf{DEMOGRAPHICS}} \\
Age, years (mean, SD) & 56.8 (19.0) & 57.1 (19.1) & 56.9 (19.0) & 56.9 (19.1) \\
Female sex & 134,251 (51.8\%) & 45,557 (52.6\%) & 45,182 (52.7\%) & 224,990 (52.2\%) \\
\textit{Race/ethnicity} \\
\quad White & 173,041 (66.8\%) & 58,113 (67.1\%) & 57,000 (66.5\%) & 288,154 (66.8\%) \\
\quad Black/African American & 42,209 (16.3\%) & 13,766 (15.9\%) & 13,983 (16.3\%) & 69,958 (16.2\%) \\
\quad Hispanic/Latino & 15,114 (5.8\%) & 5,048 (5.8\%) & 5,050 (5.9\%) & 25,212 (5.8\%) \\
\quad Asian & 8,693 (3.4\%) & 3,240 (3.7\%) & 2,960 (3.5\%) & 14,893 (3.5\%) \\
\quad Other/Unknown & 19,887 (7.7\%) & 6,422 (7.4\%) & 6,705 (7.8\%) & 33,014 (7.7\%) \\
\textit{Insurance} \\
\quad Medicare & 96,138 (37.1\%) & 32,924 (38.0\%) & 31,498 (36.8\%) & 160,560 (37.2\%) \\
\quad Medicaid & 24,767 (9.6\%) & 8,087 (9.3\%) & 8,476 (9.9\%) & 41,330 (9.6\%) \\
\quad Other & 138,039 (53.3\%) & 45,578 (52.6\%) & 45,724 (53.4\%) & 229,341 (53.2\%) \\
\midrule
\multicolumn{5}{l}{\textbf{ADMISSION CHARACTERISTICS}} \\
Emergency admission & 216,649 (83.7\%) & 72,654 (83.9\%) & 71,799 (83.8\%) & 361,102 (83.7\%) \\
Length of stay, days (mean, SD) & 4.5 (6.6) & 4.5 (6.6) & 4.5 (6.7) & 4.5 (6.6) \\
Discharged to home & 138,963 (53.7\%) & 46,347 (53.5\%) & 45,685 (53.3\%) & 230,995 (53.6\%) \\
ICU admission & 39,868 (15.4\%) & 13,261 (15.3\%) & 13,110 (15.3\%) & 66,239 (15.4\%) \\
ICU LOS, days (ICU pts only, mean, SD) & 3.8 (5.7) & 3.8 (5.6) & 3.8 (5.4) & 3.8 (5.6) \\
No. prior admissions (mean, SD) & 3.2 (7.5) & 3.2 (6.9) & 3.3 (10.0) & 3.2 (8.0) \\
No. diagnoses (mean, SD) & 11 (7) & 11 (7) & 11 (7) & 11 (7) \\
\bottomrule
\end{tabular}
}
\label{tab:cohort_mimic}
\end{table*}

\begin{table*}[!h]
\centering
\caption{Cohort Characteristics of GRACE}
\label{tab:grace_cohort}

\resizebox{\textwidth}{!}{
\begin{tabular}{lcccc}
\toprule
 & \textbf{Train (N=13,442)} & \textbf{Val (N=4,480)} & \textbf{Test (N=4,482)} & \textbf{Overall (N=22,404)} \\
\midrule

\multicolumn{5}{l}{\textbf{OUTCOME}} \\
In-hospital death & 472 (3.5\%) & 172 (3.8\%) & 173 (3.9\%) & 817 (3.6\%) \\

\midrule
\multicolumn{5}{l}{\textbf{DEMOGRAPHICS}} \\
Age, years (mean, SD) & 65.9 (13.6) & 65.7 (13.7) & 65.9 (13.7) & 65.9 (13.7) \\
Female sex & 4,358 (32.4\%) & 1,464 (32.7\%) & 1,429 (31.9\%) & 7,251 (32.4\%) \\
Male sex & 9,084 (67.6\%) & 3,016 (67.3\%) & 3,053 (68.1\%) & 15,153 (67.6\%) \\

\midrule
\multicolumn{5}{l}{\textbf{CLINICAL FEATURES}} \\
Systolic blood pressure, mmHg (mean, SD) & 141.7 (29.7) & 140.7 (29.0) & 141.5 (29.5) & 141.5 (29.5) \\
Pulse, bpm (mean, SD) & 79.7 (21.1) & 79.9 (21.4) & 80.2 (21.7) & 79.9 (21.3) \\
Creatinine, mg/dL (mean, SD) & 1.2 (0.8) & 1.2 (0.9) & 1.2 (0.9) & 1.2 (0.9) \\
CPK, initial (mean, SD) & 301.1 (660.2) & 306.2 (679.4) & 302.9 (645.5) & 302.5 (661.2) \\
CPK upper limit of normal (mean, SD) & 193.7 (94.6) & 193.8 (82.0) & 193.7 (124.1) & 193.7 (98.9) \\
CK-MB, initial (mean, SD) & 26.4 (63.5) & 28.3 (67.2) & 30.4 (72.8) & 27.6 (66.2) \\
CK-MB upper limit of normal (mean, SD) & 9.3 (15.7) & 9.7 (26.2) & 9.4 (19.1) & 9.4 (18.9) \\
History of congestive heart failure & 1,404 (10.4\%) & 463 (10.3\%) & 509 (11.4\%) & 2,376 (10.6\%) \\
Cardiac arrest at presentation & 253 (1.9\%) & 87 (1.9\%) & 81 (1.8\%) & 421 (1.9\%) \\

\textit{Killip class} \\
\quad Class I & 11,260 (83.8\%) & 3,763 (84.0\%) & 3,693 (82.4\%) & 18,716 (83.5\%) \\
\quad Class II & 1,549 (11.5\%) & 533 (11.9\%) & 553 (12.3\%) & 2,635 (11.8\%) \\
\quad Class III & 500 (3.7\%) & 156 (3.5\%) & 191 (4.3\%) & 847 (3.8\%) \\
\quad Class IV & 133 (1.0\%) & 28 (0.6\%) & 45 (1.0\%) & 206 (0.9\%) \\

\midrule
\multicolumn{5}{l}{\textbf{ECG FEATURES}} \\
Anterior ST depression & 1,867 (13.9\%) & 665 (14.8\%) & 653 (14.6\%) & 3,185 (14.2\%) \\
Anterior ST elevation & 2,218 (16.5\%) & 746 (16.7\%) & 709 (15.8\%) & 3,673 (16.4\%) \\
Inferior ST depression & 1,264 (9.4\%) & 473 (10.6\%) & 448 (10.0\%) & 2,185 (9.8\%) \\
Inferior ST elevation & 2,539 (18.9\%) & 857 (19.1\%) & 863 (19.3\%) & 4,259 (19.0\%) \\
Lateral ST depression & 2,487 (18.5\%) & 828 (18.5\%) & 836 (18.7\%) & 4,151 (18.5\%) \\
Lateral ST elevation & 1,302 (9.7\%) & 444 (9.9\%) & 460 (10.3\%) & 2,206 (9.8\%) \\
Any ST change & 1,400 (10.4\%) & 450 (10.0\%) & 475 (10.6\%) & 2,325 (10.4\%) \\

\bottomrule
\end{tabular}
}
\label{tab:cohort_grace}
\end{table*}

\subsection{Baselines}
\label{sec:App._baselines}
\paragraph{Empirical Risk Minimization (ERM)}optimizes for average performance across the entire training distribution. While effective in balanced settings, ERM is prone to prioritizing majority-group characteristics, often at the expense of performance on underrepresented subpopulations. By optimizing a global average, ERM is frequently susceptible to spurious correlations and performance degradation under subpopulation shift. 
\paragraph{Classifier Re-train (CRT, ReWeightCRT)}\cite{kang2020decouplingrepresentationclassifierlongtailed} decouples representation learning from classifier learning by first training a feature extractor using standard ERM, then retraining the classifier with class-balanced reweighting or resampling. This improves performance on underrepresented classes without modifying the learned representation.
\paragraph{Just Train Twice (JTT)}\cite{liu2021jtt}  
is a two-stage method that first trains a model using ERM to identify hard examples, defined as samples misclassified by the initial model. In the second stage, the model is retrained while upweighting these hard examples, improving robustness to spurious correlations and subpopulation shifts.
\paragraph{Mixup}\cite{zhang2017mixup} is a simple, data-agnostic augmentation technique that constructs virtual training examples via convex combinations of pairs of inputs and labels. By encouraging the model to behave linearly between examples, Mixup implicitly regularizes the decision boundary, improves generalization, and reduces sensitivity to spurious correlations. 
\paragraph{LISA}\cite{yao2022improvingoutofdistributionrobustnessselective} improves out-of-distribution robustness by applying Mixup selectively across specific groups. LISA alternates between intra-label LISA, which interpolates samples with the same label but different domains to cancel out domain-specific noise, and intra-domain LISA, which interpolates samples within the same domain but with different labels to force the model to ignore domain-wide biases.
This encourages the learning of invariant predictors that generalize better to subpopulation and domain shifts.
\paragraph{Deep Feature Reweighting (DFR)}\cite{kirichenko2022dfr} 
first trains a model via standard ERM on the full dataset to develop a robust feature extractor using only the dataset and labels. Then, the classification head is discarded and a new one is trained from scratch using a balanced validation set, requiring group information. 

\paragraph{Diverse Prototypical Ensembles (DPE)}\cite{to2025diverseprototypicalensemblesimprove} addresses subpopulation shifts by replacing the standard linear classification head with an ensemble of diversified, distance-based prototype classifiers. 

DPE employs an explicit inter-prototype similarity loss to enforce functional diversity among ensemble members. This allows the model to capture multiple distinct decision boundaries, effectively discovering and representing latent subpopulations even when group annotations or distributions are unknown. 

\paragraph{Correct-n-Contrast (CnC)}\cite{zhang2022cnc} is a two-stage framework designed to improve worst-group performance by leveraging contrastive learning without explicit group labels. In the first stage, a standard ERM model is trained to identify samples prone to spurious correlations, using its predictions as proxies for latent group attributes. In the second stage, the model is trained using a supervised contrastive loss that treats samples with the same class but different proxy attributes as ``positives" (to be pulled together) and samples with different classes but the same proxy attributes as ``negatives" (to be pushed apart). This forces the encoder to prioritize class-invariant features while ignoring the spurious attributes identified in the first stage.
\paragraph{Automatic Feature Reweighting (AFR)}\cite{qiu2023simplefastgrouprobustness} is a last-layer retraining method designed to reduce a model's reliance on spurious correlations. AFR updates the classification head of a pre-trained ERM model using a weighted loss that automatically emphasizes examples where the initial model performs poorly. By upweighting these "hard" examples, which typically correspond to minority subpopulations, AFR improves robustness to distribution shifts without requiring explicit group annotations or expensive retraining of the entire network.
\paragraph{ Group Inference via data Comparison (GIC)}\cite{han2024improvinggrouprobustnessspurious} is a group inference framework that identifies latent spurious attributes by leveraging two datasets with differing group distributions. It trains a spurious attribute classifier based on two key properties of spurious correlations: (1) spurious attributes are highly predictive of the label, and (2) the strength of this predictiveness varies across datasets with different group proportions. By encouraging inferred attributes that exhibit dataset-dependent label correlations, GIC produces more accurate pseudo group labels, which can then be used by downstream robust learning algorithms to improve worst-group performance.
\paragraph{Invariant Risk Minimization (IRM)}\cite{arjovsky2019invariant} aims to learn predictors whose performance is stable across multiple training environments, thereby promoting out-of-distribution generalization. IRM formalizes this by seeking a data representation such that the optimal classifier on top of that representation is simultaneously optimal for all environments, effectively isolating invariant correlations while ignoring spurious ones. 
\paragraph{Conditional Value-at-Risk DRO (CVaRDRO)}\cite{duchi2020learningmodelsuniformperformance} proposes a distributionally robust optimization variant that dynamically upweights samples with the highest losses without requiring explicit group annotations. Based on the Conditional Value-at-Risk (CVaR) objective, this method focuses optimization on the tail of the loss distribution, which often correspond to minority subpopulations or examples misaligned with spurious correlations. 
\paragraph{Group DRO}\cite{sagawa2020GDRO} improves subpopulation robustness by minimizing a surrogate objective that upper-bounds the worst-group loss. Instead of optimizing average empirical risk, it maintains group-level weights (requiring explicit group annotations) and dynamically upweights groups with higher loss during training.

\subsection{Implementation}
\subsubsection{MLP Architecture for Tabular Datasets}
\label{app:mlp}
The GRACE MLP is a residual feedforward network designed for tabular clinical data.
The model takes as input a 25-dimensional feature vector and projects it to a
256-dimensional hidden representation via a linear layer, followed by dropout
(rate $= 0.3$) and a GELU activation. This is succeeded by four residual hidden
blocks, each consisting of a linear transformation $\mathbb{R}^{256} \to
\mathbb{R}^{256}$, GELU activation, and dropout, with a skip connection adding
the block's input to its output. The resulting 256-dimensional embedding is passed
to a linear classification head. All weights are initialised with Xavier uniform
initialisation and biases are set to $0.01$. The same architecture is used for MIMIC-IV with a 15-dimensional model input. Two additional recorded features, race and gender, are reserved for subgroup definition rather than provided as model inputs, giving 17 features in total.

\subsubsection{ERM Training}
In this section, we provide specific implementation details to reproduce our results. We first describe how the base ERM model is trained before applying \methodname{}'s rank-limited adaptation. For ERM training (see Table~\ref{tab:stage1_config}), we follow the ERM setup used in Automatic Feature Reweighting (AFR) \citep{qiu2023simplefastgrouprobustness} for the \textsc{Waterbirds}, \textsc{CelebA}, \textsc{CivilComments}, and \textsc{MultiNLI} datasets. For \textsc{CheXpert}, we use the SubpopBench \citep{yang2023changehardcloserlook} implementation of ERM till convergence. 

\begin{table}[!h]
\centering
\caption{Base ERM Training Configuration for All Datasets}
\label{tab:stage1_config}
\resizebox{0.95\textwidth}{!}{%
\begin{tabular}{lccccccc}
\toprule
\textbf{Dataset} & \textbf{Architecture} & \textbf{Epochs} & \textbf{Batch Size} & \textbf{Learning Rate} & \textbf{Weight Decay} & \textbf{Scheduler} & \textbf{Optimizer} \\
\midrule
\textsc{Waterbirds} & ResNet-50 (ImageNet) & 50 & 32 & 0.003 & $1 \times 10^{-4}$ & Constant & SGD \\
\textsc{CelebA} & ResNet-50 (ImageNet) & 20 & 100 & 0.003 & $1 \times 10^{-4}$ & Cosine & SGD \\
\textsc{CivilComments} & BERT-base-uncased & 3 & 32 & $2 \times 10^{-5}$ & 0.0 & Constant & AdamW \\
\textsc{MultiNLI} & BERT-base-uncased & 3 & 32 & $2 \times 10^{-5}$ & 0.0 & Constant & AdamW \\
\textsc{CheXpert} & ResNet-50 (ImageNet) & 13$^*$ & 108 & 0.001 & $1 \times 10^{-4}$ & Default$^{**}$ & SGD \\
\textsc{GRACE}    & MLP (6 layers, 256 hidden) & 5{,}001$^\dagger$ & 1024 & $1 \times 10^{-3}$ & $1 \times 10^{-5}$ & Constant & SGD \\
\textsc{MIMIC-IV} & MLP (6 layers, 256 hidden) & 5{,}001$^\dagger$ & 512  & $1 \times 10^{-3}$ & $1 \times 10^{-4}$ & Constant & SGD \\

\bottomrule
\end{tabular}%
}
\begin{flushleft}
\footnotesize
$^*$ CheXpert uses 20,001 training steps (approximately 13 epochs with batch size 108). \\
$^{**}$ CheXpert uses SubpopBench's default scheduler configuration.\\
$^\dagger$ GRACE and MIMIC-IV use 5{,}001 training steps with early stopping (patience $= 20$); the batch sizes above yield approximately 5 epochs (GRACE) and 2 epochs (MIMIC-IV) per 1{,}000 steps.

\end{flushleft}
\end{table}

\subsubsection{Setting-Based Experimental Details}
\label{app:setting_details}

\begin{table}[!h]
\centering
\caption{\methodname{} Stage 2 Hyperparameter Configurations Across Datasets}
\label{tab:stage2_config}
\resizebox{0.55\textwidth}{!}{%
\begin{tabular}{lcccc}
\toprule
\textbf{Dataset} & \textbf{Epochs} & \textbf{Gamma Values} & \textbf{Rank Values} \\
\midrule
\textsc{Waterbirds} & 100 & 7.5, 8.5 & 4, 16, 32 \\
\textsc{CelebA} & 100 & 0.5, 0.75, 1.25, 1.5 & 2, 4, 12, 36, 52 \\
\textsc{CivilComments} & 50 & 500, 1000, 5000 & 1, 4, 5 \\
\textsc{MultiNLI} & 50 & 200, 500, 1000 & 1, 2, 3 \\
\textsc{CheXpert} & 50 & 0.5, 1, 1.5 & 50, 52, 54 \\
\textsc{GRACE}    & 300 & 1.0, 5.0, 10.0, 20.0, 50.0 & 1, 2, 4, 8, 16 \\
\textsc{MIMIC-IV} & 500 & 1.0, 5.0, 10.0, 20.0       & 1, 2, 4, 8, 16 \\

\bottomrule
\end{tabular}
}
\begin{flushleft}
\footnotesize
\textbf{Note:} Across all datasets, the regularization coefficient (\texttt{reg\_coeff}) is set to 0.0. The stage-2 learning rate is 0.02 for Waterbirds, CelebA, CivilComments, MultiNLI, and CheXpert; 0.01 for GRACE; and 0.001 for MIMIC-IV.

\end{flushleft}
\vspace{-5pt}
\end{table}

For the no group information setting and complete group information, asterisked results in Tables \ref{tab:no_group_info} and \ref{tab:complete_group_info} reproduce results from SubpopBench's implementation \citep{yang2023changehardcloserlook}. DPE, GIC, AFR, and CnC results are from the methods' papers respectively \citep{to2025diverseprototypicalensemblesimprove, han2024improvinggrouprobustnessspurious, qiu2023simplefastgrouprobustness, zhang2022cnc}. In the setting with complete information of group relevance, the JTT, LISA, DFR, and Group DRO results are from the methods' papers respectively \citep{liu2021jtt, yao2022improvingoutofdistributionrobustnessselective, kirichenko2022dfr, sagawa2020GDRO}. 

For the partial group setting, we provide the details of the setup. 
The \textsc{CivilComments} dataset contains binary toxicity labels $y \in \{0,1\}$ and annotations for eight demographic attributes $\mathcal{A}$ = \{male, female, LGBTQ, Christian, Muslim, other religion, Black, White\}. In our case study setup, all methods are trained on the full training data. However, we define subgroups using only a \emph{single} attribute $a \in \mathcal{A}$. Specifically, we construct four training groups based on the cross-product of toxicity and the selected attribute:
\(
g \in \mathcal{G}_{\text{train}}(a)\) = \{\text{non-toxic no $a$}; \text{non-toxic has $a$};\ \text{toxic no $a$};\ \text{toxic has $a$}\}. All methods use only these four known groups for training and validation. At evaluation time, we assess performance across all demographic attributes. Concretely, we evaluate on the 16 test groups \(\mathcal{G}_{\text{test}} = \{(y, a_j) \mid y \in \{0,1\},\ a_j \in \mathcal{A}\}.\)
Groups with $a_j\neq a$ are \emph{unknown subgroups}. We focus on Black, LGBTQ, and Muslim as the known training attributes, as they are realistic in toxic content moderation and exhibit poor worst-group performance under ERM \citep{duchene2023benchmarktoxiccommentclassification}.

The \textsc{CheXpert} dataset contains binary labels $y\in\{0,1\}$ indicating the presence of a medical condition, with race $\in\{\text{White},\text{Black},\text{Other}\}$ and gender $\in\{\text{Male},\text{Female}\}$. As above, all methods are trained on the full dataset, but subgroups are defined using only one demographic attribute at a time. We consider: (i) \emph{Race-only} training with
\( \mathcal{G}_{\text{train}}(\text{race})=\{y\times\text{race}\}
\) and
(ii) \emph{Gender-only} training with \( \mathcal{G}_{\text{train}}(\text{gender})=\{y\times\text{gender}\}
\). Only the corresponding known groups are used for training and validation. Evaluation is performed over the full intersection,
\(\mathcal{G}_{\text{test}}=\{(y,\text{race},\text{gender})\},\)
yielding 12 test groups. Groups involving the attribute not used during training are \emph{unknown subgroups}. For example, in race-only training, worst-group performance is computed over all race$\times$gender$\times$label groups. These settings reflect realistic scenarios where only partial demographic information is available during training, while deployment requires robustness across full demographic intersections. We employ the same strategy for the \textsc{MIMIC-IV} and \textsc{GRACE} datasets. We consider attribute groups for GRACE across age and sex, yielding four subgroups (Young Male, Young Female, Old Male, Old Female, where Old denotes age $\geq 65$), and for MIMIC-IV across race and gender, yielding ten subgroups (five racial categories: White, Black, Hispanic, Asian, and Other; two gender categories: Male and Female).

We also provide complete details on all the hyperparameter tuning for baselines (see Table~\ref{tab:combined_baseline_hyperparams}). To ensure fair comparison between \methodname{} and baseline, we ensure sufficient tuning by considering configurations from SubpopBench and original papers. Note that we do early stopping on the validation worst group accuracy with known groups, making epochs not as important to tune across methods and having long training horizons to ensure we train till convergence. 

\begin{table}[!h]
\centering
\caption{Hyperparameter Sweep Configurations for Baseline Methods in for Partial Group Relevance Setting}
\label{tab:combined_baseline_hyperparams}
\resizebox{0.95\textwidth}{!}{%
\begin{tabular}{lcccccc}
\toprule
\textbf{Method} & \textbf{Search Type} & \textbf{Trials} & \textbf{Epochs} & \textbf{Learning Rate} & \textbf{Hyperparameters} \\
\midrule
\methodname{} & Grid & 10 & 100 & $2 \times 10^{-2}$ & $\gamma \in [100, 1000]$, $k \in [1, 4, 5, 6, 9]$, $\lambda = 0$ \\
AFR (CC) & Grid & 12 & 100 & $3 \times 10^{-3}$ & $\gamma \in [0, 0.01, 0.1, 1, 3, 10]$, $\lambda \in [0.0, 1 \times 10^{-2}]$ \\
AFR (CX) & Grid & 32 & 100 & $[10^{-4}, 10^{-3}, 10^{-2}, 5{\times}10^{-2}]$ & $\gamma \in [1, 2, 4, 6]$, $\lambda \in [0.0, 1 \times 10^{-2}]$ \\
DFR (CC) & Random & 25 & 100 & $1 \times 10^{-2}$  & $\text{dfr\_reg} \sim 10^{\text{Uniform}(-2.0, 0.5)}$ \\
DFR (CX) & Random & 25 & 100 & $1 \times 10^{-2}$  & $\text{dfr\_reg} \sim 10^{\text{Uniform}(-0.5, 0.5)}$ \\
CRT & None & 1 & 100 & $1 \times 10^{-2}$  & None \\
Reweight & None & 1 & 100 & $1 \times 10^{-2}$  & None \\
JTT & Random & 25 & 100 & $1 \times 10^{-5}$  & $\text{step\_frac} \sim \text{Uniform}(0.2, 0.8)$, $\lambda \sim 10^{\text{Uniform}(0, 2.5)}$, $\text{wd} = 0.1$ (Stage 1 \& 2) \\
LISA & Random & 25 & 100 & $1 \times 10^{-2}$  & $\alpha \sim 10^{\text{Uniform}(-1, 1)}$, $p_{\text{sel}} \sim \text{Uniform}(0, 1)$ \\
IRM & Random & 25 & 100 & $1 \times 10^{-2}$  & $\lambda \sim 10^{\text{Uniform}(-1, 5)}$, $\text{iters} \sim 10^{\text{Uniform}(0, 4)}$ \\
CVaRDRO & Random & 25 & 100 & $1 \times 10^{-2}$  & $\alpha \sim 10^{\text{Uniform}(-2, 0)}$ \\
Mixup & Random & 25 & 100 & $1 \times 10^{-2}$  & $\alpha \sim 10^{\text{Uniform}(0, 4)}$ \\
Group DRO & None & 1 & 100 & $1 \times 10^{-2}$  & $\eta = 1 \times 10^{-2}$  \\
ERM & None & 1 & 100 & $1 \times 10^{-2}$  & None \\
\bottomrule
\end{tabular}%
}
\begin{flushleft}
\footnotesize
\textbf{Notes:}
\begin{itemize}
\itemsep0em
    \item \methodname{}: $\gamma$ = reweighting strength, $k$ = spectral rank
    \item AFR: $\gamma$ = reweighting strength, $\lambda$ = regularization coefficient
    \item DFR: dfr\_reg = L1 regularization on classifier weights
    \item JTT: step\_frac = fraction of epochs for Stage 1, $\lambda$ = upweighting factor for misclassified examples
    \item LISA: $\alpha$ = mixup Beta parameter, $p_{\text{sel}}$ = selection probability
    \item IRM: $\lambda$ = penalty weight, iters = penalty annealing iterations
    \item CVaRDRO: $\alpha$ = uncertainty set size (CVaR parameter)
    \item Mixup: $\alpha$ = mixup Beta parameter
    \item Group DRO: $\eta$ = group weight update rate
\end{itemize}
\end{flushleft}
\end{table}

\subsection{Runtime Analysis}
\label{app:runtime}
We provide Table~\ref{tab:runtime} to accompany Figure~\ref{fig:runtime_params} to clearly show the differences between baselines. Timings are obtained by are obtained by running on a single RTX8000 (48 GB) NVIDIA GPU. As with prior work \citep{qiu2023simplefastgrouprobustness}, we included the time to pre-compute and cache the embeddings, which in practice can be ignored by amortizing across a hyperparameter sweep. 
\begin{table}[h]
\centering
\caption{End-to-end Runtime Comparison (minutes)}
\label{tab:runtime}
\resizebox{0.65\textwidth}{!}{%
\begin{tabular}{lcccccc}
\toprule
\textbf{Dataset} & \textbf{ERM} & \textbf{\methodname{}} & \textbf{AFR} & \textbf{JTT} & \textbf{DPE} & \textbf{GIC} \\
\midrule
\textsc{Waterbirds} & 25.05 & 26.53 & 26.58 & 76.08 & 40.12 & 120.20 \\
\textsc{CelebA} & 158.36 & 167.03 & 169.27 & 479.71 & 254.00 & 767.24 \\
\textsc{CivilComments} & 145.95 & 164.89 & 166.68 & 441.85 & 257.51 & 757.72 \\
\textsc{MultiNLI} & 61.23 & 72.99 & 73.12 & 185.52 & 116.33 & 335.29 \\
\bottomrule
\end{tabular}
}
\end{table}

\subsection{Compute Resources}
All experiments were conducted on four NVIDIA RTX8000 (48\,GB) GPUs,
with individual jobs assigned to a single GPU and experiments run in
parallel across devices.
End-to-end runtimes for a single run (one seed, one hyperparameter
configuration) range from approximately 25 minutes (Waterbirds, ERM)
to 480 minutes (CelebA, JTT); \methodname{} adds only a small overhead over ERM,
ranging from 27 minutes on Waterbirds to 167 minutes on CelebA
(see Table~\ref{tab:runtime}).
Stage-1 ERM training dominates the compute budget; Stage-2 head
retraining is negligible by comparison, as embeddings are cached and
only the low-rank classifier update is optimized.
Settings 1 and 3 share the same underlying training runs, differing
only in the model-selection criterion applied at evaluation time.
Many baseline results for the five main vision and language datasets
are taken directly from the original papers or SubpopBench~\cite{yang2023changehardcloserlook}
and were not recomputed.
Counting the reproduced baselines, \methodname{} and AFR hyperparameter sweeps
across all seven datasets, the novel Setting~2 (partial group information)
experiments, and the larger-model ablations
(ResNet-101/152 and BERT-Large; Appendix~\ref{app:large_models}),
we estimate a total of approximately \textbf{500-800 GPU-hours}
(${\approx}125$-$200$ wall-clock hours across four GPUs).

\section{Ablation Studies}
\label{app:ablations}
We organize the ablations around two questions. First, we examine the mechanism underlying \methodname{}: whether misclassification-aware weighting identifies failure-relevant structure and whether adaptation in the resulting basis is important. Second, we evaluate sensitivity to the main method choices, including correction rank, reweighting strength, and the split between ERM training and adaptation data.

\subsection{Method Mechanism Ablations}
\label{app:real_data_subspace}
We evaluate the error-informed subspace with three complementary post-hoc analyses. First, we test whether the weighting scheme emphasizes groups on which the ERM model performs poorly. Second, we measure whether the resulting subspace aligns with subgroup-relevant and oracle correction directions. Third, we replace the learned basis with matched PCA and random alternatives while holding the adaptation procedure fixed. Held-out subgroup annotations are used only for these analyses and are never used to construct MARLA's weights, estimate its subspace, or train the correction.

\paragraph{Do underperforming groups receive greater MARLA weight?}
For each subgroup $g$, define its prevalence-adjusted weight amplification as
\[
\operatorname{amp}_g=
\frac{\sum_{i\in g}\mu_i}{n_g/n},
\qquad
\mu_i=\frac{\exp(\gamma(1-p_{y_i}(x_i)))}{\sum_j\exp(\gamma(1-p_{y_j}(x_j)))}.
\]
Thus, $\operatorname{amp}_g>1$ means that group $g$ receives more MARLA weight than expected from its prevalence. Let $g^\star=\arg\max_g \operatorname{Err}_{\mathrm{ERM}}(g)$ denote the highest-error ERM group. We also report its amplification relative to the mean amplification of the remaining groups.

\begin{table}[!h]
\centering
\caption{Post-hoc relationship between ERM subgroup error and MARLA's weight allocation. Quantities are computed per seed and reported as mean $\pm$ standard deviation over three seeds.}
\label{tab:app_error_weight_relationship}
\small
\setlength{\tabcolsep}{5pt}
\renewcommand{\arraystretch}{1.15}
\resizebox{\textwidth}{!}{
\begin{tabular}{lcccc}
\toprule
Dataset & \shortstack{Error-weight\\correlation} & \shortstack{Worst-group\\amplification} & \shortstack{Amplification relative\\to other groups} & \shortstack{ERM error on\\worst group} \\
\midrule
GRACE & $0.96 \pm 0.02$ & $1.22 \pm 0.07\times$ & $1.68 \pm 0.20\times$ & $6.7\% \pm 0.7\%$ \\
MIMIC-IV & $0.93 \pm 0.04$ & $4.1 \pm 0.6\times$ & $5.7 \pm 0.9\times$ & $27.9\% \pm 1.0\%$ \\
Waterbirds & $0.99 \pm 0.01$ & $12.7 \pm 3.9\times$ & $18.4 \pm 5.1\times$ & $30.3\% \pm 6.7\%$ \\
CivilComments & $0.94 \pm 0.04$ & $2.3 \pm 0.4\times$ & $3.2 \pm 0.6\times$ & $38.1\% \pm 8.7\%$ \\
\bottomrule
\end{tabular}
}
\end{table}
Across all four datasets, higher ERM subgroup error is strongly associated with larger prevalence-adjusted MARLA weight. The highest-error ERM group also receives greater weight than the mean of the remaining groups. These results support the first step of MARLA's mechanism: the weighting scheme emphasizes examples from groups on which the base model fails more often. We next test whether this emphasis produces a coherent geometric signal rather than merely concentrating on difficult examples.

\paragraph{Does the selected subspace align with subgroup-relevant directions?}
Let $V_k^{\mathrm{MARLA}}$ denote the top-$k$ eigenvectors of MARLA's score-weighted covariance. We compare it to an unweighted PCA subspace $V_k^{\mathrm{PCA}}$ and a random orthonormal $k$-dimensional subspace using two post-hoc references. The known-attribute reference $U_{\mathrm{attr}}$ is formed from the leading eigenvectors of between-group scatter,
\[
S_B=\sum_g n_g(\bar\phi_g-\bar\phi)(\bar\phi_g-\bar\phi)^\top.
\]
The error-fixing reference $U_{\mathrm{fix}}$ is obtained by fitting an oracle group-balanced linear head $W^\star$ on the frozen features and taking the right singular vectors of $\Delta W=W^\star-W_0$, where $W_0$ is the ERM head. We measure mean squared canonical correlation (MSCC), the mean squared cosine of the principal angles between two subspaces.

\begin{table}[!h]
\centering
\caption{MSCC between candidate rank-$k$ subspaces and post-hoc known-attribute or oracle error-fixing reference subspaces.}
\label{tab:app_attr_fix_scores}
\resizebox{\textwidth}{!}{
\begin{tabular}{lcccccc}
\toprule
Dataset & MARLA-attr & PCA-attr & Random-attr & MARLA-fix & PCA-fix & Random-fix \\
\midrule
Waterbirds & 0.44 & 0.21 & 0.01 & 0.52 & 0.42 & 0.02 \\
CivilComments & 0.55 & 0.68 & 0.02 & 0.75 & 0.22 & 0.02 \\
GRACE & 0.53 & 0.36 & 0.04 & 0.33 & 0.29 & 0.02 \\
MIMIC-IV & 0.52 & 0.31 & 0.03 & 0.64 & 0.05 & 0.01 \\
\bottomrule
\end{tabular}}
\end{table}
MARLA exceeds the random-subspace baseline throughout. Relative to unweighted PCA, MARLA has higher alignment with the oracle error-fixing subspace on all four datasets and higher known-attribute alignment on three of four datasets. These comparisons indicate that the selected directions contain subgroup- and correction-relevant structure. We next test whether the specific basis matters for downstream adaptation.

\paragraph{Matched low-rank basis comparison.}
We next hold the frozen representation and classifier, correction rank, optimizer, and training budget fixed and change only the $k$-dimensional basis used for the correction. Hyperparameters are selected using validation worst-class accuracy without subgroup labels.

\begin{table}[!h]
\centering
\caption{Matched comparison of MARLA's error-weighted basis with unweighted PCA and random orthonormal rank-$k$ bases. Each entry is test WGA / average accuracy.}
\label{tab:app_basis_comparison}
\resizebox{\textwidth}{!}{
\begin{tabular}{lccccc}
\toprule
Dataset & MARLA WGA / Avg & PCA WGA / Avg & Random WGA / Avg & MARLA-PCA WGA & MARLA-Random WGA \\
\midrule
Waterbirds & $\mathbf{90.1 \pm 0.1\;/\;93.7}$ & $76.6 \pm 6.9\;/\;92.4$ & $64.7 \pm 8.7\;/\;90.1$ & $\mathbf{+13.5}$ & $\mathbf{+25.4}$ \\
CivilComments & $\mathbf{71.6 \pm 0.7\;/\;91.4}$ & $54.2 \pm 9.5\;/\;92.7$ & $60.1 \pm 0.4\;/\;92.5$ & $\mathbf{+17.4}$ & $\mathbf{+11.5}$ \\
GRACE & $\mathbf{94.2 \pm 1.6\;/\;95.8}$ & $91.2 \pm 0.8\;/\;95.1$ & $91.1 \pm 1.0\;/\;95.1$ & $\mathbf{+3.0}$ & $\mathbf{+3.1}$ \\
MIMIC-IV & $\mathbf{78.2 \pm 2.0\;/\;81.0}$ & $76.9 \pm 0.5\;/\;80.0$ & $77.0 \pm 0.5\;/\;80.1$ & $\mathbf{+1.3}$ & $\mathbf{+1.2}$ \\
\bottomrule
\end{tabular}}
\end{table}
The error-weighted basis outperforms both matched alternatives in WGA on all four datasets, indicating that low rank alone does not explain the gains. Together, these real-data analyses complement the controlled low-rank mechanism study in Appendix~\ref{app:synthetic_lora}.

\subsection{Hyperparameter and Data-Split Ablations}
\label{app:parameter_ablation}
We next evaluate sensitivity to the main tuning choices. We vary the correction rank $k$ and reweighting strength $\gamma$, then vary the fraction of data allocated to ERM training and MARLA adaptation.

\paragraph{Rank $k$ and reweighting strength $\gamma$.}
\begin{wrapfigure}{r}{0.39\linewidth}
\vspace{-15pt}
\includegraphics[width=\linewidth]{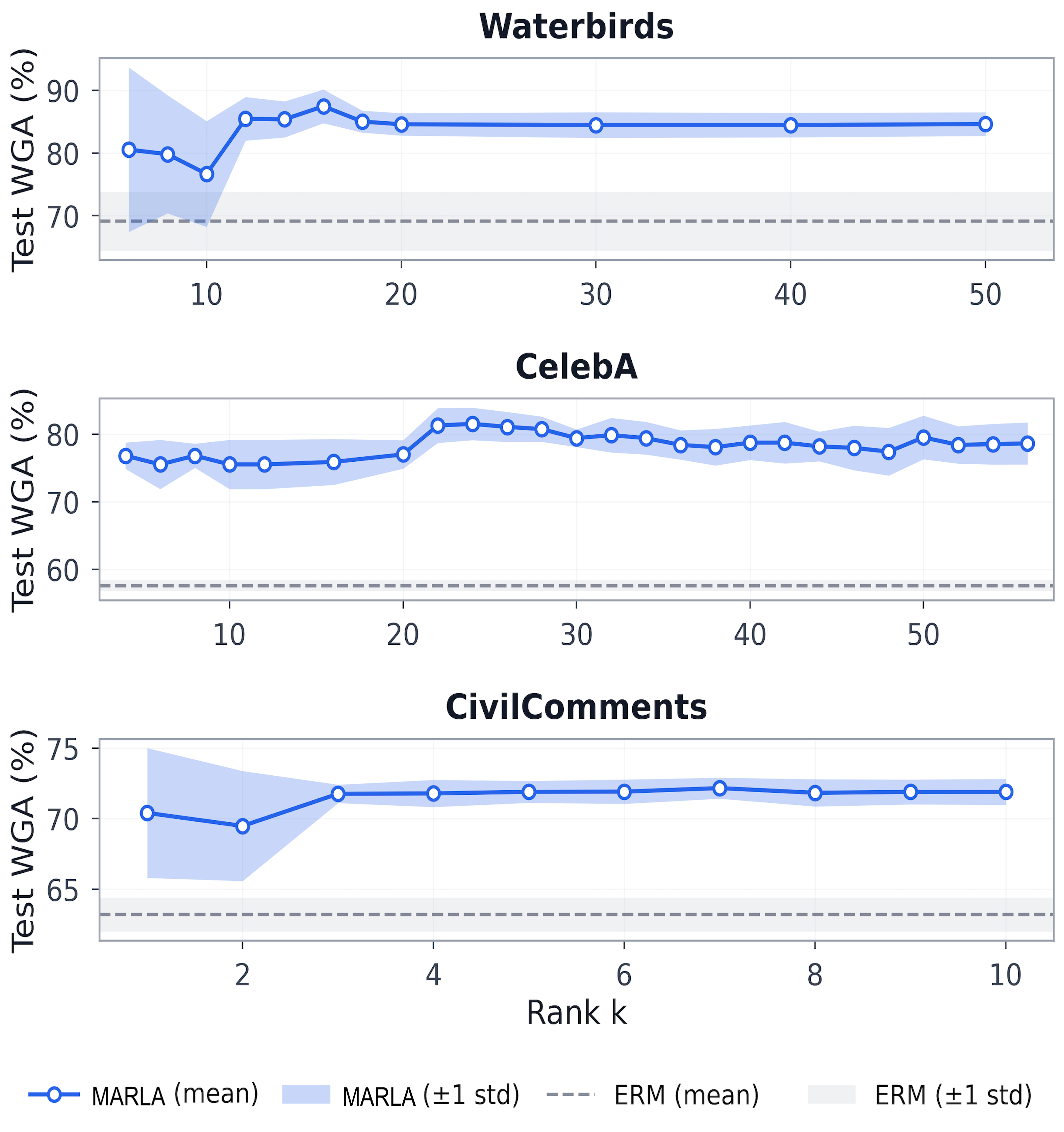}
    \caption{Sensitivity of \methodname{} to correction rank $k$. Test WGA remains stable across much of the tuning range. Curves show the mean and standard deviation over three seeds.}
    \label{fig:rank_robustness}
\end{wrapfigure} 
We first study sensitivity to the correction rank $k$. Across the tuning range, \methodname{} maintains strong WGA on all datasets. \textsc{Waterbirds} and \textsc{CivilComments} show greater variability at lower ranks, suggesting that very small subspaces may omit relevant variation (Figure~\ref{fig:rank_robustness}). Once performance stabilizes, variation is below $1\%$ across ranks. 

Beyond rank robustness, \methodname{} is also robust to the choice of the reweighting strength hyperparameter $\gamma$. Different datasets exhibit different optimal $\gamma$ values: image datasets typically favor $\gamma \approx 1$-$8$, whereas text datasets require larger values around $\gamma \approx 500$-$5000$. Figure~\ref{fig:gamma_robustness} shows that performance remains stable across a broad neighborhood of these optima. For each $\gamma$, we select the best rank based on validation worst-group accuracy (WGA) and report the mean $\pm$ one standard deviation across seeds.

Across all datasets, \methodname{} varies by less than 2-3 percentage points over the tested $\gamma$ ranges, despite optimal values differing by orders of magnitude. The exponential weighting scheme, $\mu_i = \exp\!\big(\gamma (1 - p_i)\big)$, preserves meaningful relative weight differences over a broad range of $\gamma$ values. Together with the relationship between $\gamma$ and dataset-specific confidence distributions, these results support coarse, data-guided initialization followed by tuning within a relatively stable range.

Together, these analyses support data-guided selection of both $k$ and $\gamma$ over the ranges considered here.

\begin{figure}[!h]
    \centering
    \includegraphics[width=0.85\linewidth]{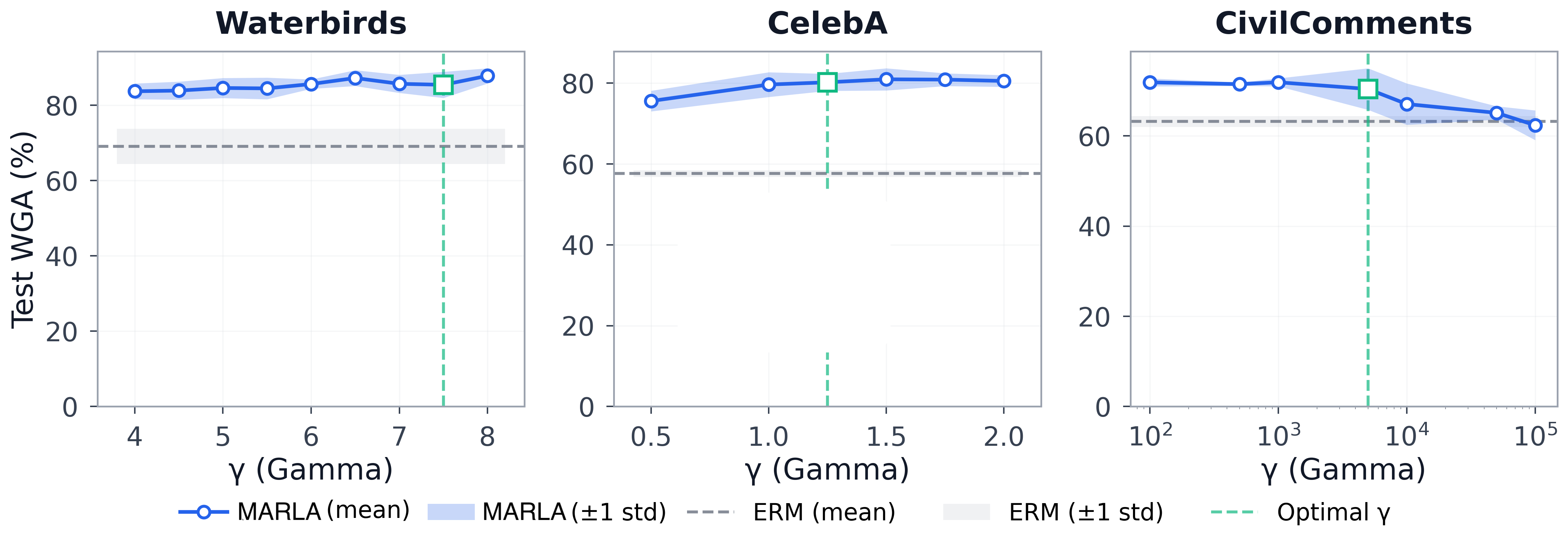}
    \caption{\textbf{\methodname{} Robustness to Gamma ($\gamma$) Selection.} Test worst-group accuracy (WGA) as a function of gamma for \textsc{Waterbirds, CelebA, and CivilComments}. For each gamma value, we select the best rank (by validation WGA) and report mean $\pm$ 1 standard deviation across seeds. The green dashed line and diamond marker indicate the empirically-observed optimal gamma for each dataset. The gray dashed line and shaded region show ERM performance (mean $\pm$ 1 std). \methodname{} demonstrates robustness to gamma selection, with performance remaining stable across a reasonable range around the optimal values.}
    \label{fig:gamma_robustness}
\end{figure}

Like $k$, we also observe that optimal $\gamma$ values vary based on the underlying dataset distribution. The hyperparameter $\gamma$ in \methodname{} controls the strength of exponential reweighting based on the misclassification score. Specifically, \methodname{} assigns each training example a weight $\mu_i = \exp\!\big(\gamma (1 - p_i)\big)$, where $p_i$ is the probability assigned to the true label and $(1-p_i)$ is the misclassification score. The parameter $\gamma$ acts as a scaling factor: larger values amplify small differences in this score into larger weight disparities, while smaller values yield more uniform weighting. As a result, the optimal choice of $\gamma$ is tied to the dataset-specific distribution of true-label probabilities.

\begin{figure}[!h]
    \centering
    \includegraphics[width=0.65\linewidth]{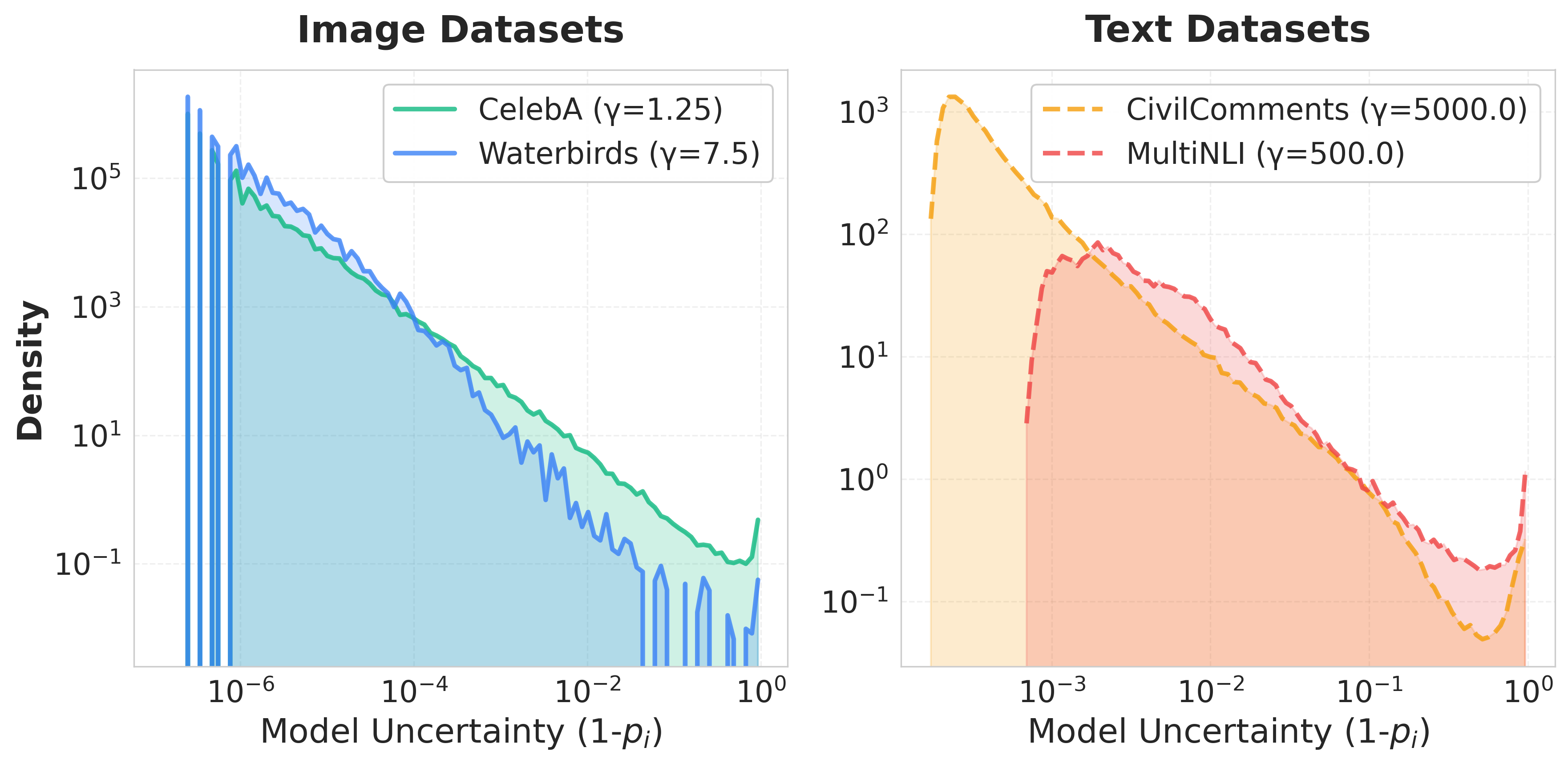}
    \caption{Differences in soft misclassification-score distributions by dataset. Image datasets have many examples with very low scores and a distinct high-score tail, whereas text datasets exhibit more compressed score distributions.}
    \label{fig:confidence_dist}
\end{figure}

Soft misclassification-score distributions vary substantially across datasets (see Figure~\ref{fig:confidence_dist}). Most examples in image datasets such as \textsc{CelebA} and \textsc{Waterbirds} have scores near zero, alongside a distinct high-score tail (95th percentile $\approx 0.4$-$0.8$). Relatively small values of $\gamma$ ($1$-$8$) are sufficient, as the separation between the bulk and the tail allows even modest scaling to produce meaningful weight differences. In contrast, text datasets such as \textsc{CivilComments} and \textsc{MultiNLI} display continuous, compressed score distributions: scores are small but nonzero for most examples (median $\approx 0.002$-$0.019$), with values concentrated in a narrow range (interquartile range spanning approximately $0.0005$-$0.03$). Much larger values of $\gamma$ ($500$-$5000$) are required to amplify small differences in the score.

\paragraph{ERM/adaptation data split.}
\label{ref:splitting_ratio}
We assess how robust our method is to the data splitting ratio used for \(\mathcal{D}_{\mathrm{ERM}}\) for base training and
\(\mathcal{D}_{\mathrm{MARLA}}\) for adaptation. We find that across various ratios, \methodname{}'s performance remains strong in the setting with no information of group relevance (Table~\ref{tab:ratio_setting1}) and the setting with complete information of group relevance (Table~\ref{tab:ratio_setting2}). 

\begin{table}[!h]
\centering
\caption{Performance comparison under varying $D_{\text{ERM}}/D_{\text{\methodname{}}}$ ratios in the setting with no group information.}
\begin{tabular}{lcc|cc}
\toprule
 & \multicolumn{2}{c}{\textsc{Waterbirds}} & \multicolumn{2}{c}{\textsc{CivilComments}} \\
\textbf{$\mathcal{D}_{\text{ERM}}/\mathcal{D}_{\text{\methodname{}}}$} & \textbf{WGA} & \textbf{Avg Acc} & \textbf{WGA} & \textbf{Avg Acc} \\
\midrule
90-10 & $88.8 \pm0.3$ & $92.3 \pm0.4$ & $70.2 \pm0.3$ & $91.0 \pm0.9$ \\
80-20 & $90.1 \pm0.1$ & $93.7 \pm0.7$ & $71.6 \pm0.7$ & $91.4 \pm0.4$ \\
70-30 & $88.8 \pm0.3$ & $92.3 \pm0.4$ & $70.9 \pm0.1$ & $91.1 \pm0.8$ \\
60-40 & $88.8 \pm0.3$ & $92.3 \pm0.4$ & $71.0 \pm0.8$ & $91.1 \pm0.8$ \\
\midrule
ERM & $69.1 \pm4.7$ & $84.1 \pm1.7$ & $63.2 \pm1.2$ & $85.4 \pm0.2$ \\
\bottomrule
\end{tabular}
\label{tab:ratio_setting1}
\end{table}

\begin{table}[!h]
\centering
\caption{Performance comparison under varying $D_{\text{ERM}}/D_{\text{\methodname{}}}$ ratios in the setting with complete group relevance information.}
\begin{tabular}{lcc|cc}
\toprule
 & \multicolumn{2}{c}{\textsc{Waterbirds}} & \multicolumn{2}{c}{\textsc{CivilComments}} \\
\textbf{$\mathcal{D}_{\text{ERM}}/\mathcal{D}_{\text{\methodname{}}}$} & \textbf{WGA} & \textbf{Avg Acc} & \textbf{WGA} & \textbf{Avg Acc} \\
\midrule
90-10 & $89.4 \pm0.3$ & $92.8 \pm0.9$ & $72.2 \pm0.7$ & $91.4 \pm0.4$ \\
80-20 & $90.1 \pm0.1$ & $92.8 \pm0.9$ & $72.9 \pm0.2$ & $90.9 \pm0.8$ \\
70-30 & $89.4 \pm0.3$ & $92.3 \pm0.4$ & $72.0 \pm0.3$ & $91.3 \pm0.5$ \\
60-40 & $89.2 \pm0.3$ & $92.5 \pm0.6$ & $71.7 \pm0.2$ & $91.3 \pm0.4$ \\
\midrule
ERM & $69.1 \pm4.7$ & $84.1 \pm1.7$ & $63.2 \pm1.2$ & $85.4 \pm0.2$ \\
\bottomrule
\end{tabular}
\label{tab:ratio_setting2}
\end{table}

\section{Synthetic Experiment for Analyzing \methodname{}'s Low-Rank Mechanism}
\label{app:synthetic_lora}

\textbf{Setup.}
We construct a controlled binary classification problem to isolate two claims underlying \methodname{}:
(i) the misclassification score can identify latent minority-relevant examples without oracle
group labels during adaptation, and (ii) the correction rank must match the intrinsic dimensionality
of the subspace needed for minority-group recovery. Both claims are made falsifiable by design:
an unweighted PCA baseline tests whether score-based weighting is necessary, and varying the number
of minority subgroups tests rank sensitivity.

Each dataset consists of $K+1$ groups: one majority group and $K$ minority subgroups.
Feature vectors $\mathbf{x}\in\mathbb{R}^{1+K+D_{\mathrm{noise}}}$ contain three disjoint parts:
\begin{itemize}
  \item \textbf{Content dimension} (index~0): carries the class signal for the majority group.
  Majority labels are generated as $y\sim\mathrm{Bernoulli}(0.5)$, then
  \[
  x_0=(2y-1)\delta+\varepsilon,
  \qquad
  \delta=1.0,
  \qquad
  \varepsilon\sim\mathcal{N}(0,0.09).
  \]
  This makes the majority labels nearly deterministic from $x_0$. Minority examples have
  $x_0\approx 0$ sampled as $0.1\varepsilon$, so the majority content feature carries
  negligible signal for minority labels.

  \item \textbf{Latent dimensions} (indices $1,\ldots,K$): each minority subgroup $k$ has
  its own latent feature $x_{1+k}\sim\mathcal{N}(0,1)$, with labels generated as
  \[
  y\sim\mathrm{Bernoulli}(\sigma(2.5\,x_{1+k})).
  \]
  Majority examples have near-zero entries in all latent dimensions, sampled as
  $0.1\varepsilon$.

  \item \textbf{Noise dimensions} (indices $1+K,\ldots,d$): $D_{\mathrm{noise}}$
  uninformative Gaussian dimensions with standard deviation $0.5$, present only in majority
  examples and set to zero for minority examples. These dimensions inflate the unweighted
  covariance with majority-dominated variation, making PCA a meaningful foil.
\end{itemize}

\textbf{Stage~1: ERM.}
A logistic regression model is fit using only the majority training data
($N_{\mathrm{maj}}=1{,}000$ samples) via L-BFGS with $\ell_2$ regularization.
This construction intentionally produces a base model that relies on the majority content
feature $x_0$ and assigns near-zero weight to the latent and noise dimensions, since those
directions carry negligible signal in the majority data. As a result, the ERM model achieves
high majority accuracy but does not use each minority subgroup's latent predictor, producing
near-chance accuracy on all minority test groups.

\textbf{Why the misclassification score identifies the latent minority set.}
For minority examples, the majority content feature is near zero and the learned weights on
the minority latent dimensions are near zero. The ERM logit is therefore driven mostly by
small content noise and the bias term, yielding true-label probabilities near $0.5$. In
contrast, majority examples are confidently classified because their labels are nearly
deterministic from $x_0$. Thus the misclassification score
\[
u_i = 1-\hat{p}_{y_i},
\qquad
\hat{p}_{y_i}=\Pr_{m_{\hat{\theta}}}(Y=y_i\mid x_i),
\]
is small for most majority examples and larger for minority examples. The resulting weights
therefore concentrate mass on the minority examples relative to the majority, without using
group membership.

This experiment does not require the score weights to distinguish finely among examples
within the same minority subgroup. The key mechanism is instead that, on the full mixed
adaptation set, score-based weighting suppresses high-confidence majority directions and
amplifies the latent minority directions. Once the covariance is dominated by low true-label
probability examples, its leading eigendirections recover the low-dimensional subspace needed
for minority-group recovery.

\textbf{No oracle group labels during adaptation.}
No group membership information is used in Stage~2. The score-weighted covariance is
computed on the full adaptation set containing all majority and minority examples, and the
additive correction is fitted on the same full set. The algorithm observes only inputs, labels,
and frozen-model predictions; it never observes the group index $g_i$ during covariance
estimation or correction learning. Group labels are used only for constructing the synthetic
data-generating process and for reporting group-wise evaluation metrics.

\textbf{Stage~2: Score-weighted covariance and \methodname{}.}
On the full adaptation set, we compute normalized score weights
\[
\mu_i
=
\frac{\exp(\gamma(1-\hat{p}_{y_i}))}
{\sum_j \exp(\gamma(1-\hat{p}_{y_j}))},
\qquad
\gamma=8.0,
\]
matching the bounded soft misclassification-score weighting used in the main \methodname{}
implementation (Appendix~\ref{app:algorithm_implementation}). Let
\[
\bar{\mathbf{x}}=\sum_i \mu_i\mathbf{x}_i.
\]
We compute the centered score-weighted covariance
\begin{equation}
  \mathbf{C}
  =
  \sum_i \mu_i
  (\mathbf{x}_i-\bar{\mathbf{x}})
  (\mathbf{x}_i-\bar{\mathbf{x}})^\top.
  \label{eq:synthetic_score_covariance}
\end{equation}
\methodname{} then fits a rank-$k$ additive correction
\[
\Delta\mathbf{w}=\mathbf{U}_k\mathbf{a},
\]
where $\mathbf{U}_k$ contains the top-$k$ eigenvectors of $\mathbf{C}$. The coefficient
vector $\mathbf{a}$ is learned by minimizing score-weighted cross-entropy over the full
adaptation set, with mild $\ell_2$ regularization to prevent instability when the score
weights are highly concentrated.

\textbf{Baseline: Unweighted PCA.}
As a control, we replace $\mathbf{C}$ with the centered unweighted covariance of the full
adaptation set and keep the same rank-$k$ adaptation procedure. Because majority examples
outnumber minority examples and carry large variance in both the content and noise dimensions,
the leading unweighted PCA directions align with majority-dominated variation rather than the
latent minority-relevant directions. This baseline directly tests whether score-based weighting
is the operative ingredient in selecting the subspace used for adaptation.

\textbf{Three cases.}
We vary $K$ to examine how the required correction rank scales with the dimensionality of the
minority-relevant latent structure. Each case uses $N_{\mathrm{min}}=150$ examples per minority
subgroup and $N_{\mathrm{test}}=500$ per group. Results are averaged over 40 random seeds with
$\pm 1$ standard deviation bands.

\begin{itemize}
  \item \textbf{Case~1} ($K=1$, $D_{\mathrm{noise}}=15$, $d=17$). A single minority
  subgroup requires one latent direction. Since the score-weighted covariance concentrates
  on one latent feature dimension, it has an effective rank-one structure with a sharp
  spectral elbow at $k=1$ (Figure~\ref{fig:synthetic_score_covariance}a). WGA jumps from
  approximately $0.50$ for ERM to approximately $0.80$ at rank~1 and then plateaus, while
  the PCA-directed correction yields no improvement at any rank
  (Figure~\ref{fig:synthetic_score_covariance}b).

  \item \textbf{Case~2} ($K=4$, $D_{\mathrm{noise}}=10$, $d=15$). Four minority
  subgroups each require a distinct latent dimension. The score mass is spread across
  four independent latent directions, giving the weighted covariance an effective rank-four
  structure. The spectrum decays across $k=1$-$4$ before dropping sharply; WGA improves
  monotonically up to rank~4 and then plateaus. PCA-directed correction again yields
  little to no improvement.

  \item \textbf{Case~3} ($K=8$, $D_{\mathrm{noise}}=5$, $d=14$). Eight minority
  subgroups each require a distinct latent direction. The score-weighted spectrum decays
  across the first eight directions before flattening, and WGA continues to improve up to
  rank~8, reflecting the higher intrinsic dimensionality of the minority-relevant subspace.
\end{itemize}

\textbf{Interpretation.}
The three cases show that \methodname{}'s synthetic recovery depends on both score-based
weighting and sufficient correction rank. Unweighted PCA uses the same rank budget and the
same adaptation procedure but selects majority-dominated directions, failing to recover the
minority-relevant subspace. Conversely, score-based weighting identifies the relevant latent
directions, but under-correcting with rank $k<K$ leaves some of these directions unaddressed.
The spectral elbow of the score-weighted covariance therefore provides a label-free diagnostic
for the intrinsic dimensionality of the selected subspace. Taken together, these results support
the intended mechanism: low true-label probability identifies a low-dimensional subspace for
adaptation, and low-rank adaptation within that subspace recovers minority performance without
using group labels during adaptation.

\section{Theory: Sufficient Conditions for Failure-Associated Subspace Recovery}
\label{app:subspace_conditions}
This appendix gives population-level sufficient conditions under which the score-weighted covariance recovers a coherent failure-associated subspace. The result is not an unconditional identification guarantee. Instead, it makes explicit the conditions under which the mechanism studied empirically in Appendix~\ref{app:real_data_subspace} is expected to hold.

Let $P$ denote the adaptation-data distribution over examples $(X,Y)$, and let
\[
p_Y(X)=\operatorname{softmax}(f(X))_Y,
\qquad
a(X,Y)=\exp\!\left(\gamma(1-p_Y(X))\right).
\]
Equivalently, define a reweighted population distribution $Q$ by
\[
dQ(x,y)=\frac{a(x,y)}{\mathbb{E}_P[a(X,Y)]}\,dP(x,y).
\]
Let $F$ denote a systematic failure component and $R=F^c$ the remaining population. Define
\[
\pi=P(F),\qquad m_F=\mathbb{E}_P[a(X,Y)\mid F],\qquad m_R=\mathbb{E}_P[a(X,Y)\mid R].
\]
The mass assigned to $F$ after reweighting is
\[
\pi_w=Q(F)=\frac{\pi m_F}{\pi m_F+(1-\pi)m_R}.
\]
If $m_F>m_R$, then the denominator is smaller than $m_F$, and therefore
\[
\pi_w>\pi.
\]
Thus, MARLA does not require every low-confidence example to belong to the failure component; it requires only that examples in the component receive greater weight on average.

Let $Z=e_{\hat\psi}(X)\in\mathbb{R}^d$ be the frozen representation. Under $Q$, let $\mu_F,\Sigma_F$ and $\mu_R,\Sigma_R$ denote the conditional means and covariances. The population weighted covariance decomposes as
\[
\Sigma_w^\star
=\pi_w\Sigma_F+(1-\pi_w)\Sigma_R+\pi_w(1-\pi_w)(\mu_F-\mu_R)(\mu_F-\mu_R)^\top.
\]
Suppose the coherent failure signal lies in a $k$-dimensional subspace with orthonormal basis $V_\star\in\mathbb{R}^{d\times k}$ and, in an idealized signal model,
\[
\Sigma_F=\sigma^2I+V_\star B V_\star^\top,\qquad
\Sigma_R=\sigma^2I,\qquad
\mu_F-\mu_R=V_\star b,
\]
with $B\succeq0$. Then
\[
\Sigma_w^\star=\sigma^2I+V_\star M V_\star^\top,
\qquad
M=\pi_wB+\pi_w(1-\pi_w)bb^\top.
\]
For real data, write deviations from this ideal model and finite-sample estimation error together as $E$:
\[
\widetilde\Sigma_w=\sigma^2I+V_\star M V_\star^\top+E.
\]
Let $\widehat V_k$ be the top-$k$ eigenspace of $\widetilde\Sigma_w$ and assume $\Delta=\lambda_k(M)>0$. If $\|E\|_2<\Delta$, a standard eigenspace perturbation bound gives
\[
\|\sin\Theta(\widehat V_k,V_\star)\|_2
\leq \frac{\|E\|_2}{\Delta-\|E\|_2}.
\]
In particular, if $\|E\|_2\leq\Delta/2$, then
\[
\|\sin\Theta(\widehat V_k,V_\star)\|_2\leq\frac{2\|E\|_2}{\Delta}.
\]
Thus, the leading weighted-covariance eigenspace approaches the failure-associated subspace when two conditions hold: (i) the failure component receives greater weight on average, and (ii) its coherent low-dimensional signal dominates nuisance variation and sampling error. Systematic label corruption, structured miscalibration, or an unrepresentative adaptation set can violate either condition.

\section{Additional Results}
\label{app:additional_results}
This section collects evaluations that extend the main results but are not direct ablations of MARLA's mechanism or hyperparameters. We report broader clinical metrics, sensitivity to probability calibration, additional baselines with complete group information, and experiments with larger vision and language backbones.

\subsection{Additional Clinical Evaluation Metrics}
\label{app:clinical_metrics}
Because average accuracy can obscure performance under class imbalance, we additionally evaluate the no-group-information setting on GRACE and MIMIC-IV using AUROC, macro AUPRC, balanced accuracy, sensitivity, specificity, demographic WGA, and label-conditioned demographic WGA.

\begin{table}[!h]
\centering
\caption{Additional no-group-information metrics on GRACE.}
\label{tab:app_grace_metrics}
\setlength{\tabcolsep}{3pt}
\resizebox{\textwidth}{!}{
\begin{tabular}{lccccccc}
\toprule
Method & AUROC & Macro AUPRC & Balanced acc & Sensitivity & Specificity & Age$\times$sex WGA & Label-cond. demo WGA \\
\midrule
ERM & $49.73 \pm 11.26$ & $51.75 \pm 2.28$ & $51.06 \pm 1.50$ & $2.27 \pm 3.21$ & $99.85 \pm 0.21$ & $92.52 \pm 0.09$ & $1.30 \pm 1.84$ \\
JTT & $49.73 \pm 11.26$ & $51.75 \pm 2.28$ & $51.06 \pm 1.50$ & $2.27 \pm 3.21$ & $99.85 \pm 0.21$ & $92.52 \pm 0.09$ & $1.30 \pm 1.84$ \\
IRM & $49.73 \pm 11.26$ & $51.75 \pm 2.28$ & $51.06 \pm 1.50$ & $2.27 \pm 3.21$ & $99.85 \pm 0.21$ & $92.52 \pm 0.09$ & $1.30 \pm 1.84$ \\
CVaRDRO & $49.77 \pm 11.26$ & $51.76 \pm 2.28$ & $51.06 \pm 1.50$ & $2.27 \pm 3.21$ & $99.85 \pm 0.21$ & $92.52 \pm 0.09$ & $1.30 \pm 1.84$ \\
LfF & $49.76 \pm 11.25$ & $51.76 \pm 2.28$ & $51.06 \pm 1.49$ & $2.27 \pm 3.21$ & $99.84 \pm 0.23$ & $92.52 \pm 0.09$ & $1.30 \pm 1.84$ \\
Mixup & $49.73 \pm 11.26$ & $51.75 \pm 2.28$ & $51.06 \pm 1.50$ & $2.27 \pm 3.21$ & $99.85 \pm 0.21$ & $92.52 \pm 0.09$ & $1.30 \pm 1.84$ \\
ReSample & $82.76 \pm 1.59$ & $59.51 \pm 1.42$ & $75.12 \pm 0.49$ & $68.94 \pm 2.29$ & $81.31 \pm 3.22$ & $71.83 \pm 3.10$ & $62.35 \pm 3.33$ \\
CBLoss & $49.79 \pm 11.24$ & $51.77 \pm 2.28$ & $51.14 \pm 1.61$ & $2.46 \pm 3.48$ & $99.81 \pm 0.26$ & $92.52 \pm 0.09$ & $1.73 \pm 2.45$ \\
Focal & $49.74 \pm 11.25$ & $51.75 \pm 2.28$ & $51.06 \pm 1.50$ & $2.27 \pm 3.21$ & $99.85 \pm 0.21$ & $92.52 \pm 0.09$ & $1.30 \pm 1.84$ \\
BSoftmax & $75.07 \pm 5.40$ & $56.77 \pm 1.26$ & $67.94 \pm 5.81$ & $65.91 \pm 2.82$ & $69.98 \pm 9.78$ & $61.05 \pm 5.35$ & $46.22 \pm 8.80$ \\
AFR & $56.24 \pm 11.60$ & $52.72 \pm 2.28$ & $50.51 \pm 4.76$ & $32.95 \pm 17.31$ & $68.07 \pm 22.80$ & $61.03 \pm 23.21$ & $19.92 \pm 10.94$ \\
\textbf{MARLA (Ours)} & $\mathbf{85.66 \pm 1.03}$ & $\mathbf{62.01 \pm 1.32}$ & $\mathbf{77.12 \pm 0.89}$ & $\mathbf{78.14 \pm 2.82}$ & $91.31 \pm 1.78$ & $\mathbf{81.56 \pm 2.19}$ & $\mathbf{72.15 \pm 2.20}$ \\
\bottomrule
\end{tabular}}
\end{table}

\begin{table}[!h]
\centering
\caption{Additional no-group-information metrics on MIMIC-IV.}
\label{tab:app_mimic_metrics}
\setlength{\tabcolsep}{3pt}
\resizebox{\textwidth}{!}{
\begin{tabular}{lccccccc}
\toprule
Method & AUROC & Macro AUPRC & Balanced acc & Sensitivity & Specificity & Race$\times$gender WGA & Label-cond. demo WGA \\
\midrule
ERM & $63.38 \pm 0.66$ & $58.88 \pm 0.51$ & $59.76 \pm 0.72$ & $60.59 \pm 1.92$ & $58.93 \pm 3.13$ & $53.25 \pm 1.52$ & $36.15 \pm 3.31$ \\
JTT & $64.18 \pm 0.14$ & $59.43 \pm 0.22$ & $60.17 \pm 0.42$ & $63.59 \pm 4.82$ & $56.75 \pm 5.66$ & $53.22 \pm 2.91$ & $37.44 \pm 1.81$ \\
IRM & $58.80 \pm 4.54$ & $55.90 \pm 2.73$ & $56.09 \pm 3.49$ & $62.37 \pm 1.78$ & $49.82 \pm 5.37$ & $42.05 \pm 0.49$ & $30.41 \pm 5.05$ \\
CVaRDRO & $54.54 \pm 2.73$ & $53.18 \pm 1.20$ & $53.19 \pm 2.33$ & $53.76 \pm 5.07$ & $52.62 \pm 3.34$ & $44.46 \pm 2.70$ & $37.00 \pm 1.90$ \\
LfF & $58.45 \pm 1.51$ & $55.67 \pm 1.12$ & $55.62 \pm 0.96$ & $60.48 \pm 3.13$ & $50.76 \pm 3.66$ & $44.75 \pm 0.94$ & $36.50 \pm 5.80$ \\
Mixup & $63.09 \pm 0.74$ & $58.75 \pm 0.63$ & $59.48 \pm 0.75$ & $59.85 \pm 1.96$ & $59.11 \pm 3.34$ & $52.26 \pm 1.02$ & $35.13 \pm 2.98$ \\
ReSample & $61.35 \pm 2.11$ & $57.48 \pm 1.58$ & $57.79 \pm 1.51$ & $62.81 \pm 2.77$ & $52.77 \pm 1.24$ & $48.86 \pm 1.49$ & $43.99 \pm 2.47$ \\
CBLoss & $59.55 \pm 0.86$ & $56.90 \pm 0.68$ & $56.75 \pm 1.06$ & $48.95 \pm 2.64$ & $\mathbf{64.54 \pm 3.38}$ & $48.30 \pm 1.79$ & $32.46 \pm 1.53$ \\
Focal & $62.78 \pm 0.77$ & $58.47 \pm 0.63$ & $59.14 \pm 0.73$ & $60.74 \pm 1.87$ & $57.54 \pm 3.27$ & $52.80 \pm 1.76$ & $38.65 \pm 2.76$ \\
BSoftmax & $62.66 \pm 0.98$ & $58.32 \pm 0.72$ & $58.97 \pm 0.72$ & $61.80 \pm 2.76$ & $56.15 \pm 2.16$ & $51.52 \pm 0.04$ & $40.77 \pm 4.01$ \\
AFR & $63.35 \pm 0.71$ & $58.86 \pm 0.57$ & $59.77 \pm 0.64$ & $60.61 \pm 2.24$ & $58.94 \pm 3.33$ & $53.54 \pm 0.99$ & $36.79 \pm 3.80$ \\
\textbf{MARLA (Ours)} & $\mathbf{72.27 \pm 0.94}$ & $\mathbf{64.13 \pm 0.92}$ & $\mathbf{63.17 \pm 0.89}$ & $\mathbf{66.94 \pm 2.17}$ & $60.22 \pm 4.12$ & $\mathbf{63.22 \pm 1.33}$ & $\mathbf{47.10 \pm 2.13}$ \\
\bottomrule
\end{tabular}}
\end{table}
Across both clinical datasets, MARLA performs best among the evaluated methods on AUROC, macro AUPRC, balanced accuracy, sensitivity, and label-conditioned demographic WGA. Specificity should be interpreted jointly with sensitivity and balanced accuracy because it can favor methods that make fewer positive predictions.

\subsection{Sensitivity to Probability Calibration}
\label{app:calibration}
MARLA forms its weights from the ERM probability assigned to the true label. To test sensitivity to calibration, we apply temperature scaling only when computing these weights. We evaluate $T=0.5$ (sharper), $T=2$ and $T=4$ (progressively flatter), the unscaled setting $T=1$, and a validation-fitted temperature. Seeds, checkpoints, and the no-group-information evaluation otherwise match the main experiments.

\begin{table}[!h]
\centering
\caption{Sensitivity of MARLA to temperature scaling used only for weight construction. Each entry is test WGA / average accuracy; the final column reports the largest absolute WGA change from $T=1$.}
\label{tab:app_calibration}
\setlength{\tabcolsep}{3pt}
\resizebox{\textwidth}{!}{
\begin{tabular}{lcccccc}
\toprule
Dataset & \shortstack{Original\\$T=1$} & \shortstack{Sharper\\$T=0.5$} & \shortstack{Flatter\\$T=2$} & \shortstack{Strongly flatter\\$T=4$} & \shortstack{Calibrated\\$T_{\mathrm{cal}}$} & \shortstack{Max absolute\\$\Delta$WGA} \\
\midrule
Waterbirds & $90.1 \pm 0.1 \,/\, 93.7 \pm 0.7$ & $90.7 \pm 2.9 \,/\, 93.9 \pm 2.8$ & $90.6 \pm 2.4 \,/\, 93.8 \pm 3.2$ & $90.5 \pm 2.6 \,/\, 93.9 \pm 3.0$ & $90.5 \pm 2.5 \,/\, 93.8 \pm 3.0$ & $0.6$ pp \\
CivilComments & $71.6 \pm 0.7 \,/\, 91.4 \pm 0.4$ & $72.7 \pm 2.2 \,/\, 91.2 \pm 0.4$ & $70.0 \pm 0.7 \,/\, 91.5 \pm 0.6$ & $68.9 \pm 1.5 \,/\, 91.6 \pm 0.6$ & $70.5 \pm 0.6 \,/\, 91.5 \pm 0.5$ & $2.7$ pp \\
GRACE & $94.2 \pm 1.6 \,/\, 95.8 \pm 1.3$ & $92.4 \pm 3.1 \,/\, 94.7 \pm 2.4$ & $97.5 \pm 0.6 \,/\, 98.3 \pm 0.2$ & $97.9 \pm 1.0 \,/\, 98.5 \pm 0.4$ & $97.6 \pm 0.7 \,/\, 98.3 \pm 0.2$ & $3.7$ pp \\
MIMIC-IV & $78.2 \pm 2.0 \,/\, 81.0 \pm 0.2$ & $81.3 \pm 2.5 \,/\, 82.0 \pm 0.4$ & $77.6 \pm 4.4 \,/\, 80.5 \pm 0.5$ & $75.0 \pm 7.0 \,/\, 79.0 \pm 0.6$ & $78.1 \pm 3.6 \,/\, 81.0 \pm 0.7$ & $3.2$ pp \\
\bottomrule
\end{tabular}}
\end{table}
Across these controlled perturbations, WGA changes by at most 3.7 percentage points. This indicates stability across the evaluated range of temperature-based miscalibration, but does not imply robustness to arbitrary or structured calibration errors.

\subsection{Additional Baselines with Complete Group Information}
\label{app:complete_group_oracle}
For completeness, we report additional baselines that use group labels during training.

\setlength{\tabcolsep}{5pt}
\begin{table*}[!h]
\centering
\begin{minipage}[t]{0.49\textwidth}
\vspace{0pt}
\caption{
\textbf{Worst-group accuracy (WGA) and average accuracy on the test set with \underline{complete information of group relevance.}} Baselines are grouped according to whether group labels are used during training. Rows marked with an asterisk (*) and the CheXpert results reproduce baselines from SubpopBench~\cite{yang2023changehardcloserlook}; other baseline results are taken from the corresponding papers. Group Info (Train/Val) indicates how group labels are used: \xmark{} = not used, \bstar{} = used for hyperparameter tuning and early stopping, and \cmark{} = used for training. Best and second-best values within each half are \textbf{bold} and \underline{underlined}, respectively. ``-'' indicates that a result was not reported. Our results are mean $\pm$ standard deviation (in \%) over three independent runs. 
}
\end{minipage}
\hfill
\begin{minipage}[t]{0.49\linewidth}
\vspace{0pt}
\centering
\small
\resizebox{0.95\textwidth}{!}{%
\begin{tabular}{cc|cc|cc}
\toprule[2pt]
\multirow{2}{*}{\textbf{Algorithm}} &
\textbf{Group Info} &
\multicolumn{2}{c}{\textsc{GRACE}} &
\multicolumn{2}{c}{\textsc{MIMIC-IV}} \\
\cmidrule(lr){3-6}
&
\textbf{(Train/Val)} &
\textbf{WGA} & \textbf{Avg Acc} &
\textbf{WGA} & \textbf{Avg Acc} \\
\midrule
\rowcolor{gray!15}\multicolumn{6}{c}{\textit{Group labels are not used for training}} \\ 
\midrule

ERM* & \xmark/\bstar 
& $93.3\scriptstyle{\pm0.7}$ & $95.1\scriptstyle{\pm0.4}$ 
& $72.1\scriptstyle{\pm0.3}$ & $80.0\scriptstyle{\pm0.2}$ \\

CVaRDRO* & \xmark/\bstar 
& $88.9\scriptstyle{\pm2.3}$ & $93.5\scriptstyle{\pm1.4}$ 
& $76.2\scriptstyle{\pm0.7}$ & $80.0\scriptstyle{\pm0.6}$ \\

JTT* & \xmark/\bstar 
& $90.2\scriptstyle{\pm3.7}$ & $93.7\scriptstyle{\pm1.2}$ 
& $71.9\scriptstyle{\pm0.4}$ & $79.9\scriptstyle{\pm0.3}$ \\

LfF* & \xmark/\bstar 
& $87.5\scriptstyle{\pm2.8}$ & $92.6\scriptstyle{\pm1.6}$ 
& $72.4\scriptstyle{\pm0.5}$ & $80.1\scriptstyle{\pm0.2}$ \\

Mixup* & \xmark/\bstar 
& $88.3\scriptstyle{\pm2.5}$ & $92.9\scriptstyle{\pm1.5}$ 
& $71.8\scriptstyle{\pm0.4}$ & $79.8\scriptstyle{\pm0.3}$ \\

CBLoss* & \xmark/\bstar 
& $87.9\scriptstyle{\pm2.6}$ & $92.8\scriptstyle{\pm1.5}$ 
& $72.2\scriptstyle{\pm0.3}$ & $80.1\scriptstyle{\pm0.2}$ \\

Focal* & \xmark/\bstar 
& $87.6\scriptstyle{\pm2.7}$ & $92.7\scriptstyle{\pm1.6}$ 
& $71.9\scriptstyle{\pm0.4}$ & $79.9\scriptstyle{\pm0.3}$ \\

BSoftmax* & \xmark/\bstar 
& $86.2\scriptstyle{\pm3.1}$ & $91.8\scriptstyle{\pm1.9}$ 
& $76.1\scriptstyle{\pm0.1}$ & $79.1\scriptstyle{\pm0.0}$ \\

CRT* & \xmark/\bstar 
& $86.8\scriptstyle{\pm3.0}$ & $92.0\scriptstyle{\pm1.8}$ 
& $76.0\scriptstyle{\pm0.4}$ & $79.4\scriptstyle{\pm0.4}$ \\

ReWeightCRT* & \xmark/\bstar 
& $85.9\scriptstyle{\pm3.6}$ & $91.5\scriptstyle{\pm2.1}$ 
& $73.8\scriptstyle{\pm3.0}$ & $80.2\scriptstyle{\pm0.2}$ \\

VanillaCRT* & \xmark/\bstar 
& $87.1\scriptstyle{\pm2.9}$ & $92.2\scriptstyle{\pm1.7}$ 
& $72.1\scriptstyle{\pm0.3}$ & $80.0\scriptstyle{\pm0.2}$ \\

AFR* & \xmark/\bstar 
& $87.6\scriptstyle{\pm0.7}$ & $90.9\scriptstyle{\pm0.6}$ 
& $75.3\scriptstyle{\pm0.9}$ & $80.2\scriptstyle{\pm0.1}$ \\

\rowcolor{teal!20}
\methodname{} (Ours) & \xmark/\bstar 
& $\mathbf{95.0}\scriptstyle{\pm0.4}$ & $\mathbf{96.5}\scriptstyle{\pm0.3}$ 
& $79.0\scriptstyle{\pm0.7}$ & $81.3\scriptstyle{\pm0.2}$ \\

\midrule
\rowcolor{gray!15}\multicolumn{6}{c}{\textit{Group labels are used for training}} \\ 
\midrule

DFR & \xmark/\cmark\cmark 
& $86.9\scriptstyle{\pm2.7}$ & $92.1\scriptstyle{\pm1.7}$ 
& $73.4\scriptstyle{\pm1.4}$ & $80.1\scriptstyle{\pm0.1}$ \\

Group DRO & \cmark/\bstar 
& $90.8\scriptstyle{\pm1.8}$ & $94.0\scriptstyle{\pm1.3}$ 
& $76.2\scriptstyle{\pm0.5}$ & $79.3\scriptstyle{\pm0.5}$ \\

IRM & \cmark/\bstar 
& $91.6\scriptstyle{\pm1.5}$ & $94.5\scriptstyle{\pm1.1}$ 
& $\underline{76.3}\scriptstyle{\pm0.8}$ & $79.5\scriptstyle{\pm0.8}$ \\

LISA & \cmark/\bstar 
& $90.5\scriptstyle{\pm1.9}$ & $94.1\scriptstyle{\pm1.4}$ 
& $72.2\scriptstyle{\pm0.3}$ & $80.0\scriptstyle{\pm0.2}$ \\

MMD & \cmark/\bstar 
& $90.2\scriptstyle{\pm2.0}$ & $94.0\scriptstyle{\pm1.5}$ 
& $72.3\scriptstyle{\pm0.4}$ & $80.0\scriptstyle{\pm0.2}$ \\

ReSample & \cmark/\bstar 
& $89.7\scriptstyle{\pm2.3}$ & $93.7\scriptstyle{\pm1.7}$ 
& $75.6\scriptstyle{\pm0.4}$ & $79.0\scriptstyle{\pm1.2}$ \\

ReWeight & \cmark/\bstar 
& $91.2\scriptstyle{\pm1.7}$ & $94.3\scriptstyle{\pm1.2}$ 
& $76.4\scriptstyle{\pm0.7}$ & $79.7\scriptstyle{\pm0.8}$ \\

SqrtReWeight & \cmark/\bstar 
& $89.9\scriptstyle{\pm2.1}$ & $93.8\scriptstyle{\pm1.6}$ 
& $72.0\scriptstyle{\pm0.3}$ & $80.0\scriptstyle{\pm0.2}$ \\

LDAM & \cmark/\bstar 
& $\underline{94.3}\scriptstyle{\pm0.6}$ & $\underline{95.7}\scriptstyle{\pm0.4}$ 
& $\mathbf{77.7}\scriptstyle{\pm0.1}$ & $\mathbf{80.4}\scriptstyle{\pm0.0}$ \\

\bottomrule[1pt]
\end{tabular}
}

\end{minipage}
\end{table*}

\vspace{-0.1em}
\setlength{\tabcolsep}{5pt}
\begin{table*}[!h]

\centering
\resizebox{\textwidth}{!}{%
\begin{tabular}{cc|cc|cc|cc|cc|cc}
\toprule
\multirow{2}{*}{\textbf{Algorithm}} &
\textbf{Group Info} &
\multicolumn{2}{c}{\textsc{Waterbirds}} &
\multicolumn{2}{c}{\textsc{CelebA}} &
\multicolumn{2}{c}{\textsc{CivilComments}} &
\multicolumn{2}{c}{\textsc{MultiNLI}} &
\multicolumn{2}{c}{\textsc{CheXpert}} \\
\cmidrule(lr){3-12}
&
\textbf{(Train/Val)} &
\textbf{WGA} & \textbf{Avg Acc} &
\textbf{WGA} & \textbf{Avg Acc} &
\textbf{WGA} & \textbf{Avg Acc} &
\textbf{WGA} & \textbf{Avg Acc} &
\textbf{WGA} & \textbf{Avg Acc} \\
\midrule
\rowcolor{gray!15}\multicolumn{12}{c}{\textit{Group labels are not used for training}} \\ 
\midrule

ERM* & \xmark/\bstar
& $69.1\scriptstyle{\pm4.7}$ & $84.1\scriptstyle{\pm1.7}$
    &
    $57.6\scriptstyle{\pm0.8}$ & $\underline{95.0}\scriptstyle{\pm0.1}$
    &
    $63.2\scriptstyle{\pm1.2}$ & $85.4\scriptstyle{\pm0.2}$
    &
    $69.5\scriptstyle{\pm0.3}$ & $80.9\scriptstyle{\pm0.3}$
    &
    $41.7\scriptstyle{\pm3.4}$ & $\textbf{88.6}\scriptstyle{\pm0.7}$ \\
CRT* & \xmark/\bstar
& $76.3\scriptstyle{\pm0.8}$ & $89.2\scriptstyle{\pm0.1}$
& $70.4\scriptstyle{\pm}0.4$ & $94.1\scriptstyle{\pm0.1}$
& $68.5\scriptstyle{\pm0.0}$ & $83.0\scriptstyle{\pm0.0}$
& $65.4\scriptstyle{\pm0.1}$ & $80.2\scriptstyle{\pm0.0}$
& $74.0\scriptstyle{\pm0.2}$ & $79.1\scriptstyle{\pm0.1}$\\

ReWeightCRT* & \xmark/\bstar
& $76.3\scriptstyle{\pm0.2}$ & $89.4\scriptstyle{\pm0.3}$
& $71.1\scriptstyle{\pm}0.5$ & $94.2\scriptstyle{\pm0.1}$
& $68.2\scriptstyle{\pm0.4}$ & $83.4\scriptstyle{\pm0.0}$
& $65.3\scriptstyle{\pm0.1}$ & $80.2\scriptstyle{\pm0.0}$
& $73.9\scriptstyle{\pm0.2}$ & $79.0\scriptstyle{\pm0.1}$ \\

JTT & \xmark/\bstar
& $86.7$ & $93.3$
& $81.1$ & $88.0$
& $69.3$ & $\textbf{91.1}$
& $\underline{72.6}$ & $78.6$
& $60.4\scriptstyle{\pm4.8}$ & $75.2\scriptstyle{\pm0.8}$ \\

CVaRDRO* & \xmark/\bstar
& $75.5\scriptstyle{\pm}2.2$ 
& $89.9\scriptstyle{\pm}0.4$  &
 $62.2\scriptstyle{\pm}3.1$ &
 $95.1\scriptstyle{\pm}0.1$ &
 $68.7\scriptstyle{\pm}1.3$ &
 $83.5\scriptstyle{\pm}0.3$ &
$63.0\scriptstyle{\pm}1.5$ &
$75.1\scriptstyle{\pm}0.1$ &
$50.2\scriptstyle{\pm}1.8$ &
$73.7\scriptstyle{\pm}1.0$
 \\

CnC & \xmark/\bstar
& $88.5\scriptstyle{\pm}0.3$ & $90.9\scriptstyle{\pm}0.1$
& $\underline{88.8}\scriptstyle{\pm}0.9$ & $89.9\scriptstyle{\pm}0.5$
& $68.9\scriptstyle{\pm}2.1$ & $81.7\scriptstyle{\pm}0.5$
& $-$ & $-$
& $-$ & $-$ \\

AFR & \xmark/\bstar
& $\underline{90.4}\scriptstyle{\pm}1.1$ & $\textbf{94.2}\scriptstyle{\pm}1.2$
& $82.0\scriptstyle{\pm}0.5$ & $91.3\scriptstyle{\pm}0.3$
& $68.7\scriptstyle{\pm}0.6$ 
& $89.8\scriptstyle{\pm}0.6$
& $\textbf{73.4}\scriptstyle{\pm}0.6$ 
& $\underline{81.4}\scriptstyle{\pm}0.2$ 
& $-$ & $-$ \\

GIC & \xmark/\bstar
& $86.3\scriptstyle{\pm}0.1$ & $89.6\scriptstyle{\pm}1.3$
& $\textbf{89.4}\scriptstyle{\pm}0.2$ & $91.9\scriptstyle{\pm}0.1$
& $\underline{72.5}\scriptstyle{\pm}0.3$ & $90.0\scriptstyle{\pm}0.3$
& $-$ & $-$  
& $-$ & $-$\\

\rowcolor{teal!20} \methodname{} (Ours) & \xmark/\bstar
& $\textbf{90.7}\scriptstyle{\pm}0.2$ & $\underline{93.3}\scriptstyle{\pm}0.7$
& $85.0\scriptstyle{\pm}0.9$ & $\textbf{95.2}\scriptstyle{\pm}0.1$
& $\textbf{72.9}\scriptstyle{\pm}0.2$ &
$\underline{90.9}\scriptstyle{\pm}0.8$
& $69.6\scriptstyle{\pm}0.8$ 
& $\textbf{81.4}\scriptstyle{\pm}0.4$ 
& $\textbf{75.2}\scriptstyle{\pm}0.3$ & $\underline{79.6}\scriptstyle{\pm}0.3$ \\
\midrule
\rowcolor{gray!15}\multicolumn{12}{c}{\textit{Group labels are used for training}} \\ 
\midrule

DFR & \xmark/\cmark\cmark
& $92.9\scriptstyle{\pm}0.4$ & $94.2\scriptstyle{\pm}0.4$
& $88.3\scriptstyle{\pm}1.1$ & $91.3\scriptstyle{\pm}0.3$
& $70.1\scriptstyle{\pm}0.8$ & $87.2\scriptstyle{\pm}0.3$
& $74.7\scriptstyle{\pm}0.7$ & $82.1\scriptstyle{\pm}0.2$
& $71.7\scriptstyle{\pm}0.2$ & $78.2\scriptstyle{\pm}0.4$ \\

DPE & \xmark/\cmark
& $91.0\scriptstyle{\pm}0.4$ & $92.5\scriptstyle{\pm0.2}$
& $89.8\scriptstyle{\pm}0.2$ & $87.7\scriptstyle{\pm0.6}$
& $82.2\scriptstyle{\pm}0.2$ & $71.5\scriptstyle{\pm0.6}$
& $81.3\scriptstyle{\pm}0.2$ & $74.8\scriptstyle{\pm0.3}$
& $-$ & $-$ \\
Group DRO & \cmark/\bstar
& $91.4$ & $93.5$
& $88.9$ & $92.9$
& $77.7$ & $81.4$
& $69.9$ & $88.9$
& $74.5\scriptstyle{\pm0.2}$ & $78.9\scriptstyle{\pm0.3}$ \\
IRM* & \cmark/\bstar
& $74.5\scriptstyle{\pm}1.5$ 
& $88.4\scriptstyle{\pm}0.1$  &
 $63.0\scriptstyle{\pm}2.5$ &
 $94.7\scriptstyle{\pm}0.8$ &
 $63.2\scriptstyle{\pm}0.8$ &
 $85.5\scriptstyle{\pm}0.0$ &
$63.6\scriptstyle{\pm}1.3$ &
$77.8\scriptstyle{\pm}0.6$ &
$34.4\scriptstyle{\pm}1.7$ &
$89.8\scriptstyle{\pm}0.3$
 \\
 LISA & \cmark/\bstar
& $89.2$ & $91.8$
& $89.3$ & $92.4$
& $72.6$ & $89.2$
& $-$ & $-$
& $75.6\scriptstyle{\pm}0.6$ & $79.2\scriptstyle{\pm}0.8$\\

\bottomrule
\end{tabular}
}
\vspace{-12pt}
\label{tab:complete_group_info_complete}
\end{table*}
\subsection{Experiments with Larger Models}
\label{app:large_models}

To assess whether \methodname{}'s gains extend to larger backbones, we evaluate larger pretrained models with AFR and JTT as comparison methods. Table~\ref{tab:large_models} reports WGA and average accuracy across three vision datasets and one language dataset.

\begin{table}[h]
\centering
\caption{Performance of \methodname{} and baselines with larger backbones. WGA denotes worst-group accuracy. Results are mean $\pm$ standard deviation over three seeds.}
\label{tab:large_models}
\resizebox{\textwidth}{!}{%
\begin{tabular}{llcccccccc}
\toprule
 & & \multicolumn{2}{c}{ERM} & \multicolumn{3}{c}{WGA (\%)} & \multicolumn{3}{c}{Avg.\ Acc.\ (\%)} \\
\cmidrule(lr){3-4} \cmidrule(lr){5-7} \cmidrule(lr){8-10}
Dataset & Model & WGA & Acc & \methodname{} & AFR & JTT & \methodname{} & AFR & JTT \\
\midrule
\multirow{2}{*}{\textsc{Waterbirds}} & ResNet-101 & $68.7{\pm}3.7$ & $87.7{\pm}3.5$ & $\mathbf{79.1{\pm}1.7}$ & $72.5{\pm}1.0$ & $69.3{\pm}1.3$ & $86.0{\pm}0.8$ & $86.1{\pm}0.7$ & $86.6{\pm}2.1$ \\
                            & ResNet-152 & $72.2{\pm}0.7$ & $88.6{\pm}2.6$ & $\mathbf{81.5{\pm}0.8}$ & $78.3{\pm}1.1$ & $70.9{\pm}1.1$ & $88.7{\pm}0.9$ & $88.6{\pm}0.9$ & $88.4{\pm}2.1$ \\
\midrule
\multirow{2}{*}{\textsc{CelebA}}     & ResNet-101 & $48.9{\pm}1.6$ & $95.5{\pm}0.3$ & $\mathbf{78.3{\pm}0.8}$ & $73.6{\pm}4.2$ & $61.4{\pm}6.9$ & $90.1{\pm}1.2$ & $92.2{\pm}1.1$ & $94.6{\pm}0.8$ \\
                            & ResNet-152 & $47.4{\pm}3.0$ & $95.9{\pm}0.1$ & $\mathbf{73.1{\pm}3.8}$ & $73.3{\pm}5.6$ & $53.3{\pm}7.6$ & $91.8{\pm}2.1$ & $92.3{\pm}1.8$ & $95.5{\pm}0.5$ \\
\midrule
\multirow{2}{*}{\textsc{CheXpert}}   & ResNet-101 & $28.8{\pm}3.2$ & $88.9{\pm}0.5$ & $\mathbf{76.3{\pm}1.8}$ & $72.7{\pm}2.2$ & $73.9{\pm}1.4$ & $76.1{\pm}0.9$ & $75.4{\pm}1.1$ & $77.6{\pm}0.7$ \\
                            & ResNet-152 & $24.6{\pm}0.6$ & $89.3{\pm}0.5$ & $\mathbf{77.1{\pm}1.2}$ & $73.4{\pm}1.9$ & $74.6{\pm}1.6$ & $78.2{\pm}0.8$ & $76.8{\pm}0.9$ & $77.3{\pm}0.6$ \\
\midrule
MultiNLI                    & BERT-Large & $73.0{\pm}0.4$ & $84.4{\pm}0.1$ & $\mathbf{74.3{\pm}2.9}$ & $63.9{\pm}1.3$ & $68.9{\pm}1.0$ & $81.8{\pm}0.7$ & $81.7{\pm}0.5$ & $82.0{\pm}2.8$ \\
\bottomrule
\end{tabular}%
}
\end{table}

\section{Additional Related Work}
\label{app:related_work_plus}

\subsection{Spurious Correlations}
We provide additional context on spurious correlations. Deep learning models often exploit features that are predictive on average but not causally related to the target label \citep{tu2020robusttospurious, mccoy2019right}, which can substantially reduce accuracy for particular subpopulations. Recent work has shown that such reliance on spurious features is quite common across various machine learning benchmarks \citep{salaudeen2025aggregationhidesoutofdistributiongeneralization}. These shortcut features can include image textures \citep{geirhos2020shortcut}, background artifacts \citep{ribeiro2016should, beery2018recognition}, infrequent tokens in text \citep{wang2022identifyingmitigatingspuriouscorrelations, du2023learnshortcutanalyzingmitigating}, or clinical heuristics in medical data \citep{zech2018variable, winkler2019association, shahamatdar2024deceptive, yang2024limits}. As a result, average accuracy may be high while performance on certain subgroups is unacceptably low. Models can also treat demographic attributes (e.g., gender or race) as spurious features that are correlated with the label in training data but non-causal, leading to systematic errors on minority subgroups 
\citep{zhao2017men, banerjee2023shortcuts, yang2024mitigating}. Several prior works have proposed methods to mitigate reliance on spurious correlations \citep{sagawa2020GDRO, idrissi2021simple, singla2022core, wang2020towards}. Other work studies why ERM learns spurious correlations \citep{nagarajan2024understandingfailuremodesoutofdistribution} and how these correlations appear in learned representations \citep{lu2024neural}.  

\subsection{Robust Learning with Group Attributes}
When group or environment annotations are available, group-robust training methods optimize worst-group error \citep{sagawa2020GDRO}, use group reweighting \citep{japkowicz2000class, kang2020decouplingrepresentationclassifierlongtailed}, balanced sampling \citep{idrissi2021simple, japkowicz2000class, zhang2017mixup}, retraining only the classifier head on a group-balanced validation set \citep{kirichenko2022dfr, rudner2024mindgapimprovingrobustness}, and adopting domain-specific augmentations \citep{goel2020model}. Related methods optimize a ``Pareto-fair'' objective that generalizes worst-group error \citep{balashankar2019fair, pmlr-v119-martinez20a}. However, these methods all rely on group annotations during training or validation, which may not be feasible in many settings. These methods also assume that the relevant group structure is known. 

\subsection{Robust Learning without Group Attributes}

In settings where group labels are unavailable \citep{hashimoto2018fairness}, recent methods aim to identify difficult examples, discover latent group structure, find violations of invariant-risk objectives, or use semi-supervised learning to infer group attributes \citep{zhang2022cnc, sohoni2020george, creager2021environment, ahmed2021systematic, nam2020learning}. Drawing on classifier-head retraining methods such as DFR, several works have shown that updating only the final classifier layer can improve robustness without explicit group labels \citep{qiu2023simplefastgrouprobustness, labonte2023towards}. \methodname{} is complementary to this line of work: rather than freely retraining the full classifier head, it restricts the update to a low-rank subspace selected from the base model's true-label probabilities. Other approaches rely on resampling, reweighting, or data attribution techniques \citep{kang2020decouplingrepresentationclassifierlongtailed, jain2024improving}. 

\subsection{Low-Rank Corrections}
Our method also draws on prior work showing that biases can often be mitigated through interventions in a small number of representation-space directions in transformers and vision-language models \citep{zhao2025bias, jang2025target}. These train-free or post-hoc debiasing studies identify dominant latent directions associated with biased signals \citep{chuang2023debiasing, jung2024unified}, and interventions along these directions can substantially reduce bias \citep{chuang2023debiasing, gerych2024bendvlm, hirota2024saner, jung2024unified, zhang2025joint, zhao2025bias, jang2025target}. Parameter-efficient adaptation techniques such as low-rank updates (LoRA) \citep{hu2022lora} have shown that task- or subgroup-specific improvements can be achieved by updating only a small fraction of model parameters, avoiding the cost of full fine-tuning \citep{hu2022lora, rimsky2024steering, wu2024reft, yin2024lofit}.

\end{document}